\documentclass{article}
\usepackage{amsmath}
\usepackage{amssymb}
\usepackage{algorithmic}
\DeclareMathOperator{\PossibleDsep}{PossibleDsep}

\usepackage{hyphenat}

\PassOptionsToPackage{numbers, compress}{natbib}

\usepackage[preprint]{neurips_2023}

\usepackage[utf8]{inputenc} % allow utf-8 input
\usepackage[T1]{fontenc}    % use 8-bit T1 fonts
\usepackage{hyperref}       % hyperlinks
\usepackage{url}            % simple URL typesetting
\usepackage{booktabs}       % professional-quality tables
\usepackage{amsfonts}       % blackboard math symbols
\usepackage{amssymb}        % for \square and other math symbols
\usepackage{amssymb}        % for \square and other symbols
\usepackage{nicefrac}       % compact symbols for 1/2, etc.
\usepackage{microtype}      % microtypography
\usepackage{xcolor}         % colors

\usepackage{eucal}
\newcommand{\test}[1]{\text{#1}}
\usepackage{amsmath, amssymb, amsthm}
\usepackage{bm}
\usepackage{graphicx}
\usepackage{float}
\usepackage{subcaption}
\usepackage{multirow}
\usepackage[linesnumbered, ruled, vlined]{algorithm2e}
\usepackage{tikz}
\usepackage{pgfplots}
\usepackage{algorithmic}
\usepackage{enumitem}
\usepackage[utf8]{inputenc}
\usepackage{mathabx}
\usetikzlibrary{arrows.meta, positioning}
\pgfplotsset{compat=1.8}
\usepgfplotslibrary{groupplots}
\usepgfplotslibrary{fillbetween}

\definecolor{color0}{HTML}{636EFA}
\definecolor{color1}{HTML}{EF553B}
\definecolor{color2}{HTML}{00CC96}
\definecolor{color3}{HTML}{AB63FA}
\definecolor{color4}{HTML}{FFA15A}
\definecolor{color5}{HTML}{19D3F3}
\definecolor{color6}{HTML}{FF6692}
\definecolor{color7}{HTML}{B6E880}
\definecolor{color8}{HTML}{FF97FF}
\definecolor{color9}{HTML}{FECB52}

\DeclareMathOperator{\Blocked}{Blocked}
\DeclareMathOperator{\Adj}{Adj}

\DeclareMathOperator{\Path}{Path}
\DeclareMathOperator{\Inducing}{Inducing}

\let \mc = \mathcal

\newcommand{\ind}{{\;\perp\!\!\!\perp\;}}

\newcommand{\AdjG}[1]{\mathrm{Adj}_G(#1)}

\newtheorem{theorem}{Theorem}
\newtheorem{corollary}{Corollary}
\newtheorem{proposition}{Proposition}
\newtheorem{definition}{Definition}

\newtheorem{lemma}{Lemma}

\newtheorem{remark}{Remark}

\title{Efficient Latent Variable Causal Discovery: Combining Score Search and Targeted Testing}

\author{%
  Joseph Ramsey \\
  Department of Philosophy \\
  Carnegie Mellon University \\
  Pittsburgh, PA 15213 \\
  \texttt{jdramsey@andrew.cmu.edu } \\
  \And
  Bryan Andrews \\
  Department of Psychiatry \& Behavioral Sciences \\
  University of Minnesota \\
  Minneapolis, MN 55454 \\
  \texttt{andr1017@umn.edu} \\
  \And
  Peter Spirtes \\
  Department of Philosophy \\
  Carnegie Mellon University \\
  Pittsburgh, PA 15213 \\
  \texttt{ps7z@andrew.cmu.edu } \\
}

\begin{document}

\maketitle

\begin{abstract}
Learning causal structure from observational data is especially challenging when latent variables or selection bias are present. The Fast Causal Inference (FCI) algorithm addresses this setting but performs exhaustive conditional independence tests across many subsets, often leading to spurious independences, missing or extra edges, and unreliable orientations.

We present a family of \emph{score-guided mixed-strategy causal search algorithms} that extend this framework. First, we introduce \emph{BOSS-FCI} and \emph{GRaSP-FCI}, variants of GFCI (Greedy Fast Causal Inference) that substitute BOSS (Best Order Score Search) or GRaSP (Greedy Relaxations of Sparsest Permutation) for FGES (Fast Greedy Equivalence Search), preserving correctness while trading off scalability and conservativeness. Second, we develop \emph{FCI Targeted-Testing (FCIT)}, a novel hybrid method that replaces exhaustive testing with targeted, score-informed tests guided by BOSS. FCIT guarantees well-formed PAGs and achieves higher precision and efficiency across sample sizes. Finally, we propose a lightweight heuristic, \emph{LV-Dumb} (Latent Variable ``Dumb''), which returns the PAG of the BOSS DAG (Directed Acyclic Graph). Though not strictly sound for latent confounding, LV-Dumb often matches FCIT’s accuracy while running substantially faster.

Simulations and real-data analyses show that BOSS-FCI and GRaSP-FCI provide robust baselines, FCIT yields the best balance of precision and reliability, and LV-Dumb offers a fast, near-equivalent alternative. Together, these methods demonstrate that targeted and score-guided strategies can dramatically improve the efficiency and correctness of latent-variable causal discovery.
\end{abstract}

\section{Introduction}

Learning causal structure from observational data is central in fields such as epidemiology, economics, and biology, where randomized experiments are often infeasible. In many real-world settings, latent variables or selection bias complicate this task, undermining the reliability of traditional discovery methods. The Fast Causal Inference (FCI) algorithm \citep{spirtes2001causation} provides a principled framework for such cases by applying conditional independence (CI) tests to subsets of adjacent variables (or Possible-D-SEP sets). While correct in theory, this strategy is fragile in practice: testing many subsets amplifies stochastic variation in $p$-values, leading to spurious independence claims, missing or extra edges, and unstable orientations. We refer to this as the \emph{repeated testing problem}.

\begin{definition}[Repeated Testing Problem]
Let $\mathcal{T}(X \ind Y \mid S)$ denote a statistical CI test at level $\alpha$ for disjoint sets $X,Y,S \subseteq V$. FCI tests many subsets $S \subseteq \text{Adj}(X) \cup \text{Adj}(Y)$ or Possible-D-SEP$(X,Y)$. As $|\mathcal{S}|$ grows, the chance that some $S \in \mathcal{S}$ spuriously satisfies $\mathcal{T}(X \ind Y \mid S)$ increases, inflating Type~I errors and producing overly sparse graphs and erroneous orientations.
\end{definition}

Beyond sampling variability, another challenge is \emph{almost unfaithfulness} \citep{zhang2008detection, spirtes2014uniformly}.
Even when the Causal Markov and Faithfulness assumptions hold exactly, small parameter perturbations can produce distributions that are nearly unfaithful—where partial correlations or other dependence measures are close to zero but not exactly zero. In finite samples, such “almost unfaithful” parameterizations lead to unstable independence decisions: small statistical fluctuations can flip dependence judgments, causing orientation errors or spurious adjacencies. As shown by \citet{zhang2008detection}, this phenomenon can make reliable causal inference difficult even for consistent algorithms, motivating procedures that are robust to weak dependencies or employ stronger assumptions such as $k$-triangle faithfulness \citep{spirtes2014uniformly}. We therefore treat almost unfaithfulness as a complementary concern to the repeated-testing problem—structural rather than statistical in origin, but equally capable of producing unstable PAGs in practice.

Several extensions have sought to mitigate such issues. Order-independent FCI variants \citep{colombo2014order} stabilize orientations, and score-based methods\footnote{Scores in Tetrad currently cover linear Gaussian and linear non-Gaussian models (using BIC-type criteria), multinomial models, mixed multinomial/linear Gaussian models, and, more recently, mixed multinomial/nonlinear post-nonlinear models via truncated basis-function scores \citep{ramsey2025scalablecausaldiscoveryrecursive}.} extend causal search to latent-variable models through likelihood-based criteria \citep{triantafillou2016score, bhattacharya2021differentiable}. Hybrid approaches such as Greedy FCI (GFCI; \citep{ogarrio2016hybrid}) combine score-based search with CI testing, but still rely on adjacency-based phases—i.e., during the adjacency-removal stage, testing all subsets of nodes adjacent to a pair—and thus remain vulnerable to repeated testing and almost unfaithfulness. Constraint-based methods, by contrast, have the advantage of being inherently nonparametric and making no distributional assumptions except for those of the selected conditional independence test, which motivates the option of keeping such testing in the algorithm.

Recent advances in score-based search motivate new hybrids that combine the distributional flexibility of constraint-based methods with the efficiency of likelihood-based optimization. Best Order Score Search (BOSS; \citep{andrews2019learning}) provides a scalable, permutation-based alternative to FGES, while GRaSP \citep{lam2022greedy} offers a related variant. Substituting these into the GFCI framework yields \emph{BOSS-FCI} and \emph{GRaSP-FCI}—hybrid algorithms that preserve the nonparametric generality of constraint-based search while leveraging score-based initialization to improve stability and scalability. The original GFCI algorithm \citep{ogarrio2016hybrid} is sound and complete in the population limit under the Causal Markov and Faithfulness assumptions, returning the correct Partial Ancestral Graph when provided with perfect independence information. However, in finite samples, GFCI and its variants still depend on adjacency-based conditioning and thus remain susceptible to repeated testing and almost unfaithfulness.

Building on these baselines, we propose \emph{FCI Targeted-Testing (FCIT)}, a mixed-strategy algorithm that replaces exhaustive adjacency-based testing with a path-based procedure. Like GFCI, FCIT begins with a score-based completed partially directed acyclic graph (CPDAG)—a representation of the Markov equivalence class of DAGs that share the same adjacencies and unshielded colliders. Score-based algorithms such as BOSS and GRaSP are \emph{score-consistent}: in the large-sample limit, they recover the correct CPDAG under the Causal Markov and Faithfulness assumptions for the model class (e.g., linear Gaussian, multinomial, or post-nonlinear) used to define the score. FCIT then applies a recursive routine, \texttt{block\_paths\_recursively}, to identify conditioning sets only when necessary to block active paths between variables, thereby avoiding the over-conditioning that arises in FCI and GFCI from testing all subsets of adjacent variables. Compared with these algorithms, FCIT performs far fewer conditional independence tests, which reduces sampling error and increases the reliability of the overall learned structure. Together with a well-formedness check that enforces the rules of a Partial Ancestral Graph (PAG)—a mixed graph representing the Markov equivalence class of maximal ancestral graphs consistent with possible latent variables and selection bias—this targeted strategy guarantees that every intermediate graph remains structurally consistent, meaning it obeys the ancestral, acyclicity, and maximality constraints required of a valid PAG. The result is greater efficiency, more stable orientations, and structural correctness not ensured by prior FCI variants \citep{zhang2008completeness, hu2024fast}.

Finally, we introduce a deliberately simple heuristic, \emph{LV-Dumb} (Latent Variables ``Dumb''), which outputs the PAG implied by the BOSS DAG without performing any latent-specific conditional independence testing. Though it cannot orient bidirected edges, LV-Dumb scales exceptionally well and, in extensive simulation studies, often matches or exceeds the empirical accuracy of more complex algorithms when compared to the true PAG.

Together, these four algorithms—BOSS-FCI, GRaSP-FCI, FCIT, and LV-Dumb—form a continuum. The first two establish sound baselines by embedding advanced score-based methods within the GFCI framework. FCIT introduces a principled path-based refinement that improves efficiency and reliability, while LV-Dumb represents a heuristic extreme: formally incomplete but practically effective. Our empirical evaluation, spanning both simulation and real data, demonstrates the complementary strengths of these approaches. BOSS-FCI and GRaSP-FCI retain theoretical correctness but are constrained by repeated testing; FCIT balances correctness with improved precision and scalability; and LV-Dumb sacrifices correctness guarantees yet achieves strong accuracy with minimal computational cost in simulation.

The remainder of this paper is organized as follows. Section~\ref{sec:graphical_preliminaries} reviews graphical preliminaries. Sections~\ref{sec:variants}–\ref{sec:fcit_description} present the four algorithms. Section~\ref{sec:optimizing_final_rules} introduces optimization details, Section~\ref{sec:theory} outlines theoretical guarantees, and Section~\ref{sec:empirical_eval} reports empirical results. All algorithms are implemented in the open-source Tetrad suite \citep{ramsey2018tetrad}, with reproducibility materials available on GitHub.

\paragraph{Contributions.}
\begin{itemize}
    \item Introduction of \emph{BOSS-FCI} and \emph{GRaSP-FCI}, hybrids substituting BOSS and GRaSP for FGES in GFCI. 
    \item A recursive blocking procedure, \texttt{block\_paths\_recursively}, that constructs conditioning sets by analyzing paths rather than adjacencies. 
    \item A targeted-testing algorithm, \emph{FCIT}, that uses recursive path blocking to maintain well-formed PAGs. 
    \item A heuristic, \emph{LV-Dumb}, that outputs the PAG of the BOSS DAG, achieving high empirical accuracy at minimal cost. 
    \item Empirical validation demonstrating accuracy, scalability, and robustness across simulated linear Gaussian datasets and real datasets.
\end{itemize}

\section{Graphical Preliminaries}
\label{sec:graphical_preliminaries}

In this section, we define key terms related to graphical causal models used throughout the paper. We focus on models based on causal directed acyclic graphs (DAGs), which consist of directed edges $X \rightarrow Y$. Such an edge indicates that variable $X$ causes variable $Y$, meaning that interventions on $X$ counterfactually alter $Y$~\citep{spirtes2001causation}.

A graph $G = (V, E)$ consists of a set of variables $V$ and a set of edges $E$, with at most one edge connecting each vertex pair, where each edge is denoted $X *{-}* Y$. The symbol “*” represents an endpoint, which may be a tail (``-''), an arrowhead (``$>$''), or a circle (``$\circ$''). Circle endpoints (“$\circ$”) indicate that a tail occurs in some members of the corresponding Markov equivalence class, and an arrowhead occurs in others. For example, $X *{-}* Y$ could represent $X \circ \!\! \rightarrow Y$, $X \leftarrow Y$, or $X \circ \!\!{-}\!\! \circ Y$. An arrow endpoint $X *\rightarrow Y$ means that $Y$ is not an ancestor of $X$, and a tail endpoint $X -*Y$ means that X is an ancestor of Y.

An edge is \emph{unoriented} if it is of the form $X \circ\!\!{-}\!\!\circ Y$. A \emph{directed} edge has one arrowhead: $X \rightarrow Y$ or $X \leftarrow Y$. An \emph{undirected} edge has tails at both ends: $X - Y$. Note that ``undirected'' refers specifically to tails at both ends, whereas ``unoriented'' refers to circles at both endpoints. We say $X$ is \textit{adjacent to} $Y$ in $G$, or $\text{adj}(X,Y)$, if there is an edge $X *{-}* Y$ in $G$, where ``*'' represents a tail (``-''), an arrowhead (``$>$''), or a circle (``$\circ$'').

A \emph{path} in $G$ is a possibly empty ordered sequence $\langle X_1, \dots, X_n \rangle$ (written $X_1 \leadsto X_n$) where each pair of vertices adjacent in the sequence is connected by an edge $X *{-}* Y$, where ``*'' represents a tail (``-''), an arrowhead (``$>$''), or a circle (``$\circ$''). A \emph{directed path} in $G$ (written $X_1 \leadsto X_n$) is a path each vertex at position $i\neq j$ in the sequence is connected by an edge $X_i \rightarrow X_{i+1}$. A graph is \emph{acyclic} if it contains no cyclic paths (that is, paths that revisit a node). A node $X$ is an ancestor of $Y$ if either $X = Y$ or else there is a directed path from $X$ to $Y$; $X$ is a parent of $Y$ if $X \rightarrow Y$ is in $G$. A collider on a path is a subpath $\langle A, B, C \rangle$ with $A \rightarrow B \leftarrow C$. It is an \emph{unshielded collider} if $A$ and $C$ are not adjacent.

We use italic letters for individual variables and boldface for sets of variables. DAGs encode CI statements of the form $X \ind Y \mid \mathbf{S}$ using the \emph{d-separation} criterion. Specifically, $X$ and $Y$ are d-separated given $\mathbf{S}$ (denoted $\textit{d-sep}(X, Y \mid \mathbf{S})$) if all paths between $X$ and $Y$ are blocked—either because they contain a non-collider in $\mathbf{S}$, or because every collider on the path has no descendant in $\mathbf{S}$.

A graph $G$ is \emph{Markov} to a distribution $P$ if $\text{d-sep}(X, Y \mid \mathbf{S})$ in $G$ implies that $X$ is independent of $Y$ conditional on $S$, denoted $X \ind Y \mid \mathbf{S}$, in $P$. Conversely, $P$ is \emph{faithful} to $G$ if $X \ind Y | \mathbf{S}$ implies that $\text{d-sep}(X,Y|\mathbf{S})$.

A \emph{Completed Partially Directed Acyclic Graph (CPDAG)} represents a Markov equivalence class of DAGs--that is, a set of DAGs all of which imply the same d-separation facts. Markov equivalent DAGs contain the same adjacencies but may differ in orientation. Thus, a CPDAG contains directed edges that appear in all DAGs in the class and undirected edges where orientations differ. Any DAG in the class can be obtained by orienting the undirected edges in a way that maintains acyclicity and preserves the set of unshielded colliders.

A \emph{Maximal Ancestral Graph (MAG)} is a mixed graph over observed variables encoding the CI facts of a causal DAG with latent and selection variables. It may contain:
\begin{itemize}
    \item \emph{Directed edges}: $X \rightarrow Y$, where $X$ is an ancestor of $Y$;
    \item \emph{Bidirected edges}: $X \leftrightarrow Y$, for latent confounding;
    \item \emph{Undirected edges}: $X - Y$, for selection bias.
\end{itemize}
MAGs satisfy:
\begin{itemize}
    \item \emph{Ancestral}: No directed or almost directed cycles.
    \item \emph{Maximality}: Every pair of non-adjacent nodes is m-separated by some set.
\end{itemize}

\emph{Selection bias} arises when inclusion in the dataset depends on one or more variables in the causal system—typically when data for a variable that is a descendant of multiple causes are incomplete, so that the observed data implicitly condition on that variable. In a causal graph, such conditioning can induce spurious associations among those causes, even when they are otherwise independent. In maximal ancestral graphs (MAGs) and partial ancestral graphs (PAGs), selection bias is represented by \emph{undirected edges} ($X {-} Y$), indicating that both $X$ and $Y$ are ancestors of a common selection variable. The \emph{selection variables} in this case are the variables that are being conditioned on in this way.

A \emph{Partial Ancestral Graph (PAG)} (sometimes referred to as a \emph{legal PAG}, i.e., simply a PAG) represents a Markov equivalence class of MAGs that imply the same m-separation facts. PAG edges may be:
\begin{itemize}
    \item Directed ($\rightarrow$), indicates a causal relationship.
    \item Bidirected ($\leftrightarrow$): indicates latent confounding.
    \item Undirected ($-$): indicates selection bias.
\end{itemize}
Alternatively, PAG edges can contain circle endpoints ($\circ$), indicating a tail occurs in some members of the equivalence class but an arrow in others.

An \emph{inducing path} between $X$ and $Y$ is a path where every non-endpoint node is a collider and an ancestor of $X$ or $Y$. If such a path exists between non-adjacent nodes, the graph is not maximal.

A MAG or PAG is \emph{maximal} (following Zhang~\citep{zhang2008completeness}). if, for every non-adjacent pair $X, Y$, there exists a conditioning set $\textbf{S}$ such that $X$ and $Y$ are m-separated given $\textbf{S}$. Equivalently, there are no inducing paths between non-adjacent nodes.

An \emph{almost-cycle} occurs when a graph contains both a bidirected edge $X \leftrightarrow Y$ and a directed path $X \leadsto Y$ (or $Y \leadsto X$), which implies that an arrowhead points to an ancestor, violating acyclicity.

Edges like $X \circ\!\!{-}\!\!\circ Y$ reflect ambiguity at both endpoints. For example, $X \rightarrow Y$ appears in all MAGs in the equivalence class, while $X \leftrightarrow Y$ indicates latent confounding, and $X \circ\!\!{-}\!\!\circ Y$ indicates unresolved ancestral relationships.

A graph G is said to be \emph{Markov to the data} if the data arise from a probability distribution $P$ over the observed variables and that the conditional independencies implied by G (via d-separation for DAGs or m-separation for MAGs and PAGs) hold in $P$.

A set of variables \(\textbf(V)\) is \emph{causally sufficient} if it includes all common causes of observed variables in \(\textbf(V)\); otherwise, it is \emph{causally insufficient}. A model is \emph{edge-minimal} of a certain type (e.g., DAG, CPDAG, PAG) if no model of that type has fewer edges while remaining Markov to the data. Edge minimality in CPDAGs is discussed by~\citet{raskutti2018learning, lam2022greedy}; here, we extend this to causally insufficient settings.

A path is \emph{blocked} by a set $\mathbf{Z}$ if:
\begin{itemize}
    \item It contains a collider (e.g., \(X \rightarrow W \leftarrow Y\)) such that \(W \notin \mathbf{Z}\) and no descendant of \(W\) is in \(\mathbf{Z}\);
    \item It contains a non-collider node in \(\mathbf{Z}\).
\end{itemize}
A set of nodes $\mathbf{Z}$ \emph{blocks} a pair of variable sets $\mathbf{X}$ and $\mathbf{Y}$ if every path between any node in $\mathbf{X}$ and any node in $\mathbf{Y}$ is blocked by $\mathbf{Z}$. \(X\) and \(Y\) are \emph{m-separated} by \(\mathbf{Z}\) if $\neg adj(X, Y)$ and all paths between $X$ and $Y$ are blocked by \(\mathbf{Z}\).

An \emph{unshielded triple} is a sequence \(\langle X, Y, Z \rangle\) where \(X\) and \(Z\) are adjacent to \(Y\) but not to each other. If it is oriented as \(X *\rightarrow Y \leftarrow* Z\), it is an \emph{unshielded collider}. These play a central role in the orientation of edges.

A triple $(X, Y, Z)$ forms a \emph{definite noncollider} at $Y$ if on of the following holds:
\begin{enumerate}
    \item The triple is unshielded and at least one of the edges $X *-* Y$ or $Z *-* Y$ has an arrowhead into $X$ or $Z$, i.e., $Y *\rightarrow X$ or $Y \rightarrow* Z$ exists.
    \item The triple is unshielded and both edges have circle endpoints at $Y$, i.e., $X \circ\!\!-\!\!\circ Y \circ\!\!-\!\!\circ Z$, and $X$ and $Z$ are not adjacent.
    \item Selection bias is allowed and $X - Y - Z$.
\end{enumerate}

For any two distinct nodes $A$ and $B$ in a maximal ancestral graph (MAG) $G$, the \emph{D-SEP} set of $A$ relative to $B$, denoted $\text{D-SEP}(A,B)$, is the set of nodes $V \neq A$ such that there exists a collider path between $A$ and $V$ in which every node is an ancestor of either $A$ or $B$ in $G$. Intuitively, $\text{D-SEP}(A,B)$ contains the variables that may need to be conditioned on to separate $A$ and $B$ when latent confounding is present. In the oracle setting, if $A$ and $B$ are not adjacent in the true MAG, then there exists a subset of $\text{D-SEP}(A,B)$ that d-separates them~\citep{spirtes2001causation}.

Given a PAG $G$, the \emph{Possible-D-SEP} set of node $A$ includes all nodes $V \neq A$ such that there exists a path $U$ between $A$ and $V$ where each triple $\langle X, Y, Z \rangle$ in $U$ satisfies:
\begin{itemize}
    \item $Y$ is a collider, or 
    \item $Y$ is not a definite non-collider and $\langle X, Y, Z \rangle$ forms a triangle in $G$.
\end{itemize}
This set may include nonadjacent nodes and is used by FCI and GFCI to search for separating sets in the presence of latent structure~\citep{spirtes2001causation}.

A \emph{discriminating path} containing nodes in the path \( p = \langle X,...,W, V_{xy}, Y \rangle \) in a PAG such that:
\begin{itemize}
    \item The path contains at least four nodes;
    \item \(X\) is not adjacent to \(Y\), but all intermediate nodes (here, \(W\) and \(V_{xy}\)), and any nodes between \(X\) and \(W\) are adjacent to \(Y\);
    \item Every node between \(X\) and \(Y\) is a collider and a parent of \(Y\).
\end{itemize}
Discriminating paths can imply that a shielded node \(V{xy}\) is a collider, helping to orient edges or, in the case of FCIT, to clarify adjacencies in the PAG. Thus, \(V_{xy}\) is the node along the path from \(x\) to \(y\) whose orientation is being discriminated.\footnote{The use of discriminating paths is primarily for orientation, though in the case of FCIT in particular, where separating sets are found by looking at longer paths, it is often necessary to know where extra colliders due to such discrimination might be postulated to exist in an evolving PAG. Consider the example in Appendix Section \ref{sec:discriminating_path_trace}, for instance.}

\paragraph{Key assumptions}
\begin{itemize}
    \item \textbf{Causal Markov.}  
          For a causally sufficient set of variables, every variable is conditionally independent of its non-descendants given its parents in the given causal DAG.
    \item \textbf{Faithfulness.}  
          All CI facts in the distribution follow from the Causal Markov condition applied to the given causal graph.
\end{itemize}
Note that violations of Faithfulness have Lebesgue measure zero if conditional independence is algebraic, though almost unfaithful violations can be common.

\section{Simple Variants of GFCI: BOSS-FCI, GRaSP-FCI, and LV-Dumb}
\label{sec:variants}

The Greedy FCI (GFCI) algorithm \citep{ogarrio2016hybrid} is a hybrid causal discovery method combining score-based search with conditional independence (CI) testing. It first estimates a Completed Partially Directed Acyclic Graph (CPDAG) using FGES, then refines it into a Partial Ancestral Graph (PAG) through FCI-style edge removal and orientation rules. This design reduces the search space while remaining robust to latent variables and selection bias.

We introduce three simple variants of GFCI that differ only in the initial scoring method and in whether the subsequent CI-based refinement is applied.

\paragraph{BOSS-FCI.}
\emph{BOSS-FCI} replaces FGES with the Best Order Score Search (BOSS) algorithm \citep{andrews2019learning}. BOSS performs permutation-based score search, offering strong accuracy and scalability—especially for dense graphs. After estimating a CPDAG, the algorithm applies the standard GFCI refinement: adjacency-based edge removal, collider identification, and the FCI orientation rules.

\paragraph{GRaSP-FCI.}
\emph{GRaSP-FCI} substitutes the GRaSP algorithm \citep{lam2022greedy} for FGES in the same framework. GRaSP employs a greedy adjacency search with penalty regularization, emphasizing sparsity and offering a complementary trade-off between accuracy and efficiency. Once the CPDAG is obtained, the usual GFCI refinement steps are applied.

\paragraph{LV-Dumb.}
\emph{LV-Dumb}\footnote{LV-Dumb has been referred to as \emph{BOSS-POD} (BOSS PAG of DAG) elsewhere, but the name ``LV-Dumb'' has stuck, so we defer.} is a deliberately simple heuristic. Like BOSS-FCI, it begins with a CPDAG estimated by BOSS, but instead of performing any CI-based refinement, it directly outputs the PAG equivalence class of the resulting DAG. The method therefore captures latent-aware patterns implied by the BOSS structure but cannot remove edges to expose bidirected relationships. Although formally incomplete, LV-Dumb is extremely fast and often achieves empirical accuracy comparable to its more principled counterparts. 

LV-Dumb cannot, however, orient bidirected or undirected edges. Part of its speed is achieved by using a reduced final orientation rule set for the DAG$\rightarrow$PAG conversion that takes advantage of these restrictions. In particular, the discriminating path rule need only be applied for a path length of 4, since longer paths require bidirected edges. Also, rules R5-R7 need not be used, since these orient undirected edges.

\paragraph{Pseudocode.}
Appendix~\ref{sec:algorithm_details} provides pseudocode for the general $*$-FCI template, 
where the initial CPDAG procedure~$\mathcal{M}$ may be FGES (standard GFCI), 
BOSS (BOSS-FCI), or GRaSP (GRaSP-FCI). 
LV-Dumb can be viewed as a degenerate instance of this template, 
omitting edge removal and collider orientation and returning the PAG of the BOSS DAG directly.

\begin{algorithm}[H]
    \label{alg:fci_backward}
    \DontPrintSemicolon
    \caption{\textsc{FCI-Backward} (template) with nested \textsc{BackwardStep}.
    \emph{Requirement on the generator} \tt{PossibleSeparatingSets}: it must satisfy the
    \emph{Coverage} condition stated below.}
    \KwIn{PAG $\mc G$ ; CI oracle $\mc T$}

    \SetKwProg{Proc}{\textbf{Procedure}}{}{}
    \Proc{FCI-Backward($\mc G, \mc T$)}{
        \Repeat{\textbf{\upshape not} \textit{modified}}{
            \textit{modified} $\gets$ \tt{BackwardStep}$(\mc G, \mc T)$\;
        }
        \Return $\mc G$\;
    }

    \vspace{1ex}
    \Proc{BackwardStep($\mc G, \mc T$)}{
        \For{$i *\!-\!* j$ \textbf{\upshape in} $\mc G$}{
            \For{$S$ \textbf{\upshape in} \tt{PossibleSeparatingSets}$(\mc G,i,j)$}{
                \If{$i \ind j \mid S$ \textbf{\upshape (via} $\mc T$\textbf{\upshape )}}{
                    $\mc G \gets$ \tt{RemoveEdge}$(\mc G,i,j)$\;
                    $\mc G \gets$ \tt{FinalZhangRules}$(\mc G)$\;
                    \Return \textit{true}\;
                }
            }
        }
        \Return \textit{false}\;
    }
\end{algorithm}e

%-------------------------------------------------
\section{The FCIT Algorithm}
\label{sec:fcit_description}
%-------------------------------------------------

Consider an algorithm templated as Algorithm~\ref{alg:fci_backward}, subject to the following conditions:

\begin{itemize}
    \item \textbf{Coverage condition for \tt{PossibleSeparatingSets}.}
    For every adjacent pair $\{x,y\}$ in the current PAG $\mc G$, if there exists a set
    $S^\star \subseteq \texttt{PossibleDsep}_{\mc G}(x,y)$ such that
    $x \ind y \mid S^\star$ (w.r.t.\ the oracle distribution),
    then the generator $\tt{PossibleSeparatingSets}(\mc G,x,y)$ must output at least one set
    $S$ with $x \ind y \mid S$.

    \item \textbf{Initial Superset.}
    The input PAG $\mc G_0$ has an adjacency set that is a superset of the true PAG’s adjacencies.

    \item \textbf{Per-deletion Reorientation and Legality.}
    After each accepted deletion, we apply the complete FCI/Zhang orientation rules to closure and revert if the result is not a legal PAG.
    This reorientation never introduces new adjacencies.

    \item \textbf{Correct Unshielded Colliders.}
    The input PAG $\mc G_0$ contains all and only the true unshielded colliders implied by the CPDAG.
    Moreover, after each accepted deletion, any additional unshielded colliders oriented by the FCI/Zhang rules are correct
    (i.e., consistent with the oracle model).
\end{itemize}

\noindent
\textbf{Correctness sketch (oracle setting).}
Assume the input PAG satisfies (i) the Initial Superset condition and (ii) the Correct Unshielded Colliders condition.
FCIT deletes an edge only when an oracle certifies $X \ind Y \mid S$, and its per-deletion orientation pass is sound,
does not create new adjacencies, and only orients true unshielded colliders (illegal states are reverted).
By the Coverage property of the separator generator (recursive blocking with PDP forbiddances and $B\setminus D$ trimming),
any remaining spurious adjacency $X *\!-\!* Y$ is eventually witnessed and removed.
Each accepted step strictly reduces the number of edges, so the process terminates.
At termination, adjacencies and unshielded colliders are correct;
applying the complete Zhang/FCI orientation rules yields the correct (maximally oriented) PAG.

\vspace{1em}
\noindent
We now instantiate this FCI-Backward template as the \textbf{FCIT (Fast Causal Inference Targeted-Testing)} algorithm, designed to mitigate the repeated-testing problem in FCI-style searches, including other hybrid score-and-test methods.
Following GFCI~\citep{ogarrio2016hybrid}, FCIT begins with a score-based procedure that outputs a well-formed CPDAG—e.g., FGES, BOSS, GRaSP, or for small models SP. By default, we use BOSS for its scalability and accuracy.

As noted by \citet{ogarrio2016hybrid}, score-based CPDAGs can include extra adjacencies or incorrect orientations when latent variables are present, since a CPDAG cannot represent bidirected edges. For instance, the MAG \(X \to Y \leftrightarrow Z \leftarrow W\) may yield \(X \to Y \to Z \leftarrow W\) (with an extra \(X \to Z\)) or \(X \to Y \leftarrow Z \leftarrow W\) (with an extra \(W \to Y\)). GFCI corrects these by verifying independence (\(X \ind Z\) or \(W \ind Y\)) and removing the spurious edge, restoring the correct collider structure. FCIT follows the same principle but uses \emph{targeted refinement}—fewer, higher-value CI tests—while guaranteeing a sound PAG.

Algorithm~\ref{alg:fci_tt} summarizes the FCIT procedure (see notation in Table~\ref{tab:notation}). The initial CPDAG is converted to a PAG by replacing each edge with \(\circ\!-\!\circ\) and transferring the unshielded colliders identified in the CPDAG. This is justified because the final PAG must be a subgraph of the CPDAG’s adjacencies, and collider orientations are the only ones that carry over reliably~\citep{ogarrio2016hybrid}.

The refinement phase employs the \texttt{block\_paths\_recursively} method
(Algorithm~\ref{alg:block_paths_recursively}, Appendix~\ref{sec:algorithm_details}).
Rather than exhaustively test subsets of adjacents or Possible-D-SEP sets,
this procedure analyzes individual paths between nodes to construct blocking sets on the fly, selecting at each step a small set of conditional-independence tests to perform, at least one of which is guaranteed to be judged independent precisely in those cases where FCI would find an independence. This strategy substantially reduces the search space, mitigating repeated testing and improving statistical efficiency.

When evaluating an edge \(x *\!-\!* y\), FCIT enumerates \emph{pre-discriminating paths} (PDPs)—paths that could become discriminating for \(\{x,y\}\) if the edge \(x *\!-\!* y\) were removed.\footnote{Appendix~\ref{app:discpaths} gives the efficient algorithm for listing discriminating and pre-discriminating paths.} These yield a conservative superset of the true discriminating paths and corresponding $V$-node candidates $\mathsf{V}_{\text{cand},xy}$. Enumerating these supersets allows edges to be processed independently and safely in parallel. For each subset \(F \subseteq \mathsf{V}_{\text{cand},xy}\) (including $\emptyset$), FCIT runs recursive blocking to propose candidate separating sets $B\setminus D$ for CI testing. Even when no PDPs are found, FCIT still checks independence with \(F = \emptyset\).

Thus, in FCIT, the FCI-Backward generator is instantiated via BOSS initialization and recursive blocking with PDP-based forbiddances and trimming by common adjacents. We prove this separator generator satisfies the Coverage condition. Since FCIT starts from a CPDAG with correct colliders, deletes edges only when $X \ind Y \mid S$ (oracle), and legality-checks each reorientation, every collider added is correct. By Coverage and the Initial Superset invariant, FCIT removes all and only spurious adjacencies and, after the final Zhang pass, returns the correct PAG.

\begin{algorithm}
\caption{FCIT (with pre-discriminating paths and recursive blocking)}
\label{alg:fci_tt}
\begin{algorithmic}[1]
\STATE \textbf{Input:} data, score method (e.g., BOSS), CI test $\mathcal T$
\STATE $C \gets \text{BOSS}(\text{data})$ \COMMENT{initial CPDAG}
\STATE $G \gets C$ \COMMENT{initialize PAG}
\STATE Reset endpoints in $G$ to $\circ\!-\!\circ$
\STATE Copy known unshielded colliders in $G$ from $C$
\STATE Apply final FCI orientation rules on $G$, yielding a PAG
\STATE $\mathit{converged} \gets \textbf{false}$

\WHILE{not $\mathit{converged}$}
  \STATE $\mathit{converged} \gets \textbf{true}$
  \FORALL{edges $x *\!-\!* y$ in $G$}
    \STATE $\mathsf{PDP}_{xy} \gets \textsc{ListPreDiscriminatingPaths}(G,x,y)$
    \STATE $\mathbf{V}_{\text{cand},xy} \gets \bigcup_{p\in \mathsf{PDP}_{xy}} V(p)$ \COMMENT{superset of true $\mathbf{V}_{xy}$}
    \STATE $\mathbf{C}_{xy} \gets \AdjG{x}\cap \AdjG{y}$ \COMMENT{common adjacents}
    \STATE \textbf{for each} $\mathbf{F} \subseteq \mathbf{V}_{\text{cand},xy}$ \textbf{(including $\emptyset$)} \COMMENT{$\mathbf{F}$ = forbiddances/not-followed from PDPs}
      \STATE \quad $\mathbf{B} \gets \textsc{block\_paths\_recursively}(G,x,y,\emptyset,\mathbf{F})$
      \STATE \quad \textbf{for each} $\mathbf{D} \subseteq (\mathbf{C}_{xy}\cap \mathbf{B})$
      \STATE \qquad \textbf{if } $\exists d\in \mathbf{D}$ \textbf{ with } $x \to d \leftarrow y$ \textbf{ then continue} \COMMENT{never drop known collider}
      \STATE \qquad $\mathbf{S} \gets \mathbf{B} \setminus \mathbf{D}$
      \STATE \qquad \textbf{if } $x \ind y \mid \mathbf{S}$ \textbf{ (via }$\mathcal T$\textbf{) then}
      \STATE \qquad\quad Remove $x *\!-\!* y$ from $G$
      \STATE \qquad\quad \textbf{if } $G$ is not a legal PAG \textbf{ then } revert; \textbf{continue}
      \STATE \qquad\quad $\mathsf{Sep}(\{x,y\}) \gets \mathbf{S}$ \COMMENT{cache sepset for this deletion}
      \STATE \qquad\quad Reset endpoints in $G$ to $\circ\!-\!\circ$
      \STATE \qquad\quad Refresh unshielded colliders in $G$ from known list; reapply orientation rules
      \STATE \qquad\quad $\mathit{converged} \gets \textbf{false}$
      \STATE \qquad\quad \textbf{break both loops}
  \ENDFOR
\ENDWHILE
\STATE \textbf{Output:} PAG $G$
\end{algorithmic}
\end{algorithm}

\begin{table}[h!]
\centering
\caption{Notation used in FCIT and recursive blocking}
\label{tab:notation}
\begin{tabular}{ll}
\toprule
Symbol / Name & Meaning \\
\midrule
$G=(\textbf{V},\textbf{E})$ & Current PAG/MAG over observed variables $V$ \\
$x,y$ & Distinct query endpoints in $V$ \\
$\mathbf{C}$ & Required conditioning set (``containing'') \\
$\mathbf{F}$ & Forbidden nodes (``notFollowing'') \\
$\mathbf{B}$ & Blocking set built by \texttt{block\_paths\_recursively} \\
$\mathbf{uc}$ & Set of unshielded colliders from CPDAG or updates \\
$\mathrm{Adj}_G(v)$ & Adjacency set of node $v$ in $G$ \\
$\mathbf{C}_{xy}$ & Common adjacents of $x$ and $y$ \\
$\mathbf{V}_{xy}$ & Candidate colliders on discriminating paths w.r.t.\ $(x,y)$ \\
$\mathbf{Extra}(x,y)$ & Certified separating set used to remove $x *{-}* y$ \\
$\texttt{Knowledge}$ & Forbidden/required directed edge constraints \\
$\mathbf{path}$ & DFS frontier set to prevent revisits (cycle guard) \\
$\text{m-sep}(x,y\mid \mathbf{S})$ & $x$ and $y$ are m-separated by $\mathbf{S}$ in $G$ \\
\bottomrule
\end{tabular}
\end{table}

\section{Optimizing the Final FCI Orientation Rules}
\label{sec:optimizing_final_rules}

A key feature of FCIT is that the final FCI orientation rules \citep{zhang2008completeness} are applied repeatedly throughout the refinement process rather than only once at the end. This ensures maximally informative orientations but introduces performance bottlenecks, as several rules are computationally expensive in their standard form. To make FCIT scalable, we optimize the slowest of these rules. Because the same rules are shared across all FCI-style algorithms, the resulting improvements benefit the entire family, including FCI, GFCI, BOSS-FCI, and GRaSP-FCI.

We focus on four rules that dominate runtime in practice. Full details of their optimized implementations are provided in Appendix~\ref{sec:optimizing_zhang_rules}.

% \begin{definition}[Uncovered Path]
% A path $\pi = \langle v_0, v_1, \dots, v_k \rangle$ in a graph $G$ is \emph{uncovered} if for every $i \in \{1, \dots, k-1\}$, the nonconsecutive pair $(v_{i-1}, v_{i+1})$ is not adjacent in $G$:
% \[
% \forall i \in \{1,\dots,k-1\},\; \neg adj(v_{i-1},v_{i+1}).
% \]
% \end{definition}

\begin{itemize}
    \item \textbf{Rule R4} (\emph{discriminating paths}).  
    Identifying discriminating paths traditionally involves testing subsets of adjacents for nonadjacent pairs. We replace this with the \texttt{block\_paths\_recursively} procedure, which derives separating sets directly from path structure, avoiding redundant subset enumeration.
    
    \item \textbf{Rules R5 and R9} (\emph{all-circle-endpoint paths}).  
    These rules require scanning paths with circle endpoints of potentially unbounded length. We introduce a modified Dijkstra-style shortest-path search that efficiently detects such circle-path structures without exhaustive traversal.
    
    \item \textbf{Rule R10} (\emph{semidirected paths}).  
    Detecting uncovered, potentially directed paths can be costly.  
    A path $\langle v_0, v_1, \dots, v_k \rangle$ is \emph{uncovered} if no nonconsecutive pair $(v_{i-1},v_{i+1})$ is adjacent in $G$.  
    We implement a memoized depth-first search that enforces uncoveredness and directional constraints exactly while pruning redundant traversals.
\end{itemize}

These optimizations substantially reduce both the number of CI tests and the number of path enumerations required. In FCIT—where orientation rules are invoked repeatedly—these savings compound to large runtime improvements. The same implementations can be applied directly to other FCI-style algorithms, improving efficiency without altering correctness.\footnote{For consistency with the original algorithm definition, for FCI we use the original adjacency-based discriminating path rule, though for the other algorithms we use the recursive one.}

%=================================================
\section{Theoretical Guarantees}
\label{sec:theory}
%=================================================

We outline the guarantees enjoyed by our introduced algorithms, including FCIT, under the standard \emph{Causal Markov} and \emph{Faithfulness} assumptions (Section~\ref{sec:graphical_preliminaries}), assuming access to an \emph{oracle} conditional independence (CI) test. Finite-sample behavior is examined in Section~\ref{sec:empirical_eval}.

We first summarize the correctness status of the simpler variants introduced in Section~\ref{sec:variants}, then turn to FCIT itself. For particular choices of score and/or test, the procedure will make any additional assumptions which chosen score and/or test themselves make, though correctness of the procedures themselves, given non-parametric tests or scores, is itself non-parametric.

%-------------------------------------------------
\subsection{Correctness of BOSS-FCI and GRaSP-FCI}
%-------------------------------------------------

BOSS-FCI and GRaSP-FCI inherit the correctness guarantees of the GFCI algorithm \citep{ogarrio2016hybrid}. In its original form, GFCI included an additional Possible-D-SEP search step that ensures oracle correctness. Although this step is often omitted in practice, and no counterexamples have been observed even under a $d$-separation oracle, a formal proof of its redundancy remains open. With that step included, BOSS-FCI and GRaSP-FCI are correct for the same reason as GFCI, differing only in the score-based method used for the initial CPDAG.

%-------------------------------------------------
\subsection{Correctness of LV-Dumb}
%-------------------------------------------------

By contrast, \textsc{LV-Dumb} lacks formal correctness guarantees. It reports only the PAG equivalence class of the DAG returned by BOSS and therefore cannot remove edges to expose latent confounding or orient bidirected edges. Nevertheless, it always produces a structurally well-formed PAG and, as shown in Section~\ref{sec:empirical_eval}, performs well on adjacency and orientation accuracy measures. 

To provide intuition for why \textsc{LV-Dumb}, while incorrect, nevertheless achieves strong performance on PAG accuracy measures, we note the following considerations.  
It is unsurprising that \textsc{LV-Dumb} performs well on arrowhead precision for common adjacencies (AHPC), since it should in principle produce a subset of the arrowheads that exist over the adjacencies shared with the true PAG.  
Similarly, adjacency precision (AP) tends to be high: by GFCI theory, \textsc{LV-Dumb} should add only a small fraction of ``extra'' adjacencies, primarily to cover colliders present in the true PAG.  
Arrowhead recall (AHR) is also expected to be good if the hidden arrowheads represent only a small proportion of those in the true PAG.  

Taken together with the high accuracy of BOSS for producing estimates of CPDAGs, these and other such        effects explain why \textsc{LV-Dumb} generally yields graphs that are very close to the true PAG—especially in large models—despite its lack of formal guarantees.\footnote{We did not pursue, but plan to in future work, estimating effects of direct intervention.}

%-------------------------------------------------
\subsection{Correctness of FCIT}
%-------------------------------------------------

FCIT introduces a recursive blocking procedure that replaces adjacency-only conditioning and integrates discriminating paths 
directly into edge removal. The following results summarize its guarantees in the oracle setting; detailed proofs appear 
in Appendix~\ref{sec:rb-soundness} and \ref{sec:rb-completeness}.

\begin{theorem}[Soundness of Recursive Blocking]\label{thm:rb-sound}
Let $x \neq y$ be distinct measured vertices in a graph $G$, where $G$ is a PAG (or, more generally, a MAG or DAG) interpreted under $m$-separation.
Let $\mathit{block\_paths\_recursively}(x,y,C,\mathbf{F})$ denote
Algorithm~\ref{alg:block_paths_recursively}. If the algorithm halts and
returns a blocking set $\mathbf{B}$, then either $x$ and $y$ are adjacent, or every path
$p$ from $x$ to $y$ in $G$ is blocked by $\mathbf{B}$:
\[
adj(x, y) \;\;\vee\;\;
\forall p.\; \Path(x,y,p) \Rightarrow \Blocked(p,\mathbf{B}).
\]
\end{theorem}

\begin{corollary}[Adjacency Case]\label{cor:rb-adj}
If $\mathit{block\_paths\_recursively}(x,y,C,\mathbf{F})$ halts and returns a non-null
$\mathbf{B}$ with $adj(x,y)$ in a MAG or PAG $G$, then $\mathbf{B}$ blocks all $x$–$y$ paths except the direct
single-edge path:
\[
adj(x,y) \wedge \mathbf{B} \neq \textsc{Null}
\;\Longrightarrow\;
\forall p.\; \Path(x,y,p) \wedge |p|\!\ge\!2
\Rightarrow
\Blocked(p,\mathbf{B}).
\]
\end{corollary}

\emph{Intuition.}
When $x$ and $y$ are adjacent in a MAG, the recursive procedure still verifies that
every indirect route $x,b,\ldots,y$ through any neighbor $b\!\neq\!y$
is $m$-blocked by $\mathbf{B}$. If the algorithm returns a non-null $\mathbf{B}$, all non-inducing paths are closed, leaving only the direct $x *\!-\!* y$ link potentially open in the corresponding MAG or PAG representation.

%-------------------------------------------------
\subsection{Edge-Minimality of the Returned PAG}
\label{sec:edge_minimality}
%-------------------------------------------------

\begin{lemma}[RB completeness over $\boldsymbol{\mathcal{S}}$]\label{lem:rb-complete}
If there exists $\mathbf{S}\in\boldsymbol{\mathcal{S}}(X,Y)$ such that $X \ind Y \mid \mathbf{S}$ (in the oracle PAG),
then $\mathit{block\_paths\_recursively}(X,Y,C,\textbf{F})$ halts and returns some blocking set
$\mathbf{B} \subseteq \boldsymbol{\mathcal{S}}(X,Y)=\mathrm{Adj}(X)\cup\mathrm{Adj}(Y)
\cup\mathrm{Possible\text{-}D\text{-}SEP}(X,Y)$ with $X \ind Y \mid \mathbf{B}$.
\end{lemma}

\emph{Intuition.}
The recursive blocking procedure explores all open paths between $X$ and $Y$ and
adds noncolliders that lie on those paths until every path is $m$-blocked.
Every node it ever considers conditioning on lies on a path between $X$ and $Y$
whose internal vertices are either colliders or ancestors of colliders—exactly
the definition of \emph{Possible-D-SEP}.
Hence, any node that recursive blocking may add is already contained in
$\boldsymbol{\mathcal{S}}(X,Y)$.  
If some separating set $\mathbf{S}\in\boldsymbol{\mathcal{S}}(X,Y)$ exists, recursive blocking will
therefore discover a (possibly different) subset $\mathbf{B}\subseteq \boldsymbol{\mathcal{S}}(X,Y)$
that also blocks all $X$–$Y$ paths.

\begin{theorem}[Edge-Minimality]\label{thm:edge_min}
Let $G$ be the true PAG and $G'$ the PAG returned by FCIT.
If an edge $X *\!-\!* Y$ remains in $G'$, then there exists no conditioning set
$\mathbf{S}$ from the standard FCI search space
$\boldsymbol{\mathcal{S}}(X,Y)=\mathrm{Adj}(X)\cup\mathrm{Adj}(Y)\cup\mathrm{Possible\text{-}D\text{-}SEP}(X,Y)$
that $d$-separates $X$ and $Y$.
\end{theorem}

\noindent
\emph{Proof sketch.}
If such a $\mathbf{S}$ existed, Lemma~\ref{lem:rb-complete} implies that
$\mathit{block\_paths\_recursively}(X,Y,C,\textbf{F})$ would find some blocking set
$\mathbf{B} \subseteq \boldsymbol{\mathcal{S}}(X,Y)$ with $X \ind Y \mid \mathbf{B}$.
FCIT would then test this independence and remove the edge.
Since FCIT removes no other edges, any adjacency that remains corresponds to a
pair $(X,Y)$ for which no valid separating set exists within $\boldsymbol{\mathcal{S}}(X,Y)$.
Hence, $G'$ is edge-minimal.

% \noindent
% \emph{Proof sketch.}
% Suppose such an $S$ existed. By the completeness of the recursive blocking procedure (Appendix~\ref{sec:rb-completeness}),
% FCIT would discover some blocking set $B \subseteq \mathcal{S}(X,Y)$ that also $d$-separates $X$ and $Y$. The 
% algorithm would then remove the edge after confirming $X \ind Y \mid B$ by the independence test. Since FCIT 
% removes no other edges, every adjacency that remains corresponds to a pair $(X,Y)$ for which no valid separating 
% set exists within $\mathcal{S}(X,Y)$. Hence $G’$ is edge-minimal.

%-------------------------------------------------
\subsection{Soundness of Orientations}
\label{sec:soundness_orient}
%-------------------------------------------------

\begin{theorem}[Orientation Soundness]\label{thm:orient_sound}
Let $G$ be the true PAG and $G'$ the PAG returned by FCIT.
Every oriented edge in $G'$ (arrowhead or tail) appears with the same
orientation in $G$.
\end{theorem}

\noindent
\emph{Proof sketch.}
FCIT uses exactly the arrow- and tail-complete rule set R1–R10 of
\citet{zhang2008completeness}.  
Given correct adjacencies and unshielded colliders,  
Spirtes and Zhang prove each rule sound;  
hence all orientations in $G'$ are valid.

%-------------------------------------------------
\subsection{Optional Completeness of Orientations}
\label{sec:optional_completeness}
%-------------------------------------------------

\begin{theorem}[Orientation Completeness]\label{thm:orient_complete}
FCIT produces a tail- and arrow-complete PAG (i.e., every orientation entailed
by Markov+Faithfulness appears) \emph{iff} the full rule set R1–R10 is applied
to closure.  
If any rule is disabled or the process halts early, the resulting PAG may be
partially oriented, but all stated orientations remain sound by
Theorem~\ref{thm:orient_sound}.
\end{theorem}

%-------------------------------------------------
\subsection{Summary Corollary}
%-------------------------------------------------

\begin{corollary}
Under Causal Markov and Faithfulness, the PAG $G'$ returned by FCIT  
(i) is Markov to the observed distribution,  
(ii) is edge-minimal with respect to the standard FCI conditioning space, and  
(iii) contains only sound orientations.  
If all rules R1–R10 are applied to closure, $G'$ is additionally tail- and
arrow-complete.
\end{corollary}

\section{Well-formedness of PAGs and Markov}

One observes that FCIT and LV-Dumb always return graphs that are well-formed, in the sense that they satisfy strictly the definition of a PAG. (For a simulation comparison of how well these various algorithms recover PAGs, consider Figure \ref{fig:pag20} and related figures in the Appendix.) We can determine this independently by checking whether the identified Zhang MAG calculated from estimated graph \citep{zhang2008completeness} satisfies the definition of a MAG and whether when converted back to a PAG the orientations of the estimated PAG are recovered precisely. That is, we check the following

\begin{itemize}
  \item \textbf{No Almost-Cycles or Directed Cycles.}
  We check in the Zhang MAG whether any bidirected edge $A \leftrightarrow B$ has a directed path $A \leadsto B$ or $B \leadsto A$ or whether there are any directed cycles from nodes to themselves. If not, the final graph cannot exhibit an almost-cycle or a directed cycle in the induced MAG.

  \item \textbf{Maximality.}
  Next, we check to make sure the graph does not contain any inducing paths from nodes $A$ to $B$ in the graph while $\neg adj(A, B)$, in the Zhang MAG of the estimated graph. If not, then the MAG is maximal.

  \item \textbf{Orientation Rule Compliance.}
  Finally, we check whether converting the Zhang MAG back into the equivalence class PAG it is a member of yields a graph that has the same endpoint markings as the original estimated graph. If not, then the graph lies somewhere between a MAG and a PAG.
\end{itemize}

These three conditions are individually necessary and jointly sufficient for a singly connected graph over measured variables alone to be a structurally well-formed PAG; we can therefore verify well-formedness by checking these conditions directly.\footnote{The Tetrad suite of software contains such a check.} By construction, FCIT and LV-Dumb always return structurally well-formed PAGs, whereas empirically we find that other FCI-style procedures do not always do so.

In particular:

\begin{itemize}
  \item \textbf{FCIT.}  
  By construction, FCIT checks every proposed update for structural validity, ensuring that the algorithm always moves among well-formed PAGs. Consequently, its output is guaranteed to be structural well-formed regardless of unfaithfulness or finite-sample noise.

  \item \textbf{LV-Dumb.}  
  LV-Dumb trivially returns a well-formed PAG by construction, since it outputs the PAG equivalence class of the BOSS output, which is guaranteed to be a DAG (as BOSS is an algorithm that infers a DAG from a particular permutation of the variables).

  \item \textbf{BOSS-FCI GRaSP-FCI, and GFCI.}  
  Under unfaithfulness, these algorithms can return malformed graphs in practice.

  \item \textbf{FCI.}
  Under unfaithfulness, FCI can also produce a graph that is not a well-formed PAG.
\end{itemize}

Note that each of these algorithms in Oracle mode is guaranteed to produce a well-formed PAG.

We next compare FCIT with FCI, GFCI, BOSS-FCI, GRaSP-FCI, and the heuristic LV-Dumb using simulations implemented in Tetrad \citep{ramsey2018tetrad},  focusing on accuracy and efficiency in the presence of latent variables. 

\section{Empirical Evaluation}
\label{sec:empirical_eval}

%----------------------------------------
\subsection{Simulation Results}
%-------------------------------------------------

%-------------------------------------------------
\subsubsection{Simulation Setup}
%-------------------------------------------------

We simulated data from linear Gaussian structural equation models generated from random forward DAGs with latent variables.\footnote{As suggested in an earlier footnote, this is a modeling choice in Tetrad. Scores and tests are also available for the multinomial case, the mixed multinomial/Gaussian case, and the mixed multinomial/post-nonlinear case (see \citep{ramsey2025scalablecausaldiscoveryrecursive}. Also, it should be noted that linear Gaussian scores and tests also work well in the linear non-Gaussian regime \citep{andrews2019learning}). Tests are also available for the general non-parametric case, though not scores currently; these are less scalable.} The data were produced using a standard linear Gaussian simulation method (\texttt{SemSimulation} in Tetrad) with the following design:

\begin{itemize}
  \item \textbf{Graph size and density.} 
    $20$ measured variables with average degree $2$, $4$, or $6$. 
    For larger regimes we used $100$ measured variables (avg. deg. $6$, $10$ latents) and a single stress test with $200$ variables (avg. deg. $6$, $20$ latents).
  \item \textbf{Latent structure.} 
    Each $20$-node run included $0$, $4$, or $8$ latent variables in addition to the measured ones. 
    Latent variables were not saved out but contributed to inducing dependence among observed variables.
  \item \textbf{Coefficients.} 
    Edge coefficients were drawn uniformly from $[-1, 1]$, allowing both positive and negative effects. (Note that we do not use an interval about zero.)
  \item \textbf{Variances.} 
    Independent noise variances were sampled from $U(1,3)$; no additional measurement noise was added.
  \item \textbf{Sample sizes.} 
    $N = 200, 500, 1000, 5000$ for the $20$-node grids. 
    The $100$- and $200$-node settings used $N = 1000$.
  \item \textbf{Replicates.} 
    Each condition was repeated $20$ times with randomized column order.
\end{itemize}

Throughout, the score was the linear Gaussian BIC score with penalty discount $2.0$, and the conditional independence test was Fisher’s $Z$-test at $\alpha = 0.01$. These parameters did not vary across conditions. Graphs were not required to be connected, and maximum degree limits were set sufficiently high to avoid constraining the generated structures.

%-------------------------------------------------
\subsubsection{Evaluation Metrics}
%-------------------------------------------------

We report standard adjacency and orientation measures, along with three path-aware precision statistics that compare estimated endpoints to directed reachability in the \emph{true} graph.

\begin{itemize}
  \item \textbf{Adjacency}: 
    precision and recall for adjacency recovery (AP, AR).
  \item \textbf{Arrowheads and tails}: 
    precision and recall for arrow correctness (AHP, AHR); precision and recall for arrow correctness for adjacencies in common with the true graph (AHPC, AHRC).
  \item \textbf{Arrow Path Precision} (\emph{no back-reach}): 
    proportion of estimated edges $X *\!\to Y$ such that there is \emph{no} directed path $Y \leadsto X$ in the true DAG ($\text{*->-Prec}$)
  \item \textbf{Tail Path Precision} (\emph{forward reach}): 
    proportion of estimated edges $X \to Y$ such that there \emph{is} a directed path $X \leadsto Y$ in the true DAG. ($\text{-->-Prec}$)
  \item \textbf{Bidirected Latent Precision}: 
    proportion of estimated $X \leftrightarrow Y$ edges for which a latent variable $L$ exists with $X \leftarrow L \rightarrow Y$ in the generating DAG ($\text{<->-Lat-Prec}$)
  \item \textbf{Well-formed PAGs}: 
    proportion of discovered graphs that satisfy the structural definition of a valid PAG.
\end{itemize}

For AP, AR, AHP, AHR, AHPC, and AHRC, estimated graphs are compared to the true PAG; for $\text{*->Prec}$, $\text{-->-Prec}$, and $\text{<->-Lat-Prec}$, graphs are compared to the true DAG (with latents).

We also report \textbf{runtime} (E-CPU) to assess computational efficiency.

\subsubsection{20 nodes.}
Results for the 20-node case with average degree 4 are shown in Figures~\ref{fig:ap20}–\ref{fig:cpu20}. Each plotted point is the average of 20 runs per condition, with sample sizes $N \in \{200, 500, 1000, 5000\}$ (log-scaled on the x-axis) and averaged across 0, 4, and 8 latent common causes. All metrics here compare each estimated graph to the true PAG.

All methods achieve very high adjacency precision. \textsc{FCIT} and \textsc{LV-Dumb} consistently lead most orientation-precision metrics (AP/AHP, AHPC, arrow-path precision) and always return valid PAGs. Tail-path precision is highest for \textsc{BOSS-FCI} and \textsc{GFCI} at large $N$, with \textsc{FCIT} close behind. \textsc{FCIT} attains the best bidirected latent-path precision at large samples, while \textsc{LV-Dumb}—by design—does not orient bidirected edges. In runtime, \textsc{LV-Dumb} (and typically \textsc{BOSS-FCI}) are fastest; \textsc{GFCI} is slowest; \textsc{FCIT} and \textsc{GRaSP-FCI} fall in between. Results for average degrees 2 and 6 (see Appendix) show qualitatively similar patterns, with denser graphs highlighting the tail-path strengths of \textsc{BOSS-FCI}/\textsc{GFCI} and the robustness of \textsc{FCIT}/\textsc{LV-Dumb} on PAG validity.

\subsubsection{100 nodes.}
Results for the 100-node case are summarized in Table~\ref{tab:100node}. These models have 10 latent variables, average degree 10, and $N=1000$, using the same score and test as in the 20-node experiments; values are averaged over 20 runs. Because DAG$\rightarrow$PAG conversion is computationally expensive at this scale, we do not evaluate against the true PAG. Instead, we report only statistics that compare each algorithm’s estimated PAG to the ground-truth DAG (i.e., *->-Prec, -->-Prec, and <->-Lat-Prec), along with CPU time (ms) and whether the returned graph is a legal PAG. To keep runtimes reasonable, we capped the maximum conditioning-set size (“depth”) in all conditional independence tests at 7.

At this larger scale, \textsc{LV-Dumb} and \textsc{FCIT} both maintain exceptionally high arrow-path precision ($\approx0.99$) and consistently produce structurally valid PAGs, though their tail-path precision is somewhat lower ($\approx0.93$). In contrast, \textsc{BOSS-FCI} and \textsc{GRaSP-FCI} display the opposite trade-off—achieving very high tail-path precision ($\approx0.97$) but lower arrow-path precision ($\approx0.89$). Their PAG validity, however, drops to roughly 0.10–0.25. \textsc{GFCI} performs comparably on both arrow and tail paths but with reduced latent-edge precision and moderate PAG validity (0.20), whereas \textsc{FCI} shows the weakest precision scores overall but remains by far the fastest method. Bidirected-edge precision remains low for all methods at this scale, suggesting that in large models such edges should be interpreted cautiously or reserved for qualitative guidance. Overall, these results mirror the smaller-graph findings: \textsc{FCIT} offers the best balance of accuracy, scalability, and structural validity, while \textsc{LV-Dumb} remains the most efficient and structurally stable baseline.

\subsubsection{Takeaway.}
Across sparsity levels, \textsc{FCIT} and \textsc{LV-Dumb} remain the most accurate and stable overall. Both consistently produce valid PAGs, sustain near-perfect arrow- and tail-path precision, and recover adjacencies reliably even under latent confounding. \textsc{LV-Dumb} matches \textsc{FCIT}’s orientation accuracy while running substantially faster, making it a strong baseline when latent structure need not be modeled explicitly.
\textsc{BOSS-FCI} and \textsc{GRaSP-FCI} maintain high directional precision but occasionally yield malformed PAGs and incur heavier computational cost at scale. \textsc{GFCI} continues to perform reasonably well on directed edges yet shows moderate PAG validity and weaker latent-edge precision, while standard \textsc{FCI} remains the fastest but the least accurate overall.
In sum, \textsc{FCIT} offers the best balance of correctness, scalability, and structural validity among algorithms that consistently return legal PAGs across graph densities.

\begin{figure}
    \centering
    \includegraphics[width=0.5\linewidth]{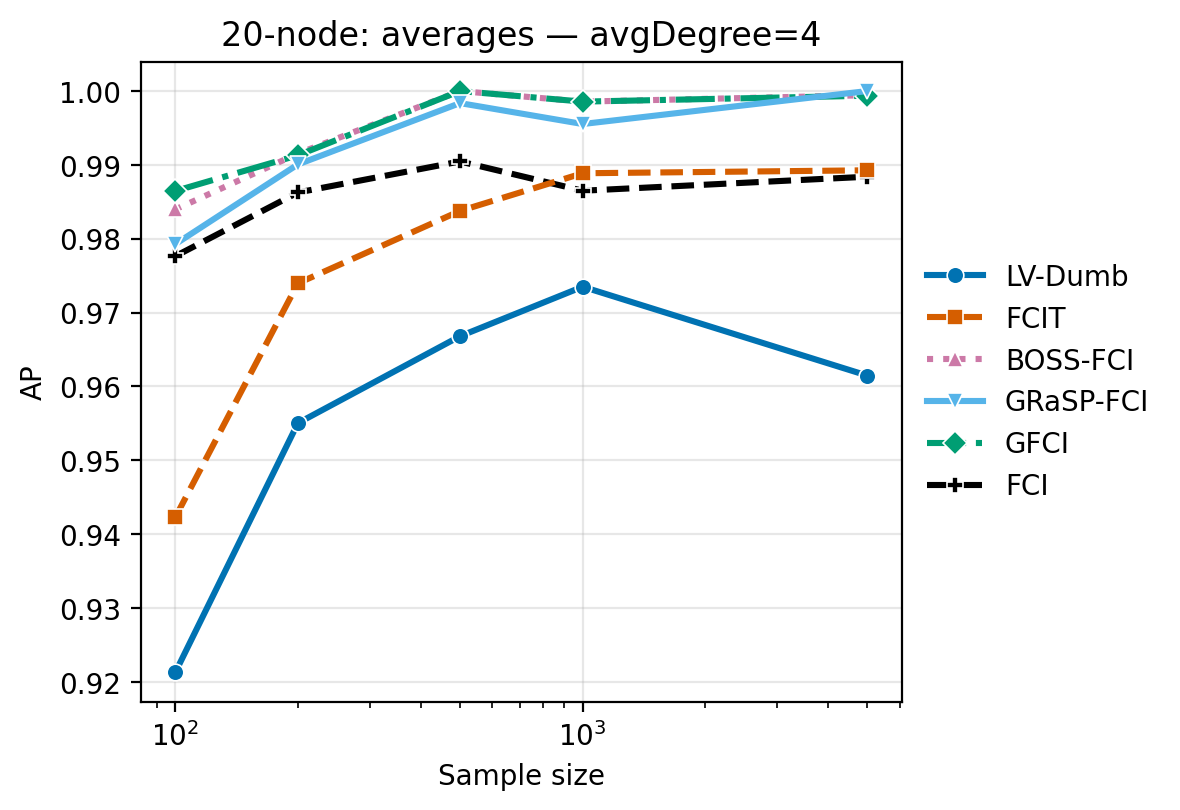}
    \caption{Adjacency Precision (AP) for 20-node graphs with average degree 4. 
    All algorithms maintain very high adjacency precision across sample sizes. LV-Dumb trails slightly.
    \textsc{LV-Dumb} is slightly lower than the others but remains acceptable.}
    \label{fig:ap20}
\end{figure}

\begin{figure}
    \centering
    \includegraphics[width=0.5\linewidth]{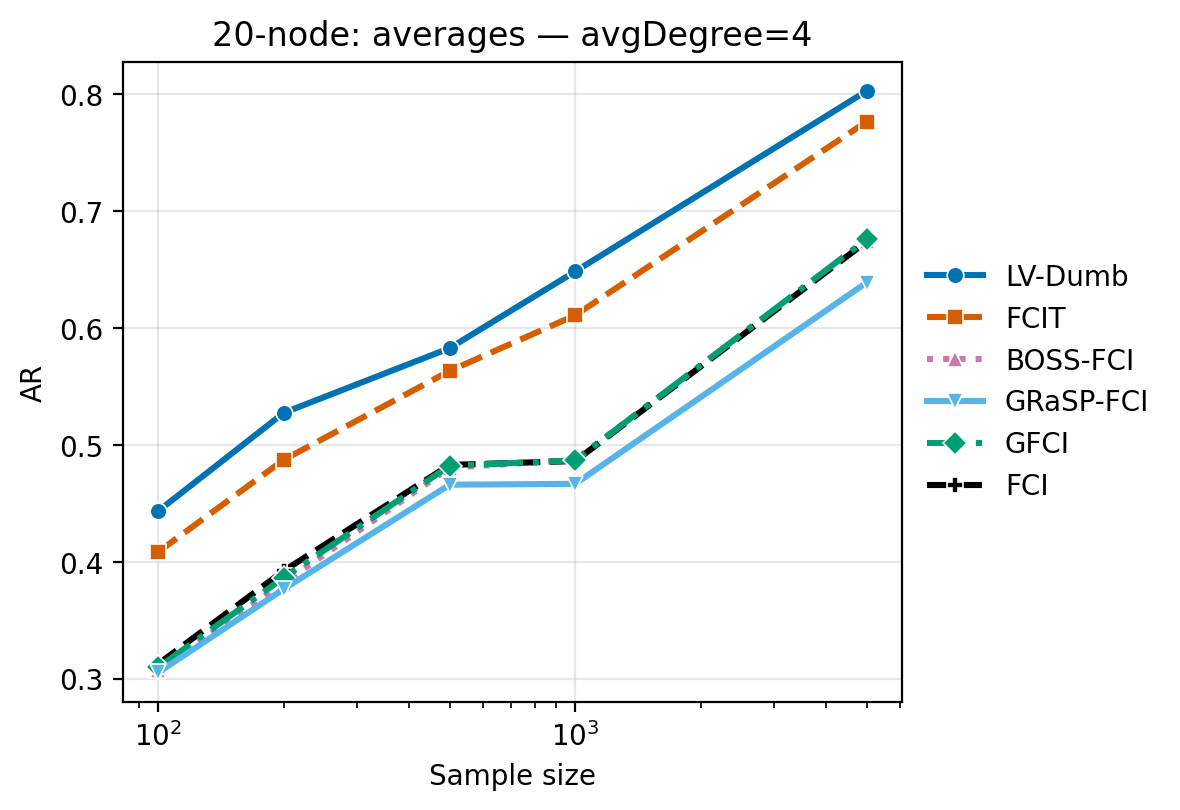}
    \caption{Adjacency Recall (AR) for 20-node graphs with average degree 4. 
    \textsc{LV-Dumb} and \textsc{FCIT} achieve the highest recall across all sample sizes, 
    reflecting their ability to recover true adjacencies even in the presence of latent confounding. 
    \textsc{BOSS-FCI}, \textsc{GRaSP-FCI}, \textsc{GFCI}, and \textsc{FCI} improve steadily with $N$. 
    Results are averaged over graphs with 0, 4, and 8 latent common causes.}
    \label{fig:ar20}
\end{figure}

\begin{figure}
    \centering
    \includegraphics[width=0.5\linewidth]{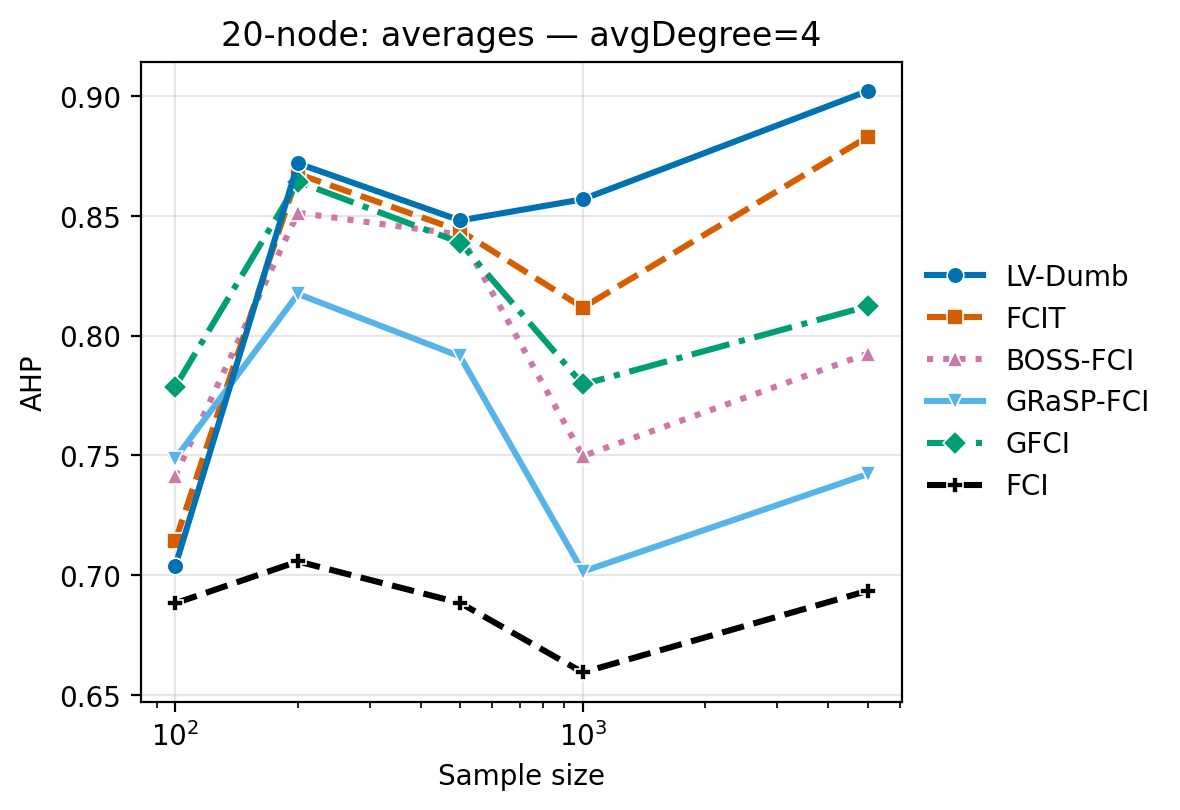}
    \caption{Arrowhead Precision (AHP) for 20-node graphs with average degree 4. 
    \textsc{LV-Dumb} and \textsc{FCIT} maintain the highest precision, followed by \textsc{BOSS-FCI} and \textsc{GRaSP-FCI}. 
    \textsc{FCI} remains lowest throughout, while \textsc{GFCI} improves moderately at large $N$. 
    Results are averaged over graphs with 0, 4, and 8 latent common causes.}
    \label{fig:ahp20}
\end{figure}

\begin{figure}
    \centering
    \includegraphics[width=0.5\linewidth]{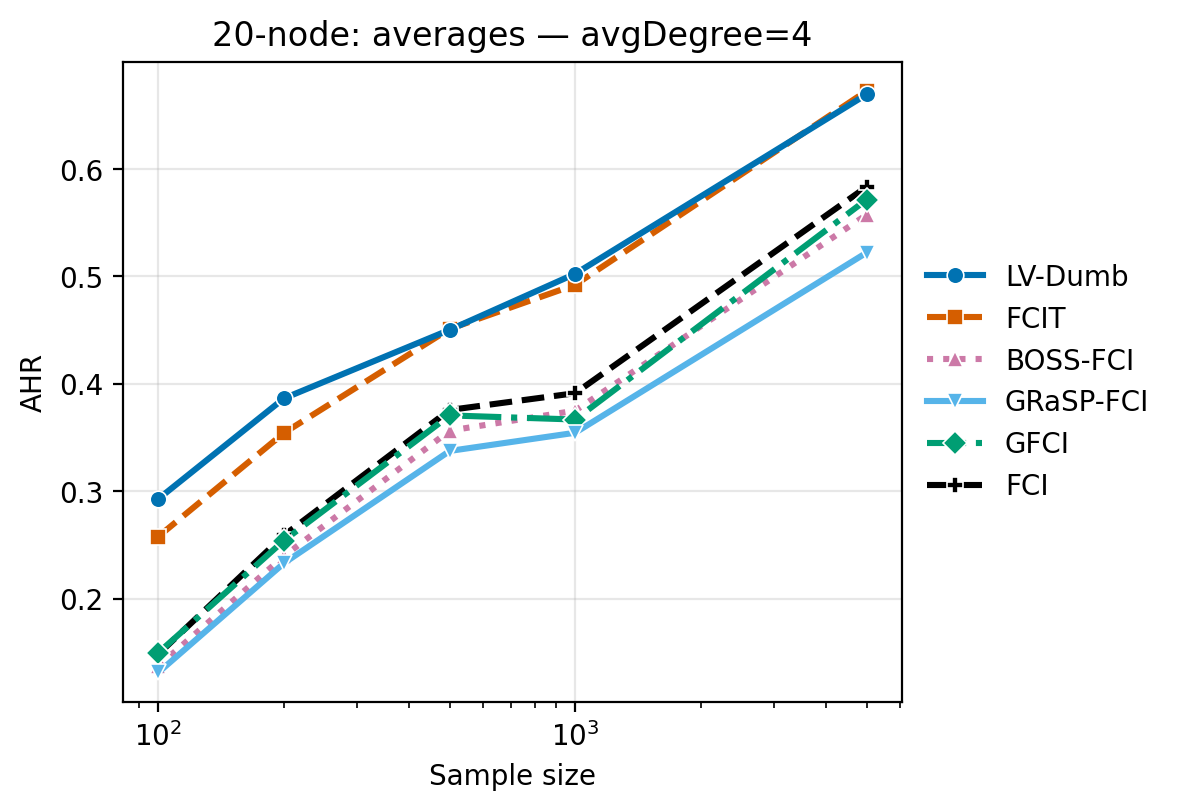}
    \caption{Arrowhead Recall (AHR) for 20-node graphs with average degree 4. 
    \textsc{FCIT} and \textsc{LV-Dumb} show the strongest growth in recall with increasing $N$, approaching 0.6 at large $N$. 
    \textsc{BOSS-FCI}, \textsc{GFCI}, and \textsc{GRaSP-FCI} track closely, while \textsc{GFCI} remains lower overall. 
    Results are averaged over graphs with 0, 4, and 8 latent common causes.}
    \label{fig:ahr20}
\end{figure}

\begin{figure}
    \centering
    \includegraphics[width=0.5\linewidth]{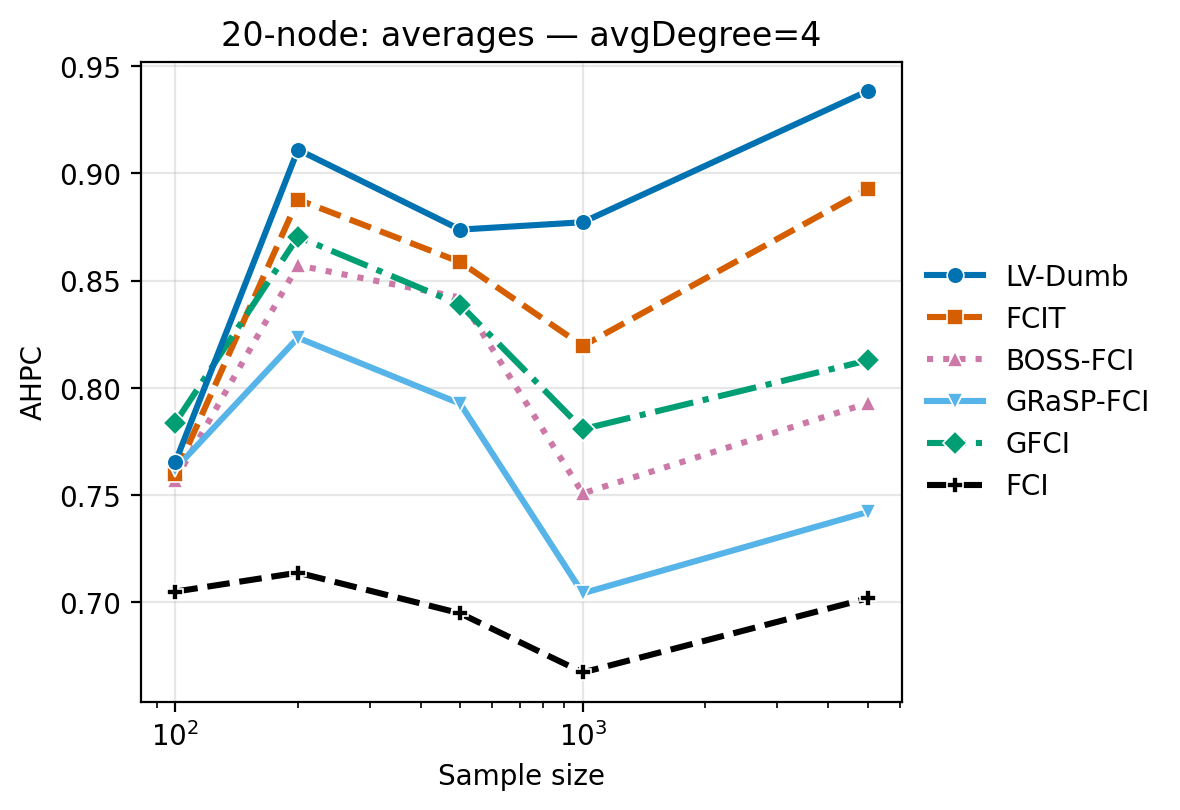}
    \caption{Arrowhead Precision for Common Adjacencies (AHPC) for 20-node graphs with average degree 4. 
    \textsc{LV-Dumb} and \textsc{FCIT} achieve the highest arrowhead precision, with \textsc{LV-Dumb} reaching 0.92 at large $N$. 
    \textsc{BOSS-FCI} and \textsc{GRaSP-FCI} show similar trends, while \textsc{GFCI} and especially \textsc{FCI} remain lower. 
    Results are averaged over graphs with 0, 4, and 8 latent common causes.}
    \label{fig:ahpc20}
\end{figure}

\begin{figure}
    \centering
    \includegraphics[width=0.5\linewidth]{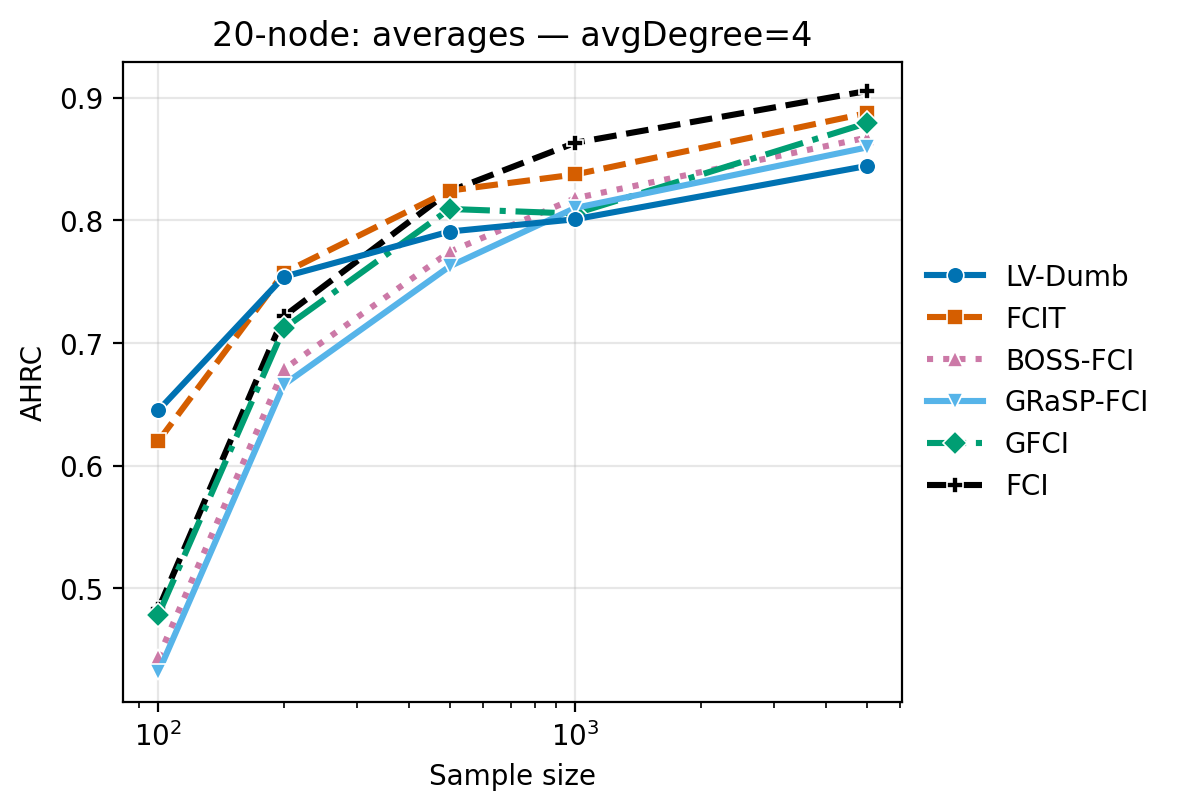}
    \caption{Arrowhead Recall for Common Adjacencies (AHRC) for 20-node graphs with average degree 4. 
    All algorithms improve steadily with increasing $N$, with \textsc{FCI}, \textsc{GRaSP-FCI}, and \textsc{BOSS-FCI} converging near 0.9. 
    \textsc{FCIT} levels off slightly below these, while \textsc{LV-Dumb} stabilizes around 0.85. 
    \textsc{GFCI} performs comparably to the best methods at large samples. 
    Results are averaged over graphs with 0, 4, and 8 latent common causes.}
    \label{fig:ahrc20}
\end{figure}

\begin{figure}
    \centering
    \includegraphics[width=0.5\linewidth]{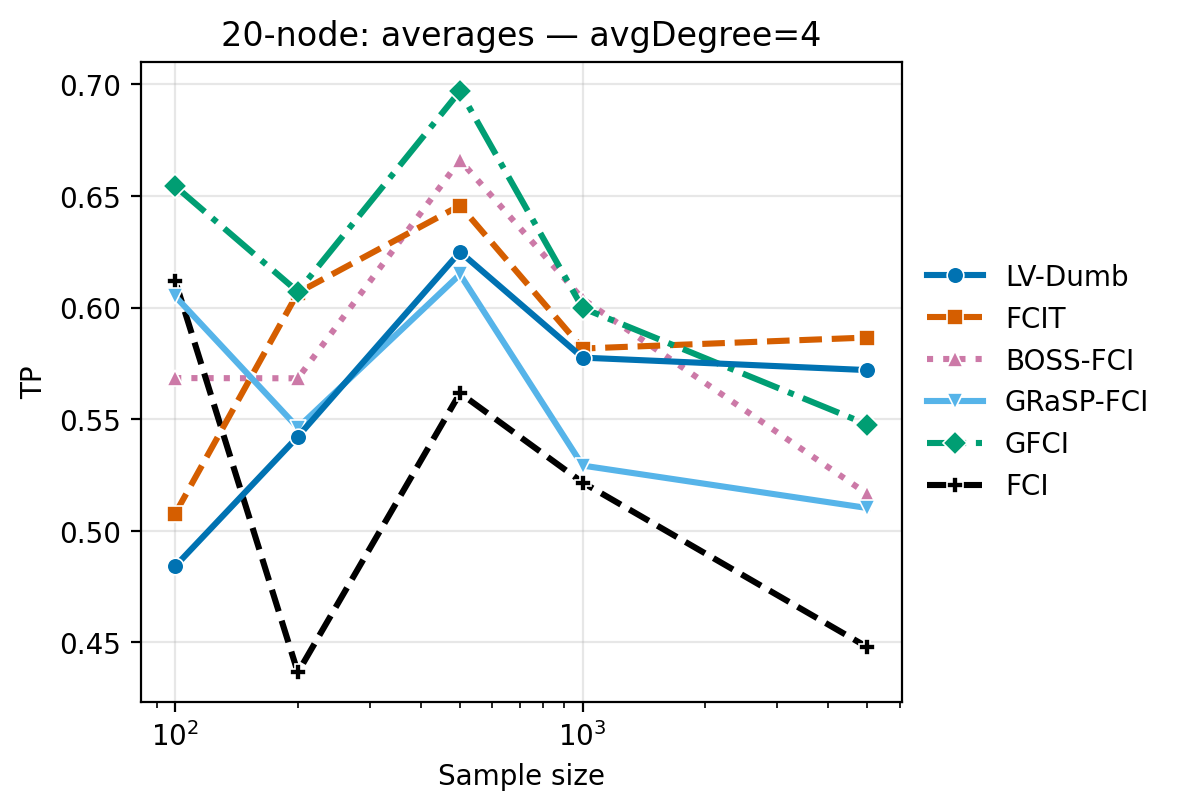}
    \caption{Tail Precision (TP) for 20-node graphs with average degree 4. 
    \textsc{GFCI}, \textsc{LV-Dumb}, and \textsc{LV-Dumb} lead here, with other algorihtms lagging behind.
    This metric measures tail precision relative to the true PAG; for a comparison to the directed paths in the true DAG, see Tail Path Precision.}
    \label{fig:tp20}
\end{figure}

\begin{figure}
    \centering
    \includegraphics[width=0.5\linewidth]{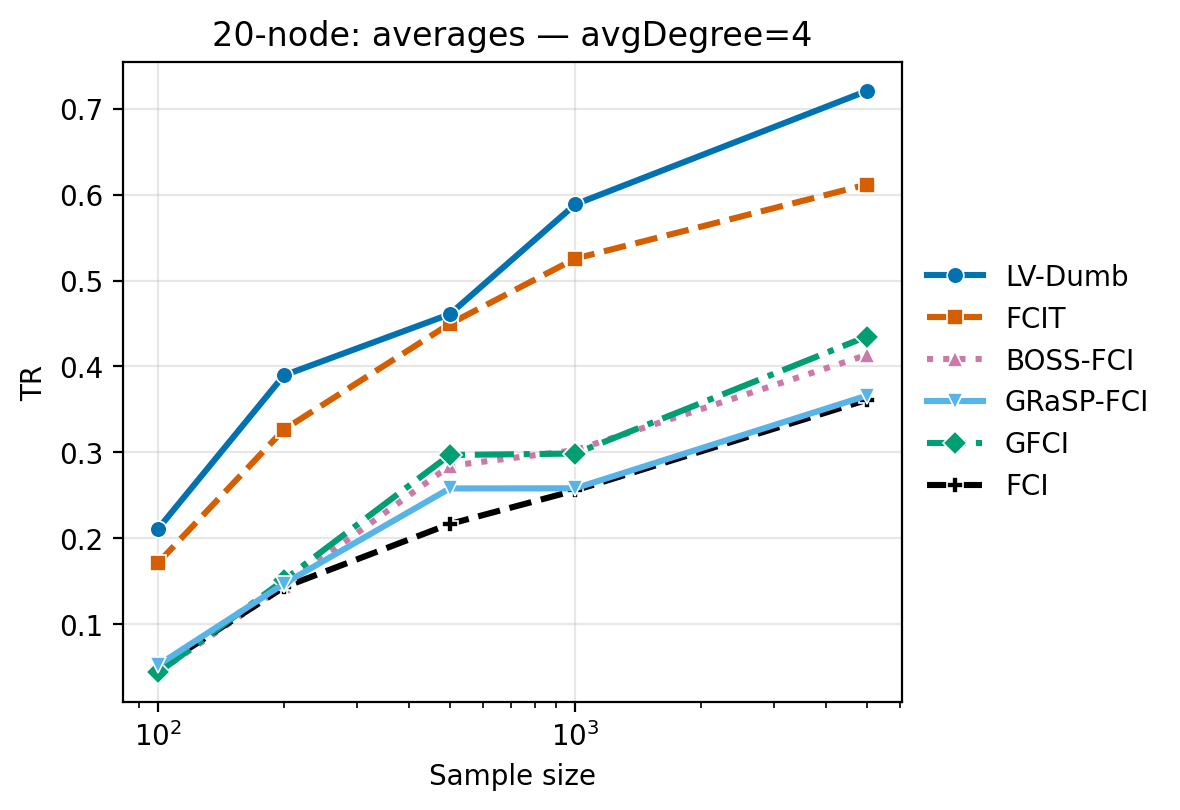}
    \caption{Tail Recall (TR) for 20-node graphs with average degree 4. 
    \textsc{LV-Dumb} and \textsc{FCIT} achieve the highest recall across all sample sizes, reflecting their ability to recover correct tail orientations even in the presence of latent confounding. 
    The other algorithms are grouped substantially lower. 
    Results are averaged over graphs with 0, 4, and 8 latent common causes.}
    \label{fig:tr20}
\end{figure}

\begin{figure}
    \centering
    \includegraphics[width=0.5\linewidth]{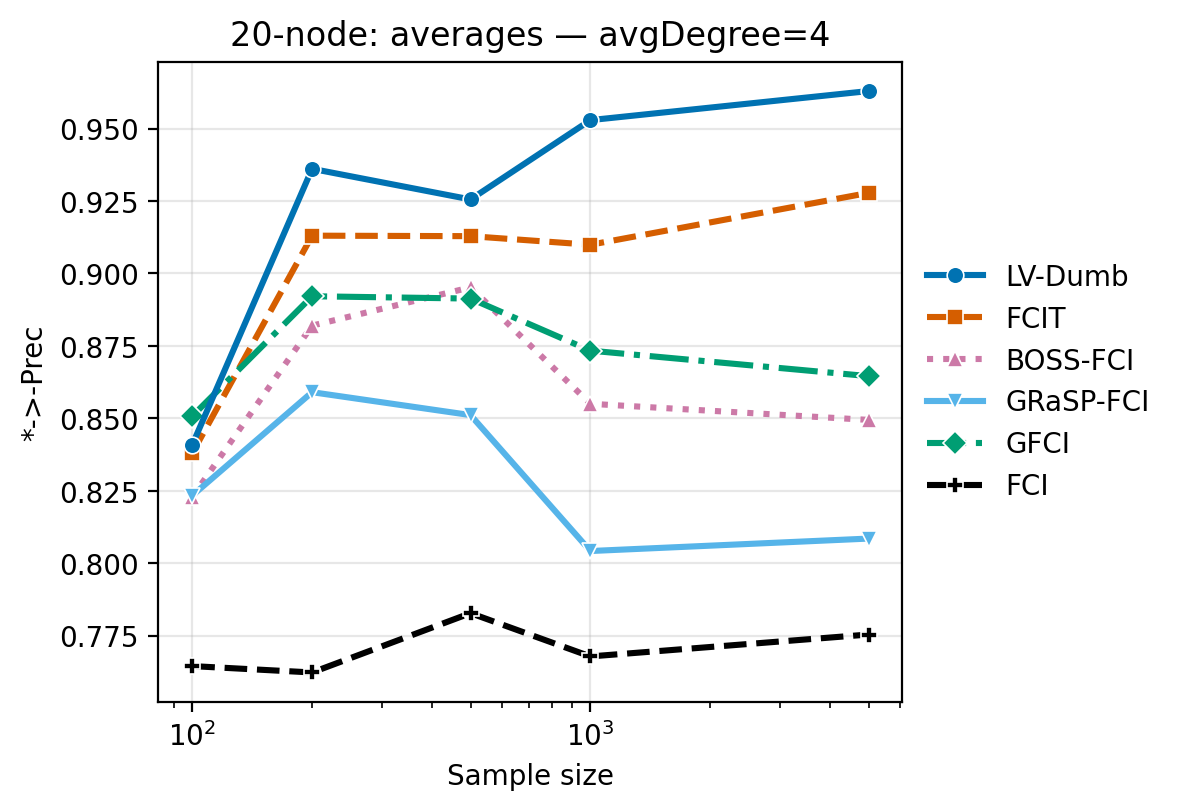}
    \caption{Arrow Path Precision for 20-node graphs with average degree 4. 
    \textsc{LV-Dumb} and \textsc{FCIT} again lead in directional precision, followed closely by \textsc{BOSS-FCI}, \textsc{GRaSP-FCI}, and \textsc{GFCI}. 
    \textsc{FCI} trails across sample sizes. 
    Results are averaged over graphs with 0, 4, and 8 latent common causes.}
    \label{fig:arrow20}
\end{figure}

\begin{figure}
    \centering
    \includegraphics[width=0.5\linewidth]{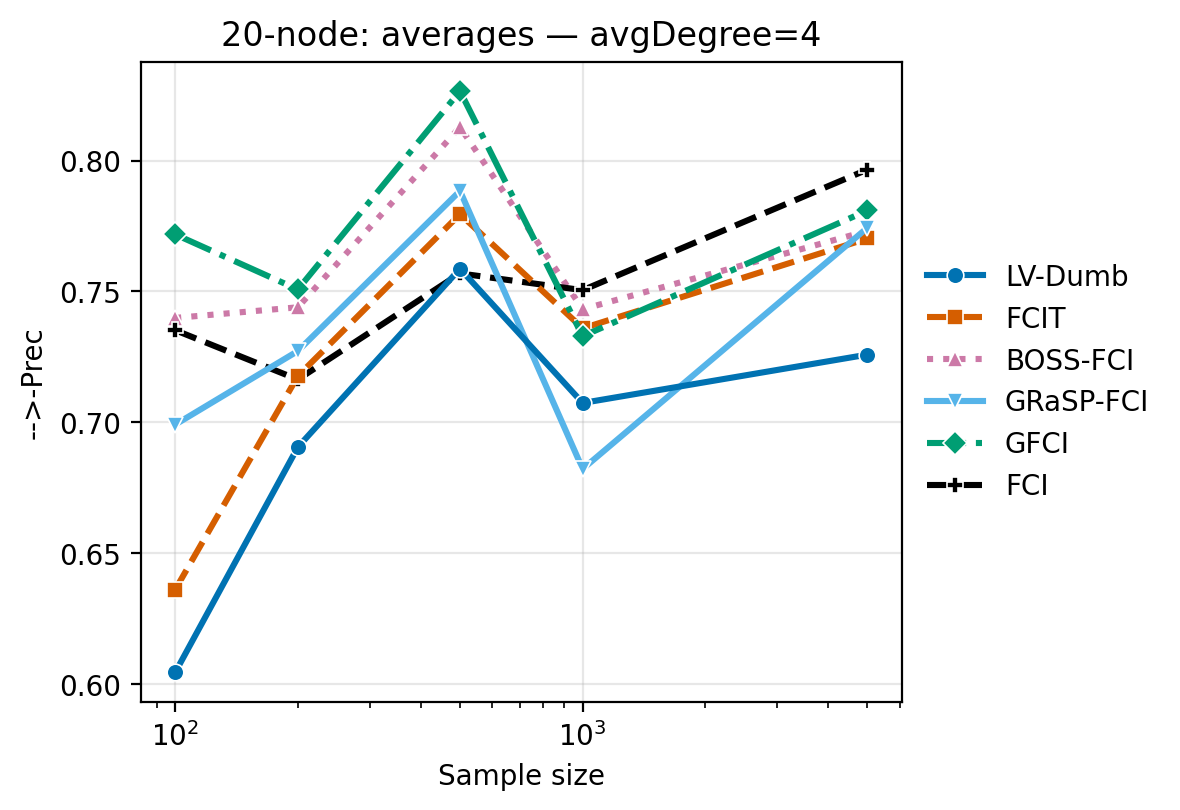}
    \caption{Tail Path Precision for 20-node graphs with average degree 4. 
    The algorithms are quite competitive on this measure, with \textsc{FCI} lagging somewhat behind.
    Results are averaged over graphs with 0, 4, and 8 latent common causes.}
    \label{fig:tail20}
\end{figure}

\begin{figure}
    \centering
    \includegraphics[width=0.5\linewidth]{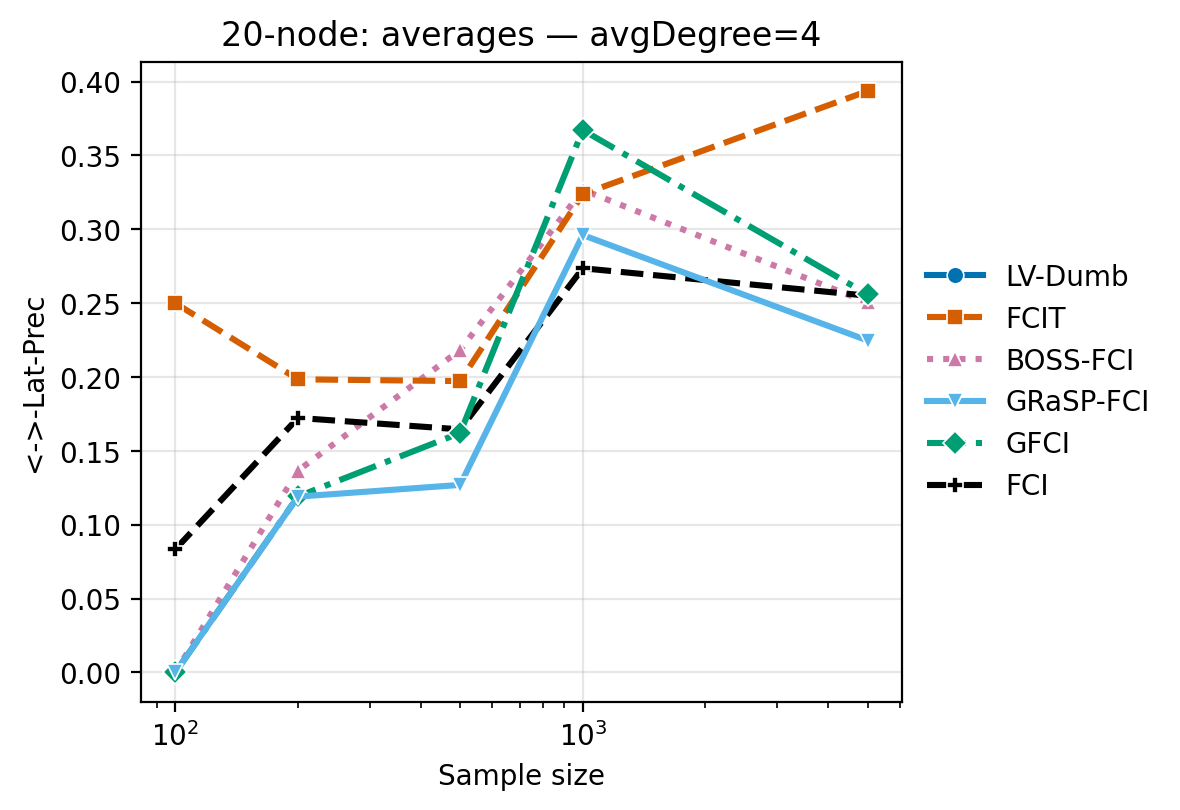}
    \caption{Bidirected Latent Path Precision for 20-node graphs with average degree 4. 
    Averaged over the number of latent variables, bidirected edge precision is somewhat variable across methods. 
    \textsc{FCIT} is the most consistent across sample sizes. 
    Results are averaged over graphs with 0, 4, and 8 latent common causes.}
    \label{fig:bidirected20}
\end{figure}

\begin{figure}
    \centering
    \includegraphics[width=0.5\linewidth]{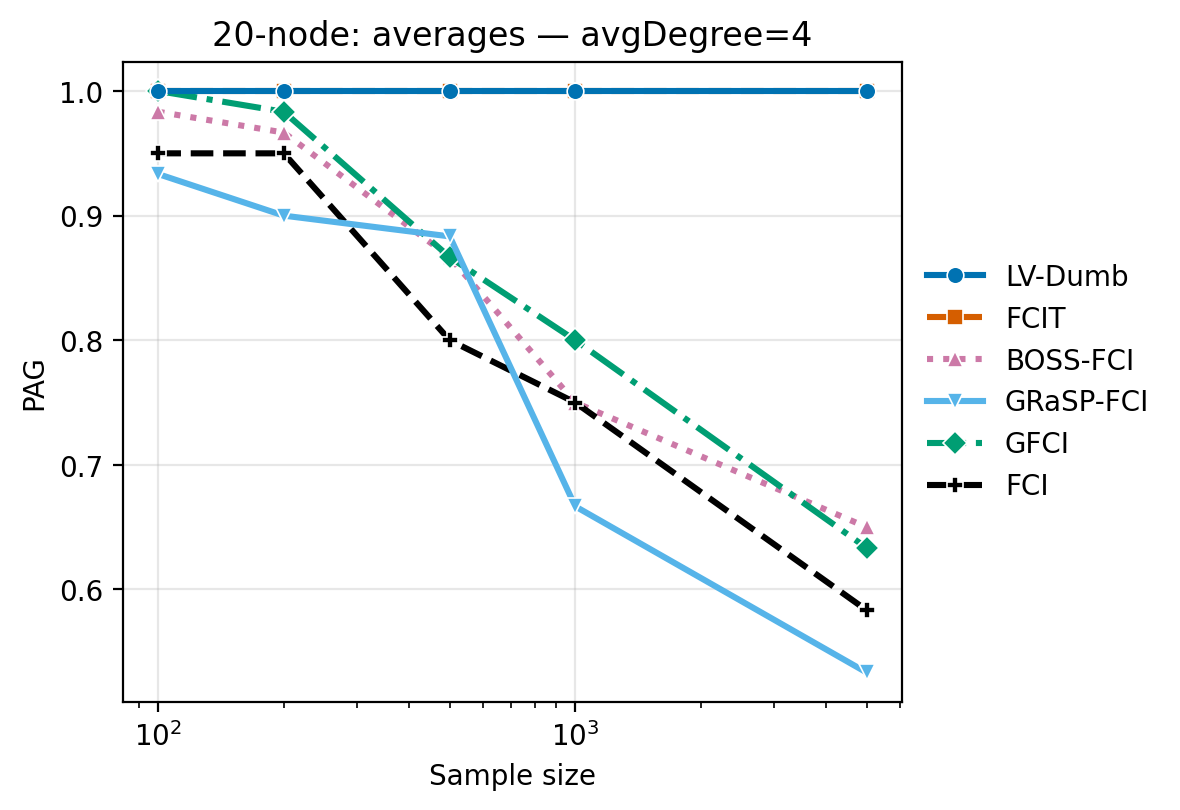}
    \caption{Proportion of well-formed PAGs for 20-node graphs with average degree 4. 
    As theoretically required, \textsc{LV-Dumb} and \textsc{FCIT} achieve perfect PAG validity, whereas other algorithms fall off—sometimes dramatically. 
    Results are averaged over graphs with 0, 4, and 8 latent common causes.}
    \label{fig:pag20}
\end{figure}

\begin{figure}
    \centering
    \includegraphics[width=0.5\linewidth]{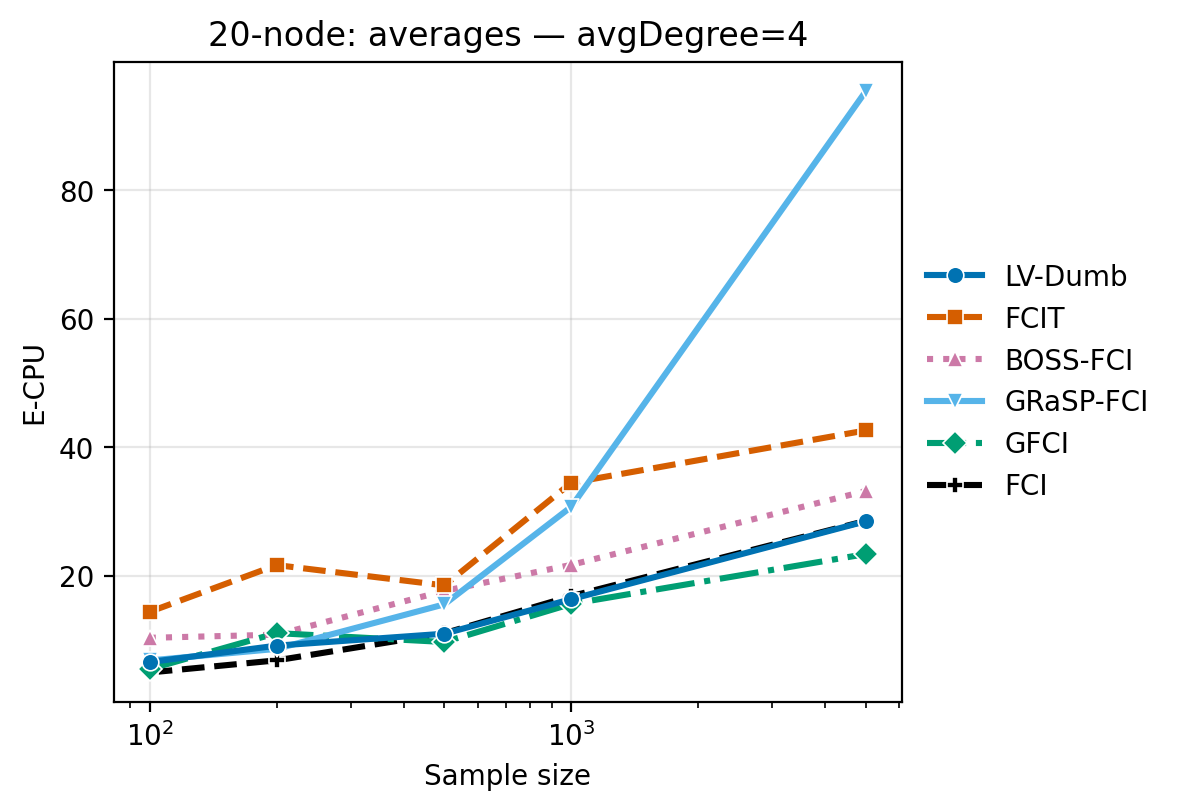}
    \caption{Runtime (CPU milliseconds) for 20-node graphs with average degree 4.  
    All algorithms are quite fast in this regime, with the exception of \textsc{GFCI}, which slows at $N{=}5000$. 
    \textsc{LV-Dumb} lags behind \textsc{LV-Dumb}.
    Results are averaged over graphs with 0, 4, and 8 latent common causes.}
    \label{fig:cpu20}
\end{figure}

\begin{table}[htbp]
\centering
\caption{Comparison of algorithms on precision and PAG metrics.}
\begin{tabular}{lccccc}
\toprule
\textbf{Algorithm} & \textbf{*->-Prec} & \textbf{-->-Prec} & \textbf{<->-Lat-Prec} & \textbf{E-CPU} & \textbf{PAG} \\
\midrule
LV-Dumb    & 0.9939 & 0.9320 & *      & 21850.5654 & 1.0000 \\
FCIT       & 0.9937 & 0.9319 & 0.0000 & 24430.5061 & 1.0000 \\
BOSS-FCI   & 0.8971 & 0.9763 & 0.0961 & 23183.7960 & 0.1000 \\
GRaSP-FCI  & 0.8889 & 0.9693 & 0.0966 & 92879.1827 & 0.2500 \\
GFCI       & 0.7861 & 0.8812 & 0.0604 & 54695.2840 & 0.2000 \\
FCI        & 0.6519 & 0.8200 & 0.0442 & 1253.2596  & 0.0500 \\
\bottomrule
\end{tabular}
\label{tab:100node}
\end{table}

%-------------------------------------------------
\subsection{Real-Data Analysis}
%-------------------------------------------------

We analyzed the Algerian Forest Fire dataset from the UCI Machine Learning Repository \citep{algerian_forest_fires_547}. Because we have not discussed nonlinearity or determinism, we removed the six deterministic Canadian Forest Fire System indices—BUI (Build-Up Index), DC (Drought Code), DMC (Duff Moisture Code), FFMC (Fine Fuel Moisture Code), and FWI (Fire Weather Index). We also excluded \emph{Day} (cyclical variable) and \emph{Year} (constant at 2012 and therefore non-causal). This left the variables \emph{Region} (binary), \emph{Month}, \emph{Relative Humidity (RH)}, \emph{Rain}, \emph{Temperature}, \emph{Wind Speed (Ws)}, and \emph{Fire} (binary occurrence of fire).

Because the dataset mixes discrete and continuous variables, we used the Degenerate Gaussian score and test \citep{andrews2019learning}, with penalty discount $1.0$ and $\alpha = 0.05$. This method models discrete variables as multinomial and continuous variables as linearly related, enabling hybrid analysis under a unified likelihood framework.

The resulting FCIT model (Figure~\ref{fig:algerian}) closely follows the initial BOSS CPDAG while clarifying uncertain endpoints in the inferred PAG. Notably, \emph{Fire} has two direct causes—\emph{Temperature} and \emph{Relative Humidity}—with \emph{Rain} as a possible third, a pattern that is meteorologically plausible. \emph{Month} and \emph{Region} are correctly identified as exogenous, and the remaining edges are sensible.

For comparison, LV-Dumb returns a PAG nearly identical to the BOSS CPDAG, without introducing bidirected edges. BOSS-FCI and GRaSP-FCI produces similar but slightly denser graphs. 

In \citet{ramsey2025scalablecausaldiscoveryrecursive}, we conducted a complementary analysis using ChatGPT-4 to provide “expert” judgments of edge and orientation correctness for this dataset. This included models where the degenerate Gaussian score was applied without the Canadian Fire System variables (i.e., the same configuration analyzed here), in cases where no domain expert was available. In that study, the FCIT graph was judged to offer the most accurate overall summary, whereas alternative models exhibited spurious adjacencies or misoriented edges. FCIT preserved the accuracy of the BOSS initialization while refining orientations in a manner consistent with PAG semantics, highlighting its advantage—when paired with the degenerate Gaussian test and score—for real-world data containing mixtures of continuous and discrete variables.

Beyond our own benchmark datasets, FCIT and the related nonlinear basis-function test/score \citep{ramsey2025scalablecausaldiscoveryrecursive} have recently been applied to astrophysical data at scale.
\citep{desmond2025causalstructuregalacticastrophysics} applied FCIT to a highly nonlinear, non-Gaussian sample of roughly $5\times10^5$ low-redshift galaxies from the NASA Sloan Atlas. Using the Basis-Function LRT for independence testing and Basis-Function BIC for scoring, they recovered a causal graph revealing hierarchical, mass-driven relations among galaxy properties while distinguishing them from observational selection effects, which in this note was judged to be quite plausible, recovering known relationships among the studied variables. This study, submitted to MNRAS, demonstrates that FCIT with an appropriate test and score scales to genuinely astronomical sample sizes and provides interpretable and plausible causal hypotheses for physical mechanisms—marking the first published scientific deployment of FCIT.

\begin{figure}
    \centering
    \includegraphics[width=0.3\linewidth]{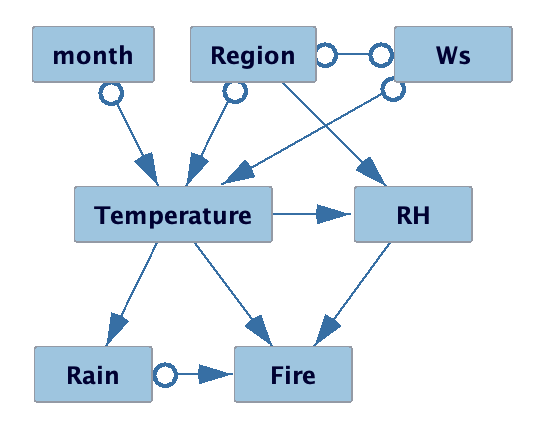}
    \caption{FCIT result for Algerian Forest Fire data using the Degenerate Gaussian score and test, excluding deterministic indices.}
    \label{fig:algerian}
\end{figure}

\section{Discussion and Conclusion}
%=================================================

We have introduced a family of score-guided, mixed-strategy algorithms for causal discovery under latent variables and selection bias. \emph{BOSS-FCI} and \emph{GRaSP-FCI} extend the GFCI framework by substituting BOSS or GRaSP for FGES in the initial CPDAG search. \emph{LV-Dumb} provides a simple yet effective heuristic alternative that returns the PAG equivalence class of the BOSS DAG, scaling extremely well but lacking the ability to remove edges indicative of latent confounding. Finally, \emph{FCIT} introduces a targeted-testing strategy that addresses a key statistical limitation of FCI-style methods.

A central innovation of FCIT is its integration of recursive path blocking into edge removal, in conjunction with discriminating paths, allowing separating sets to be identified using full path structure rather than adjacency subsets alone. Each update is checked for structural validity, guaranteeing that FCIT always produces a well-formed PAG. Our theoretical results establish the soundness of recursive blocking, the edge-minimality of the returned PAG, and the soundness of all orientations in the large-sample limit. Empirically, FCIT achieves consistently high adjacency and orientation precision, scales effectively to large models, and outperforms existing FCI-style methods in both accuracy and efficiency. Although not formally correct, \emph{LV-Dumb} remains a remarkably fast and accurate heuristic for very large problems.

These contributions open several directions for future work. First, automated strategies for jointly tuning score penalties and test thresholds could adaptively balance Type~I and Type~II errors in practice. Second, extending FCIT to dynamic or time-series data is a natural next step, as are variants that incorporate interventional or mixed observational–interventional data. Finally, while FCIT infers PAGs that reflect latent confounding, it does not explicitly model the latent variables themselves; integrating latent-variable identification within the same framework remains an important challenge.

Finally, we would be remiss not to note an important opportunity for future software development relevant to all algorithms in this paper that begin with a BOSS initialization—\emph{LV-Dumb}, \emph{BOSS-FCI}, and \emph{FCIT}—for the linear Gaussian or linear non-Gaussian cases. Our implementations have focused on Java for portability and integration with the Tetrad platform, which already provides strong performance advantages over many alternatives. However, a recent C-language implementation, \emph{CBOSS}, available in the \texttt{causal-get} GitHub package (\url{https://github.com/bja43/causal-get}, due to Bryan Andrews), achieves substantial speed improvements.\footnote{The speed improvements are largely due to low-level linear algebra optimizations for calculating likelihoods. The practical extension to the linear non-Gaussian case is benchmarked in \citet{andrews2023fast}.} For problems with 100 nodes and average degree 10—the “large” case profiled above—CBOSS completes in under a second with performance identical to our Java-based BOSS and scales efficiently to much larger problems. This opens the possibility of integrating CBOSS as an external module for the BOSS phase of our algorithms. In Java, such integration poses some engineering challenges, but a hybrid approach—running the BOSS stage in C and importing the resulting graph into Java for subsequent processing—would be straightforward. Alternatively, the remaining steps of our BOSS-based algorithms could themselves be reimplemented in C. Either direction represents a practical and worthwhile programming goal for future work.

In summary, FCIT provides a practical and statistically efficient alternative to existing FCI-style algorithms, consistently yielding well-formed PAGs with improved accuracy and scalability. Together with \emph{BOSS-FCI}, \emph{GRaSP-FCI}, and the heuristic \emph{LV-Dumb}, it expands the toolkit for causal discovery under latent variables—bridging the gap between correctness, scalability, and practical usability.

\begin{ack}
Thanks to Clark Glymour for helpful comments.

JR was supported by the US Department of Defense under Contract Number FA8702-15-D-0002 with Carnegie Mellon University for the operation of the Software Engineering Institute. The content of this paper is solely the responsibility of the authors and does not necessarily represent the official views of these funding agencies. BA was supported by the US National Institutes of Health under the Comorbidity: Substance Use Disorders and Other Psychiatric Conditions Training Program T32DA037183. The authors have no conflict of interest to report. PS was supposed by NIH Award Number: 1043252, ``Interpretable graphical models for large multi-model COPD data'' and NSF award number 2229881, ``AI Institute for Societal Decision Making.''
\end{ack}

\bibliographystyle{apalike}
\bibliography{refs.bib} 

% \clearpage  
% \newpage
\appendix

\section{Algorithm Details}
\label{sec:algorithm_details}

\subsection{GFCI Variants, Find Path to Target, Recursive Blocking}

The structure of GFCI, BOSS-FCI and GRaSP-FCI is shown in Algorithm \ref{alg:Star-FCI}, which we refer to using the templated name, ``Star-FCI''. Note that we do not include the possible d-sep step in this code which was included in the published version of GFCI, following our conjecture that it is not needed.

For FCIT, The crucial algorithm \textit{find\_path\_to\_target} is given in Algorithm \ref{alg:find_path_to_target}, along with helper methods in Algorithm \ref{alg:getReachableNodes}. This is used in the algorithm, \textit{block\_paths\_recursively}, Algorithm  \ref{alg:block_paths_recursively}, to construct sets \(\textbf{B}\) on the fly to block all paths between variables \(X\) and \(Y\). This uses the notion of reachability given in \citep{geiger1990d}. The optimization of the discriminating path rule R4 is given in Algorithm \ref{alg:ddp_recursive}.
% Finally, Algorithm \ref{alg:ordered_local_markov_property} gives pseudocode for the ordered local Markov property used to assess whether a graph passes Markov. 
Section \ref{app:discpaths} gives the procedure we use to list discriminating and pre-discriminating paths.

As an intuition, the recursive blocking procedure underlying FCIT is closely related to the standard depth-first search algorithm for determining d-separation between variables $X$ and $Y$ given a conditioning set $\mathbf{Z}$ \citep{geiger1990d}. Its main innovation is to construct the conditioning set dynamically: as the algorithm explores paths from $X$ to $Y$, it incrementally builds a blocking set $\mathbf{B}$ by adding nodes that close open trails—“in a clever way.”\footnote{We borrow this turn of phrase from \citet{zhang2008completeness}.} When the procedure halts and $X$ and $Y$ are not adjacent, it returns a final set $\mathbf{B_{final}}$ such that $\text{d-sep}(X,Y \mid \mathbf{B}_{final})$, or \textsc{Null} if no such separating set exists. The same logic applies to DAGs (directed acyclic graphs), CPDAGs (completed partially directed acyclic graphs), MAGs (mixed ancestral graphs), and PAGs (partial ancestral graphs). If $X$ and $Y$ are adjacent, the procedure instead returns a set that blocks all non-inducing paths between them—a property that FCIT directly exploits.

\begin{algorithm}
\caption{\texttt{getReachableNodes} and \texttt{reachable}}
\label{alg:getReachableNodes}
\textbf{Goal:} Determine which adjacent nodes \(c\) of \(b\) can be traversed, 
given a prior node \(a\) and a conditioning set \(\mathbf{B}\).

\vspace{1em}
\textbf{Procedure:} \(\texttt{getReachableNodes}(\textit{graph}, a, b, \mathbf{B})\)
\begin{enumerate}
  \item Initialize an empty list \(\textit{reachable}\).
  \item For each node \(c \in \textit{graph.getAdjacentNodes}(b)\):
  \begin{enumerate}
    \item If \(c = a\), \textbf{continue}.
    \item If \(\texttt{reachable}(\textit{graph}, a, b, c, \mathbf{B})\) is true, append \(c\) to \(\textit{reachable}\).
  \end{enumerate}
  \item \textbf{return} \(\textit{reachable}\).
\end{enumerate}

\vspace{1em}
\textbf{Helper:} \(\texttt{reachable}(\textit{graph}, a, b, c, \mathbf{B})\)
\begin{enumerate}
  \item Let \(\textit{collider} \gets \textit{graph.isDefCollider}(a,b,c)\).
  \item \textbf{If} \((\lnot \textit{collider} \;\lor\; \textit{graph.isUnderlineTriple}(a,b,c))\) 
        \textbf{and} \(b \notin \mathbf{B}\), \textbf{return} \(\text{true}\).
  \item \textbf{Else} \textbf{return} \(\textit{collider} \wedge \textit{isAncestorOfAnyB}(b,\mathbf{B})\).
\end{enumerate}
\end{algorithm}

\begin{algorithm}
\caption{Recursive Blocking Procedure: \texttt{block\_paths\_recursively}}
\label{alg:block_paths_recursively}
\begin{algorithmic}[1]
\STATE \textbf{Goal:} Build a candidate conditioning set $\mathbf{B}$ that blocks all blockable $m$-connecting paths from $x$ to $y$, ignoring the direct edge $x\!-\!y$ on the first hop. If no valid blocking set exists, return \textsc{Null}.
\STATE \textbf{Input:} PAG $G$; nodes $x,y$; initial conditioning set $\textit{containing}$ (typically $\varnothing$).
\STATE Set \(\mathbf{B} \gets \textit{containing}\).
\STATE Initialize \(\textbf{P} \gets \langle x \rangle\).
\FORALL{nodes \(b \in \mathrm{Adj}_G(x)\)}
  \IF{\(b = y\)}
    \STATE \textbf{continue} \COMMENT{Ignore the direct edge \(x\text{--}y\) on the first hop}
  \ENDIF
  \STATE \(r \gets \texttt{find\_path\_to\_target}(G,\, a{=}x,\, b,\, y,\, \textbf{P},\, \mathbf{B})\)
  \IF{\(r = \textsc{Unblockable}\) \textbf{or} \(r = \textsc{Indeterminate}\)}
    \STATE \textbf{return} \textsc{Null} \COMMENT{No valid graphical sepset exists}
  \ENDIF
\ENDFOR
STATE \textbf{return} $\mathbf{B}$ \COMMENT{All explored branches were blocked; may be $\varnothing$}
\end{algorithmic}
\end{algorithm}

\begin{algorithm}
\caption{*-FCI Algorithm (``Star-FCI'')}
\label{alg:Star-FCI}

\KwIn{
    Dataset $\mathcal{D}$ with $n$ variables \\
    Significance level $\alpha$ \\
    Edge-minimal Markovian CPDAG procedure $\mathcal{M}$ for initial graph estimation \\
    Conditional independence oracle
}

\KwOut{Estimated PAG $\mathcal{G}$}

\BlankLine
\textbf{Step 1: Initial Graph Estimation} \\
Obtain an initial graph $\mathcal{G}_{\text{CPDAG}}$ using edge-minimal Markovian CPDAG procedure $\mathcal{M}$\\
Initialize PAG $\mathcal{G}_{\text{PAG}}$ by copying $\mathcal{G}_{\text{CPDAG}}$\\

\BlankLine
\textbf{Step 2: Extra Edge Removal Step} \\
\ForEach{edge $X *\!\!{-}\!\!* Y$ in $\mathcal{G}_{\text{PAG}}$}{
    \ForEach{set $S$ where $S \subseteq \text{adj}(X,\mathcal{G}_{\text{PAG}})\setminus\{Y\}$ or $S \subseteq \text{adj}(Y,\mathcal{G}_{\text{PAG}})\setminus\{X\}$}{
        \If{$X \ind Y \mid S$}{
            Remove edge $X *\!\!{-}\!\!* Y$ from $\mathcal{G}_{\text{PAG}}$\\
            Store $S$ in $\mathcal{S}(X,Y)$\\
            \textbf{break}\\
        }
    }
}

\BlankLine
\textbf{Step 3: Revised R0 Rule (Collider Identification)} \\
Orient each edge $X *\!\!{-}\!\!* Y$ in $\mathcal{G}_{\text{PAG}}$ as $X \circ - \circ Y$\\

\ForEach{unshielded triple $\langle X,Y,Z\rangle$ in $\mathcal{G}_{\text{PAG}}$}{
    \If{$X \rightarrow Y \leftarrow Z$ in $\mathcal{G}_{\text{CPDAG}}$} {
        Orient as $X *\to Y \gets* Z$ in $\mathcal{G}_{\text{PAG}}$\\
    }
    \ElseIf{$\sim adj(X,Y,\mathcal{G}_{\text{CPDAG}})$}{
        $S \gets \mathcal{S}(X,Y)$\\
        \If{$S \neq \textsc{Null}$ and $Y \notin S$}{
            Orient as $X *\to Y \gets* Z$ in $\mathcal{G}_{\text{PAG}}$\\
            \textbf{break}\\
        }
    }
}

\BlankLine
\textbf{Step 4: Final Collider Orientation and FCI Rules} \\
Apply FCI final orientation rules to refine $\mathcal{G}_{\text{PAG}}$\\

\Return $\mathcal{G}_{\text{PAG}}$\\
\end{algorithm}

\begin{algorithm}
\caption{\texttt{find\_path\_to\_target} Procedure}
\label{alg:find_path_to_target}
\textbf{Goal:} Decide whether there exists an unblocked path from node $a$ to node $y$ passing through node $b$, given a dynamically maintained conditioning set $\mathbf{B}$.

\textbf{Return values:}  
\textsc{Null} if no valid blocking set can be found; otherwise $\mathbf{B}$ (possibly $\varnothing$).

\textbf{Inputs:}
\begin{itemize}
  \item \textit{graph}: The graph structure
  \item \textit{a,b,y}: Node identifiers
  \item \textbf{P}: Set of nodes already visited on this branch
  \item \(\mathbf{B}\): Current conditioning set (modifiable)
  \item \(\textbf{F}\): Optional set of vertices to exclude from traversal (e.g., fixed colliders on discriminating paths)
\end{itemize}

\textbf{Procedure:}
\begin{enumerate}
  \item \textbf{Target check.} If \(b = y\), \emph{return} \texttt{UNBLOCKABLE}.
  \item \textbf{Cycle guard.} If \(b \in \textbf{P}\), \emph{return} \texttt{UNBLOCKABLE}.\\
        \emph{(A revisitation means we cannot certify blocking on this branch.)}
  \item \textbf{Add \(b\) to the path}: insert \(b\) into \(\textbf{P}\).
  \item \textbf{Compute neighbors to explore}:
        \[
          \textit{reachable} \gets \texttt{getReachableNodes}(\textit{graph}, a, b, \mathbf{B})
          \setminus \textbf{P} \setminus \textbf{F}.
        \]

  \item \textbf{Case 1: \(b\) is latent or \(b \in \mathbf{B}\).}
    \begin{enumerate}
      \item \textbf{For each} \(c \in \textit{reachable}\):
        \begin{itemize}
          \item \(\textit{res} \gets \textit{find\_path\_to\_target}(\textit{graph}, b, c, y, \textbf{P}, \mathbf{B}, \textbf{F})\).
          \item \textbf{If} \(\textit{res} = \texttt{UNBLOCKABLE}\): remove \(b\) from \(\textbf{P}\); \emph{return} \texttt{UNBLOCKABLE}.
        \end{itemize}
      \item \textbf{All continuations blocked:} remove \(b\) from \(\textbf{P}\); \emph{return} \texttt{BLOCKED}.
    \end{enumerate}

  \item \textbf{Case 2: \(b\) is neither latent nor in \(\mathbf{B}\).}
    \begin{enumerate}
      \item \emph{Branch A (do not add \(b\) to \(\mathbf{B}\)).}
        \begin{itemize}
          \item \textbf{For each} \(c \in \textit{reachable}\):
            \begin{itemize}
              \item \(\textit{res} \gets \textit{find\_path\_to\_target}(\textit{graph}, b, c, y, \textbf{P}, \mathbf{B}, \textbf{F})\).
              \item \textbf{If} \(\textit{res} = \texttt{UNBLOCKABLE}\): remove \(b\) from \(\textbf{P}\); \emph{return} \texttt{UNBLOCKABLE}.
            \end{itemize}
        \end{itemize}

      \item \emph{Branch B (temporarily add \(b\) to \(\mathbf{B}\)).}
        \begin{itemize}
          \item Insert \(b\) into \(\mathbf{B}\).
          \item Recompute
                \[
                  \textit{reachable} \gets \texttt{getReachableNodes}(\textit{graph}, a, b, \mathbf{B})
                  \setminus \textbf{P} \setminus \textbf{F}.
                \]
          \item \textbf{For each} \(c \in \textit{reachable}\):
            \begin{itemize}
              \item \(\textit{res} \gets \textit{find\_path\_to\_target}(\textit{graph}, b, c, y, \textbf{P}, \mathbf{B}, \textbf{F})\).
              \item \textbf{If} \(\textit{res} = \texttt{UNBLOCKABLE}\): remove \(b\) from \(\mathbf{B}\); remove \(b\) from \(\textbf{P}\); \emph{return} \texttt{UNBLOCKABLE}.
            \end{itemize}
          \item Remove \(b\) from \(\mathbf{B}\) \textit{(backtrack)}.
        \end{itemize}

      \item \emph{Conclude Case 2:} remove \(b\) from \(\textbf{P}\); \emph{return} \texttt{BLOCKED}.
    \end{enumerate}
\end{enumerate}

\textbf{Return Value:}
\texttt{UNBLOCKABLE} if some \(a\!\leadsto\!y\) path through \(b\) cannot be blocked under the explored choices;
\texttt{BLOCKED} if all such paths are blocked by \(\mathbf{B}\) (with or without temporarily conditioning on \(b\)).
\end{algorithm}

\newcommand{\msep}{\mathrel{\perp\!\!\!\perp_m}}

\newpage

\begin{algorithm}[H]
\caption{\textsc{R4 Discriminating-Path Optimization via Recursive Blocking}}
\label{alg:ddp_recursive}
\DontPrintSemicolon
\SetKwInOut{Input}{Input}
\SetKwInOut{Output}{Output}

\Input{%
  PAG \(G=(V,E)\);\;
  distinct vertices \(x,y\in V\);\;
  (optional) oracle \(\mathcal M\) providing \(m\)-separation in \(G\);\;
  independence test \(\mathcal T\);\;
  orientation helper \(\mathcal F\);\;
  maximum blocking-path length \(L_b\);\;
  maximum discriminating-path length \(L_d\);\;
  subset-depth bound \(d\) (\(-1\) = no bound)}
\Output{Separating set \(S\subseteq V\):
  in oracle mode \(x \msep_G y \mid S\);
  in test mode \(x \ind y \mid S\);
  or \(\textsc{Null}\) if none exists}

\BlankLine
%-- 1. pre-orient everything except R4 ------------------------
\(\mathcal F.\mathrm{setDoR4}(\mathrm{false})\);\;
\(\mathcal F.\mathrm{finalOrientation}(G)\);\;
\(\mathcal F.\mathrm{setDoR4}(\mathrm{true})\);\;

\BlankLine
%-- 2. discriminating paths up to length L_d ------------------
\(\mathcal P \leftarrow \mathrm{listDiscriminatingPaths}(G,L_d)\)\;

\BlankLine
%-- 3. keep only paths whose endpoints are \{x,y\} ------------
\(\mathcal P \leftarrow \{\,p\in\mathcal P \mid \{p.x,p.y\}=\{x,y\}\,\}\)\;

\BlankLine
%-- 4. candidate “not-followed’’ vertices ---------------------
\(NF_{\mathrm{cand}}\leftarrow
  \{\,p.v \mid p\in\mathcal P \wedge \mathrm{Endpoint}(p.y,p.v)=\circ \,\}\)\;
\(d_1\leftarrow\begin{cases}
  |NF_{\mathrm{cand}}|,& d=-1\\
  d,&\text{otherwise}
\end{cases}\)

\BlankLine
%-- 5. vertices adjacent to both x and y ----------------------
\(\textit{Common}\leftarrow\Adj_G(x)\cap\Adj_G(y)\)\;
\(d_2\leftarrow\begin{cases}
  |\textit{Common}|,& d=-1\\
  d,&\text{otherwise}
\end{cases}\)

\BlankLine
\ForEach{\(NF\subseteq NF_{\mathrm{cand}}\) with \(|NF|\le d_1\)}{%
  \(B\leftarrow
    \mathrm{block\_paths\_recursively}(G,x,y,\emptyset,NF,L_b)\)\;

  \ForEach{\(C\subseteq\textit{Common}\) with \(|C|\le d_2\)}{%
    \If(\tcp*[f]{skip definite colliders})%
      {\(\exists c\in C:\;x\to c\gets y\)}{\textbf{continue}}

    \(S\leftarrow B\setminus C\)\;

    \eIf{\(\mathcal M\neq\varnothing\)}{%
       \(\textit{indep}\leftarrow
         \mathcal M.\mathrm{markovIndependence}(x,y,S)\) \tcp*{\(x \msep_G y \mid S\)}
    }{%
       \(\textit{indep}\leftarrow
         \mathcal T.\mathrm{checkIndependence}(x,y,S)\) \tcp*{\(x \ind y \mid S\)}
    }

    \If{\(\textit{indep}\)}{\Return \(S\)}
  }
}
\Return \textsc{Null}\;
\end{algorithm}

\subsection{Enumerating Discriminating and Pre-Discriminating Paths}
\label{app:discpaths}

\paragraph{Setup and conventions.}
Let $G$ be a PAG (or partially oriented graph) over measured variables.
We treat paths as ordered tuples $\langle v_0,\dots,v_k\rangle$ (not bold),
and we boldface node sets (e.g., $\mathbf{C}, \mathbf{S}$).
We assume access to the usual adjacency/orientation queries:
$\Adj_G(\cdot)$, $\text{isParentOf}(\cdot,\cdot)$,
$\text{getEndpoint}(\cdot,\cdot)\in\{\textsc{Tail},\textsc{Circle},\textsc{Arrow}\}$,
and $\text{nodesInto}(t,\textsc{Arrow})=\{x: x *\!\!\rightarrow t\}$.

\begin{definition}[Pre-discriminating and discriminating paths]
Fix distinct nodes $w,y$ with $w\!\rightarrow y$ (strict) or $w *\!\!\rightarrow y$ (relaxed).
A \emph{pre-discriminating path} for $\langle w,y\rangle$ is a path
$\langle x,\dots,v,y\rangle$ whose internal nodes (between $v$ and $x$)
are all colliders and parents of $y$, and whose penultimate edge satisfies
$v \circ\!\!\rightarrow y$.
A \emph{discriminating path} additionally satisfies the usual FCI/PAG
non-adjacency side condition between the endpoints of certain triples (as in
the standard R4/R5 usage); we delegate the subtle case checks to a predicate
$\textsc{ExistsDiscriminatingPath}(x,w,v,y,G)$.
\end{definition}

\vspace{0.5ex}
\noindent
The routine below enumerates all such paths, with two modes:
(i) \emph{strict} (\texttt{checkXYNonadjacency}=\texttt{true}) requires $w\!\rightarrow y$;
(ii) \emph{relaxed} allows a covering edge $w *\!\!\rightarrow y$ but forbids $y\!\rightarrow w$.
We also bound search length by \texttt{maxLen} (set $-1$ for no bound).

\begin{algorithm}
\caption{\textsc{ListDiscriminatingPathsAllPairs}$(G,\texttt{maxLen},\texttt{checkXYNonadjacency})$}
\label{alg:listdisc-all}
\begin{algorithmic}[1]
\STATE $\mathbf{Out}\gets\emptyset$
\FORALL{$w\in V(G)$}
  \FORALL{$y\in \Adj_G(w)$}
    \STATE $\mathbf{Out}\gets \mathbf{Out} \cup\ \textsc{ListDiscriminatingPaths}(G,w,y,\texttt{maxLen},\texttt{checkXYNonadjacency})$
  \ENDFOR
\ENDFOR
\STATE \textbf{return} $\mathbf{Out}$
\end{algorithmic}
\end{algorithm}

\begin{algorithm}
\caption{\textsc{ListDiscriminatingPaths}$(G,w,y,\texttt{maxLen},\texttt{checkXYNonadjacency})$}
\label{alg:listdisc}
\begin{algorithmic}[1]
\STATE $\mathbf{Out}\gets\emptyset$
\IF{\texttt{checkXYNonadjacency} \textbf{and} $\neg(w\!\rightarrow y)$}
  \STATE \textbf{return} $\mathbf{Out}$
\ENDIF
\IF{$\neg$\texttt{checkXYNonadjacency}}
  \IF{$y\!\rightarrow w$}
     \STATE \textbf{return} $\mathbf{Out}$
  \ENDIF
\ENDIF
\STATE $\mathbf{V_{cand}} \gets \Adj_G(w)\cap \Adj_G(y)$ \COMMENT{nodes adjacent to both $w$ and $y$}
\FORALL{$v\in \mathbf{V_{cand}}$}
  \IF{$v\in\{w,y\}$} 
     \STATE \textbf{continue} 
  \ENDIF
  \IF{$y \circ\!\!\rightarrow v$ \textbf{and} $v\!\rightarrow y$} \label{line:circle-arrow}
     \STATE $\textsc{DiscBFS}(G,w,v,y,\texttt{maxLen},\texttt{checkXYNonadjacency},\mathbf{Out})$
  \ENDIF
\ENDFOR
\STATE \textbf{return} $\mathbf{Out}$
\end{algorithmic}
\end{algorithm}

\begin{algorithm}
\caption{\textsc{DiscBFS}$(G,w,v,y,\texttt{maxLen},\texttt{checkXYNonadjacency},\mathbf{Out})$}
\label{alg:discbfs}
\begin{algorithmic}[1]
\STATE Initialize FIFO queue $Q$
\STATE Enqueue state $(t{=}w,\; p{=}\textsc{Null},\; \textit{body}{=}\langle\,\rangle)$
\WHILE{$Q$ not empty}
  \STATE Pop $(t,p,\textit{body})$
  \IF{$p \neq \textsc{Null}$}
     \IF{$p \not\!\!\rightarrow t$}
        \STATE \textbf{continue} \COMMENT{collider at $t$}
     \ENDIF
     \IF{$\neg(t \!\rightarrow y)$}
        \STATE \textbf{continue} \COMMENT{interior node must parent $y$}
     \ENDIF
  \ENDIF
  \FORALL{$x \in \text{nodesInto}(t,\textsc{Arrow})$}
     \IF{($x = p \lor (x \in \textit{body}$)}
        \STATE \textbf{continue}
     \ENDIF
     \STATE $\textit{body}' \gets \textit{body}$ appended with $t$
     \IF{($\texttt{maxLen} \neq -1$) \AND ($|\textit{body}'| > \texttt{maxLen}$)}
        \STATE \textbf{continue}
     \ENDIF
     \STATE Define candidate path $\pi \gets \langle x,\; \textit{body}',\; v,\; y\rangle$ 
     \IF{$\textsc{ExistsDiscriminatingPath}(x,w,v,y,\pi,G,\texttt{checkXYNonadjacency})$}
        \STATE $\mathbf{Out} \gets \mathbf{Out} \cup \{\pi\}$
     \ENDIF
     \IF{$x \!\rightarrow y$}
        \STATE Enqueue $(t{=}x,\; p{=}t,\; \textit{body}{=}\textit{body}')$
     \ENDIF
  \ENDFOR
\ENDWHILE
\end{algorithmic}
\end{algorithm}

\begin{algorithm}
\caption{\textsc{ExistsDiscriminatingPath}$(x,w,v,y,\pi,G,\texttt{checkXYNonadjacency})$}
\label{alg:existsdisc}
\begin{algorithmic}[1]
\STATE \textbf{Input:} candidate path $\pi = \langle x,\ldots,v,y\rangle$; PAG $G$; flag \texttt{checkXYNonadjacency}
\STATE \textbf{Output:} \textbf{true} if $\pi$ forms a valid discriminating path in $G$
\IF{$\neg$Distinct$(x,w,v,y)$} \RETURN \textbf{false} \ENDIF
\IF{\texttt{checkXYNonadjacency} \textbf{and} $x$ adjacent to $y$} \RETURN \textbf{false} \ENDIF
\IF{$v$ not adjacent to $y$} \RETURN \textbf{false} \ENDIF
\IF{$w \notin$ interior nodes of $\pi$} \RETURN \textbf{false} \ENDIF
\STATE Let $\mathbf{p} \gets$ reversed interior of $\pi$ with $x$ prepended and $v$ appended
\FOR{$i=1$ to $|\mathbf{p}|{-}2$}
   \STATE $(n_1,n_2,n_3)\gets(\mathbf{p}[i{-}1],\mathbf{p}[i],\mathbf{p}[i{+}1])$
   \IF{not $\textsc{DefCollider}(n_1,n_2,n_3,G)$} \RETURN \textbf{false} \ENDIF
   \IF{\texttt{checkXYNonadjacency} \textbf{and} $n_2 \not\!\!\rightarrow y$} \RETURN \textbf{false} \ENDIF
   \IF{$\neg$\texttt{checkXYNonadjacency} \textbf{and} $(y,n_2)$ not adjacent \textbf{or} $y\!\rightarrow n_2$}
        \RETURN \textbf{false}
   \ENDIF
\ENDFOR
\IF{$\neg v \rightarrow w$}
    \STATE \textbf{error} ``$v$ must point to $w$''
\ENDIF
\RETURN \textbf{true}
\end{algorithmic}
\end{algorithm}

\paragraph{Remarks and guarantees.}
\begin{itemize}
  \item \emph{Strict vs.\ relaxed mode.} When \texttt{checkXYNonadjacency} is \texttt{true}, we enforce $w\!\rightarrow y$; when \texttt{false}, we allow $w *\!\!\rightarrow y$ but reject $y\!\rightarrow w$. This matches the two modes in our implementation.
  \item \emph{Interior-collider/parent invariant.} Lines 6–8 in Alg.~\ref{alg:discbfs} enforce that every interior node on the upstream body is a collider and a parent of $y$, which is exactly the structure required by (pre-)discriminating paths used by FCI rules.
  \item \emph{Termination and complexity.} The BFS is acyclic on each branch (we forbid repeats in \textit{body}) and can be bounded by \texttt{maxLen}. In the worst case the search is $O\!\left(\sum_{t}\deg^{+}(t)\right)$ per seed with pruning, and typically much faster due to the parent-of-$y$ filter.
  \item \emph{Validation predicate.} The routine $\textsc{ExistsDiscriminatingPath}(\cdot)$ encapsulates the FCI side-conditions (e.g., uncoveredness/non-adjacency of certain triples, endpoint patterns). Keeping these checks local isolates graph-specific details from the traversal.
  \item \emph{Output.} $\mathbf{Out}$ is a set of (pre-)discriminating paths $\pi$ (tuples). No node sets are returned here; all set-valued worklists or outputs are boldfaced.
\end{itemize}

%-------------------------------------------------
\section{Correctness of the Recursive Blocking Procedure}
\label{sec:recursive-blocking-correctness}
%-------------------------------------------------

%------------------------------------------------------------
\subsection{Soundness (formal fix-point lemma)}
\label{sec:rb-soundness}
%------------------------------------------------------------

% \text{Let }
% \mathtt{block\_paths\_recursively}(G,x,y,\textit{containing},\textit{notFollowed},L)
% \text{ denote the recursive blocking procedure with a maximum path-length limit }L.

% \begin{theorem}[Soundness of Recursive Blocking]\label{thm:rb-sound-appendix}
% Let \(G\) be a finite mixed ancestral graph (MAG). Let
% \(\mathtt{block\_paths\_recursively}(G,x,y,\textit{containing},\textit{notFollowed},L)\)
% return a set \(B\) (i.e., it terminates without returning \(\textsc{Null}\)).
% Then
% \[
% \Adj(x,y,G)\ \ \vee\ \ \forall p\,\bigl[\Path(x,y,p,G)\ \Rightarrow\ \Blocked(p,B)\bigr],
% \]
% where \(\Blocked(p,B)\) is \(m\)-blocking by \(B\) in \(G\), and \(\Adj(x,y,G)\) is adjacency in \(G\).
% \end{theorem}

%------------------------------------------------------------
\subsection{Soundness (formal fix-point lemma)}
\label{sec:rb-soundness}
%------------------------------------------------------------

Let 
\(\mathtt{block\_paths\_recursively}(G,x,y,\textit{containing},\textit{notFollowed},L)\)
denote the recursive blocking procedure with a maximum path-length limit \(L\)
(i.e., paths longer than \(L\) edges are not explored).

\begin{theorem}[Soundness of Recursive Blocking]\label{thm:rb-sound-appendix}
Let \(G\) be a finite mixed ancestral graph (MAG). Let
\(\mathtt{block\_paths\_recursively}(G,x,y,\textit{containing},\textit{notFollowed},L)\)
return a set \(B\) (i.e., it terminates without returning \(\textsc{Null}\)).
Then
\[
\Adj(x,y,G)\ \ \vee\ \ \forall p\,\bigl[\Path(x,y,p,G)\ \Rightarrow\ \Blocked(p,B)\bigr],
\]
where \(\Blocked(p,B)\) is \(m\)-blocking by \(B\) in \(G\), and \(\Adj(x,y,G)\) is adjacency in \(G\).
\end{theorem}

\begin{proof}[Proof sketch (via inducing paths)]
Assume the procedure halts and returns B, and suppose for contradiction that
\(\neg\Adj(x,y,G)\) and there exists an x\text{–}y path p that is open given B.
We show that p must then be an \emph{inducing path}, which in a MAG implies
\(\Adj(x,y,G)\)—a contradiction.

\smallskip
\emph{Step 1 (All interiors on p are colliders).}
Let $p=\langle x=v_0,v_1,\ldots,v_k=y\rangle$ be open given B. If some interior node $v_i$ were a noncollider on p with $v_i\notin B$, then the triple $(v_{i-1},v_i,v_{i+1})$ would be traversable under the algorithm’s \texttt{reachable} predicate and would force the recursive call from the first hop $x\to v_1$ ultimately to return \textsc{Unblockable} unless $v_i$ were added to B. Since the final outcome is that the routine halts \emph{without} ever returning \textsc{Unblockable}, this cannot happen. Hence every interior node of p is a collider or lies in B. But if an interior noncollider were in B, p could not be open given B. Therefore, every interior node on p is a collider.

\smallskip
\emph{Step 2 (No activation via descendants; colliders not in B).}
Suppose a collider v on p were activated via a strict descendant $d\in B$. Along the directed segment $v\!\to\!\cdots\!\to d$, let u be the first noncollider encountered on the witnessed continuation. If the segment had length 1 ($v\!\to d$), then d would be a direct noncollider child of v; in that case the continuation would already be blocked once d entered B, and v would not be activated via d. Hence the segment must have length $\ge 2$, so such a u exists. By the algorithm’s minimality invariant—it adds each necessary noncollider as soon as it is encountered on an open continuation—u would have been added to B before d, contradicting minimality. Hence no collider on p is activated via a strict descendant in B. Moreover, by invariant, the procedure never adds colliders to B. Therefore colliders on p are not activated by membership in B.

\smallskip
\emph{Step 3 (Inducing path).}
By Step 1, all interiors of p are colliders. By Step 2 and the invariant that B contains only noncolliders, each interior collider must be an ancestor of an endpoint (or of a selection node, if present). Hence every interior node is a collider in $\mathrm{An}(\{x,y\}\cup S)$, so p is an inducing path.

\smallskip
\emph{Step 4 (Inducing path implies adjacency in a MAG).}
By maximality of MAGs, if there exists an inducing path between two distinct vertices, those vertices are adjacent. Thus the existence of the open path p (which we have shown must be inducing) implies \(\Adj(x,y,G)\), contradicting our assumption
\(\neg\Adj(x,y,G)\).

\smallskip
Therefore, if the routine halts with B and x and y are nonadjacent, no x\text{–}y path remains open given B. Equivalently, \(\forall p\,(\Path(x,y,p,G)\Rightarrow \Blocked(p,B))\).
\end{proof}

\begin{corollary}[Adjacency Case]\label{cor:rb-adj}
Under the assumptions of Theorem~\ref{thm:rb-sound}, suppose
$\mathtt{block\_paths\_recursively}(G,x,y,\dots)$ halts normally
(returning a non-null set B) and that x and y are adjacent in G.
Then every non-inducing x{\text{–}}y path in G is m-blocked by B:
\[
\Adj(x,y,G)\ \wedge\ B\neq\textsc{Null}\ \Longrightarrow\
\forall p\,\bigl[\Path(x,y,p,G)\ \wedge\ \neg\Inducing(p)\ \Rightarrow\ \Blocked(p,B)\bigr].
\]
\end{corollary}

Proof sketch, When \(\Adj(x,y,G)\), there exists at least one inducing path between x and y (the path defining the adjacency).
We must show that every non-inducing x{\text{–}}y path p is m-blocked by the final blocking set B returned by the procedure.

\smallskip
\emph{Step 1 (Exploration over neighbors of x).} For each neighbor \(b\in\Adj_G(x)\) with $b\neq y$, the algorithm invokes $\texttt{find\_path\_to\_target}(x,b,y,\dots)$ to explore every admissible continuation from x through b toward y, subject to the same \texttt{reachable} predicate used in the nonadjacent case. By the reachability completeness invariant, every potential open continuation of a non-inducing x{\text{–}}y path will eventually be examined.

\smallskip
\emph{Step 2 (Blocking of noncollider continuations).} Whenever the exploration encounters a triple (a,v,c) in which v is a noncollider not yet in B, the algorithm—by its noncollider-only addition rule—either (i) adds v to B to block that continuation, or (ii) returns \textsc{Unblockable}. Since the assumption here is that the routine halts normally with non-null B (i.e., no branch ever returned \textsc{Unblockable}), case (ii) never occurs. Therefore, every reachable noncollider continuation is eventually closed by adding the relevant noncollider to B.

\smallskip
\emph{Step 3 (Halting and completeness).} By the halting condition, termination implies that no further reachable triple with a noncollider $v\notin B$ remains open between x and y. Hence, after halting, every $x{\text{–}}y$ path whose interior nodes include at least one noncollider is m-blocked by the final B.

\smallskip
\emph{Step 4 (Remaining open paths are inducing).} The only x{\text{–}}y paths that can remain m-open given B are those whose interior nodes are all colliders and each such collider is an ancestor of an endpoint. These are precisely the inducing paths in G. In a MAG, inducing paths correspond exactly to adjacencies.

\smallskip
\emph{Step 5 (Conclusion).} Therefore, when \(\Adj(x,y,G)\) and the routine halts with non-null B, every non-inducing x{\text{–}}y path is m-blocked by B, while any remaining open path must be inducing, accounting for the observed adjacency between x and y.
\qed

\begin{remark}[Generality of the result]
Although Theorem~\ref{thm:rb-sound-appendix} and Corollary~\ref{cor:rb-adj} are stated for MAGs, the recursive blocking argument applies unchanged to
DAGs, CPDAGs, and PAGs. In DAGs and CPDAGs, \(m\)-separation reduces to \(d\)-separation, and in PAGs the same reasoning holds with mixed endpoints.
\end{remark}

%------------------------------------------------------------
\subsection{Completeness (over the FCI search space)}
\label{sec:rb-completeness}
%------------------------------------------------------------

\begin{lemma}[Domain restriction]\label{lem:rb-domain}
Let $G$ be a PAG/MAG and fix $x\neq y$. Every vertex $b$ that $\mathtt{block\_paths\_recursively}(G,x,y,\dots)$
adds to $B$ lies in the standard FCI search space $\mathcal{S}(x,y)=\Adj(x)\cup\Adj(y)\cup\PossibleDsep(x,y)$.
\end{lemma}

\begin{proof}[Proof sketch]
Recursive blocking only proposes conditioning on vertices that occur as \emph{noncolliders} on paths it actively traverses from $x$ toward $y$. Every interior node on such a traversed path is either a collider or an ancestor of a collider (otherwise the traversal would stop), and triangles are only crossed when the middle node is not a definite noncollider. This is precisely the inclusion criterion for $\PossibleDsep(x,y)$ in a PAG. Neighbors of $x$ or $y$ are in $\Adj(x)$ or $\Adj(y)$. Hence any $b$ considered by recursive blocking lies in $\mathcal{S}(x,y)$.
\end{proof}

\begin{lemma}[Termination]\label{lem:rb-termination}
Fix $x\neq y$ and let $U=\mathcal S(x,y)$. Along any run of $\mathtt{block\_paths\_recursively}(G,x,y,\dots)$ with no path-length cap and $\textit{notFollowed}=\varnothing$, the sequence of conditioning sets $B_0\subset B_1\subset\cdots$ produced by the algorithm has strictly
increasing $B_i$ and is bounded by $U$. Hence the run performs at most $|U|$ additions and halts.
\end{lemma}

\begin{proof}[Proof sketch]
By construction, RB adds only fresh vertices (never removes), so $B$ grows strictly when it grows. By Lemma~\ref{lem:rb-domain}, all additions lie in $U$, which is finite. Thus there is no infinite sequence of proper inclusions.
\end{proof}

\begin{corollary}[Completeness over $\mathcal S$]\label{cor:rb-complete-simple}
Let $U=\mathcal S(x,y)$ and assume the oracle setting with no path-length cap and $\textit{notFollowed}=\varnothing$. If there exists $S\subseteq U$ with $\test{m-sep}(x, y \mid S)$, then $\mathtt{block\_paths\_recursively}(G,x,y,\dots)$ halts and returns some $B\subseteq U$ with $\text{m-sep}(x, y \mid B)$.
\end{corollary}

\begin{proof}[Proof sketch]
% By Lemma~\ref{lem:rb-noescape}, the run never aborts early.
By
Lemma~\ref{lem:rb-termination}, it must halt after finitely many additions. At halt, Theorem~\ref{thm:rb-sound-appendix} (soundness) shows that all $x$--$y$ paths are blocked by $B$, i.e. $\text{m-sep}(x, y \mid B)$. By Lemma~\ref{lem:rb-domain}, $B\subseteq U$.
\end{proof}

%------------------------------------------------------------
\subsection{Certification Completeness}
%------------------------------------------------------------

\begin{proposition}[Certification completeness under Markovness]\label{prop:fcit-cert}
Assume the oracle setting where \(G\) is a MAG or PAG Markov to the true distribution~\(\mathbb P\)
(with no Faithfulness assumed), and that
\(\mathtt{block\_paths\_recursively}(G,x,y,\dots)\)
halts with blocking set \(B \subseteq \mathcal S(x,y)\).
Let \(C_{xy} = \Adj_G(x) \cap \Adj_G(y)\) denote the common adjacents of \(x\) and \(y\), and define
\[
\mathcal F_{xy} = \{\, B \setminus D' \mid D' \subseteq C_{xy} \,\}.
\]
Then, if there exists any conditioning set \(S\) (not necessarily contained in \(\mathcal S(x,y)\))
such that \(X \ind Y \mid S\) in~\(\mathbb P\),
there exists some \(S' \in \mathcal F_{xy}\) with \(X \ind Y \mid S'\).
\end{proposition}

\begin{proof}[Proof sketch]
\emph{(i) RB baseline.}
By Theorem~\ref{thm:rb-sound-appendix}, the blocking set \(B\) \(m\)-separates \(x\) and \(y\) in \(G\). Each variable is added to \(B\) only when required to block a reachable open noncollider continuation, ensuring that \(B\) is locally minimal with respect to all explored \(x\text{--}y\) paths.

\smallskip
\emph{(ii) Essential vs.\ dispensable variables.} Every \(v \in B \setminus C_{xy}\) was added to block a specific non-inducing \(m\)-connecting path. Removing such a \(v\) would reopen that path and destroy separation. Hence, any valid separating set for \((x,y)\) must include all of these “essential’’ vertices. By contrast, vertices in \(C_{xy}\) may not lie on any active non-inducing path; they can sometimes be dropped without re-opening a connection.

\smallskip
\emph{(iii) Completeness of \(\mathcal F_{xy}\).}
The family \(\mathcal F_{xy}\) enumerates all sets obtained by optionally omitting some or all of the potentially dispensable common neighbors \(C_{xy}\) from \(B\), while retaining all essential blockers. Thus, if there exists any true separating set \(S\) in \(\mathbb P\), it must coincide with or be a subset of some member of \(\mathcal F_{xy}\), since any smaller set lacking an essential blocker would fail to \(m\)-separate \(x\) and \(y\).

\smallskip
\emph{(iv) Markov-consistent certification.}
Because \(G\) is Markov to \(\mathbb P\), every \(S' \in \mathcal F_{xy}\) that \(m\)-separates \(x\) and \(y\) also satisfies \(X \ind Y \mid S'\) in~\(\mathbb P\). Hence, if \(X \ind Y \mid S\) holds in the true distribution, there exists a corresponding \(S' \in \mathcal F_{xy}\) that certifies the same independence. Deleting the edge \(x *\!-\!* y\) or orienting the induced collider preserves Markovness and prepares the graph for the subsequent FCI orientation phase.
\qedhere
\end{proof}

%------------------------------------------------------------
\subsection{Effects of Unfaithfulness}
\label{sec:rb-unfaithfulness}
%------------------------------------------------------------

The preceding result assumes only Markovness, not Faithfulness.  In the oracle setting this is sufficient for correctness: every conditional independence detected by recursive blocking corresponds to a genuine \(m\)-separation in the underlying MAG, and each certified separating set \(S' \in \mathcal F_{xy}\) therefore reflects a true structural independence.  However, in finite-sample or unfaithful situations, apparent independences may arise that are not \(m\)-separations in any causal MAG.

\smallskip
An \emph{unfaithful independence} occurs when two variables are statistically independent only because of numerical parameter cancellation rather than structural separation.  If FCIT acts on such a spurious independence—by deleting an edge or orienting a collider—the resulting graph may violate the ancestral or maximal properties that define a legal PAG.  In this case no valid MAG can realize the asserted independence, and the resulting structure ceases to represent a Markov equivalence class.

\smallskip
To ensure interpretability of search results, the implementation therefore includes a \emph{legality reversion} step: whenever an update would yield an illegal PAG, the modification is reverted and the current graph is projected back into the space of legal PAGs.  This step has no effect under Faithfulness, where all detected independences are structural and the procedure preserves legality automatically, but it provides robustness in the presence of unfaithful or noisy data by preventing the search from producing non-representable graphs.

% -----------------------------------------------------
\subsection{Relationship to Zhang's m-Separation}
\label{sec:rb-zhang}
%-----------------------------------------------------

Algorithms~\ref{alg:block_paths_recursively} and~\ref{alg:find_path_to_target} operate on a PAG or MAG \(G=(V,E)\), two distinct vertices \(x,y\in V\), a set of mandatory conditioning vertices \(C\subseteq V\) (\texttt{containing}), and a set of vertices that may \emph{not} be traversed \(F\subseteq V\) (\texttt{notFollowing}). They return a blocking set \(B\) obeying \(C\subseteq B\subseteq V\setminus F\).

A direct edge \(x *\!-\!* y\) has no interior vertices and therefore provides no opportunity to satisfy the blocking clauses by conditioning on a noncollider or leaving a collider unactivated.  Hence, while clause (ii) (“all colliders are in or have descendants in \(B\)”) holds trivially, this does not imply independence as long as \(x\) and \(y\) remain adjacent.  To certify \(m\)-separation, both adjacency and path conditions must be satisfied.  Zhang’s criterion~\citep{zhang2008completeness} explicitly requires
\[
  \neg\Adj(x,y,G)
  \quad\text{and}\quad
  \forall p\,\Blocked(p,B),
\]
which matches the structure of our Soundness theorem.  Once the edge \(x *\!-\!* y\) is removed, Soundness supplies the second clause, ensuring that all \(x\text{--}y\) paths are blocked.  The following observations then establish constructive equivalence:

\begin{enumerate}[label=(\roman*),leftmargin=*]
  \item Once \(x\) and \(y\) become nonadjacent, the Fix-point Lemma ensures that the current blocking set \(B\) blocks every \(x\text{--}y\) path; hence \(\text{m-sep}(x,y\mid B)\).

  \item Conversely, if there exists a set \(B^{\star}\) with \(\text{m-sep}(x,y\mid B^{\star})\), the algorithm eventually halts while \(x\) and \(y\) are still nonadjacent.  Monotone growth gives \(C\subseteq B\subseteq B^{\star}\), and Soundness yields \(\text{m-sep}(x,y\mid B)\).  Thus the procedure always returns a valid separating set whenever one exists.
\end{enumerate}

\smallskip
\noindent
\emph{Hence, whenever \(x\) and \(y\) are \(m\)-separable, the recursive blocking procedure finds a separating set consistent with the user’s constraints.}
% -----------------------------------------------------
\subsection{Complexity}\label{sec:rb-complexity}
Let $\Delta_x = |\mathrm{Adj}_G(x)|$ and let $\ell$ be the path-length cap used by the
routine (\texttt{maxPathLength}). In each call, the recursion explores a binary
branch (first without, then with $b \in \mathbf{B}$), and along any single branch
the \texttt{path} guard prevents revisiting vertices. The cost of one branch is
linear in the size of the explored subgraph (dominated by adjacency scans), i.e.,
$\mathcal{O}(|E|)$ in the worst case.

\paragraph{Bounded depth.}
If $\ell$ is a fixed constant, the total number of branches is $\mathcal{O}(2^\ell)$,
so the overall complexity is
\[
  \mathcal{O}\!\left(\Delta_x \cdot 2^\ell \cdot |E|\right).
\]
Since $2^\ell$ is constant, this is effectively $\mathcal{O}(\Delta_x\,|E|)$.

\paragraph{Unbounded depth.}
If $\ell$ grows with the graph (no path-length cap), the worst-case number of
branches can be exponential in the path length, yielding
\[
  \mathcal{O}\!\left(\Delta_x \cdot 2^{\Theta(\ell)} \cdot |E|\right),
\]
which in the worst case is exponential in $|V|$.

\paragraph{Space.}
The algorithm stores $\mathbf{B}$ and the current \texttt{path}, each of size
$\mathcal{O}(|V|)$, plus the descendants index (precomputed once), which is
$\mathcal{O}(|V|+|E|)$.

\subsection{Concrete Example: Peter’s MAG}

Consider the MAG in Figure~\ref{fig:peters-mag} with the following edges:
\[
  x \;\to\; b,\quad
  b \;\to\; z,\quad
  z \;\leftrightarrow\; y,\quad
  x \;\leftrightarrow\; a,\quad
  a \;\leftrightarrow\; y,\quad
  a \;\to\; b.
\]
We want to determine whether \(x\) and \(y\) are m-separable.

\paragraph{1. Initial Conditioning Set.}
We begin with \(B = \varnothing\). The recursive blocking driver repeatedly calls \(\textit{find\_path\_to\_target}(x, y, B)\) to detect any m-unblocked path from \(x\) to \(y\).

\paragraph{2. Checking Paths.}
A candidate path is
\[
  x \;\to\; b \;\to\; z \;\leftrightarrow\; y.
\]
Since \(b \to z\) and \(z \leftarrow y\) (by virtue of \(z \leftrightarrow y\)), 
the node \(z\) is a \emph{collider}. 
A collider path is BlockedAfter by default if neither the collider nor its descendants are in \(B\). Because \(B=\varnothing\), the path is already inactive. Hence, the procedure has no need to add \(b\). Likewise, the path \(x \leftrightarrow a \leftrightarrow y\) or any variant is also BlockedAfter 
under \(\varnothing\). 

\paragraph{3. No Node Added.}
The algorithm finds \emph{no} m-unblocked path and therefore terminates without modifying \(B\). Consequently, \(B\) remains empty. Indeed, \(\varnothing\) suffices to m-separate \(x\) and \(y\).

\paragraph{4. Conclusion.}
This example shows that, with correct collider detection, the path \(x \to b \to z \leftrightarrow y\) is recognized as BlockedAfter under \(\varnothing\). A naive method that tries to block \(z\) because of the noncollider path \(x \rightarrow b \rightarrow z\) would incorrectly add \(b\). By contrast, the proper recursive blocking procedure sees no reason to alter \(B\), returning \(\varnothing\). This illustrates the crucial role of the \textit{find\_path\_to\_target} method.

\begin{figure}
    \centering
    \includegraphics[width=0.3\linewidth]{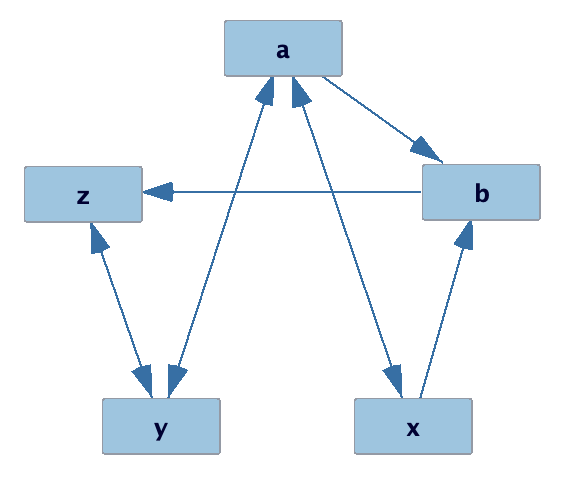}
    \caption{Peter’s MAG example. Since \(z\) is a collider on the path 
    \(x \to b \to z \leftrightarrow y\), it is already BlockedAfter under \(\varnothing\). 
    No additional conditioning is needed.}
    \label{fig:peters-mag}
\end{figure}

\section{Optimizing the Zhang Rules}
\label{sec:optimizing_zhang_rules}

In FCIT, the final FCI orientation rules are applied repeatedly, and some of these rules can become bottlenecks in practice. The Zhang rules are given in Section \ref{sec:zhang_rules}. In particular, we focus on optimizing four rules, whose standard implementations are typically the most time-consuming:

\begin{itemize}
    \item \textbf{Rule R4}, which requires searching for \emph{discriminating paths} and can involve conditioning on many subsets of adjacent nodes of varaibles.
    \item \textbf{Rules R5 and R9}, both of which involve identifying \emph{all-circle-endpoint} paths of potentially unbounded length.
    \item \textbf{Rule R10}, which requires finding certain \emph{semidirected paths} to finalize circle endpoints.
\end{itemize}

Below, we give our optimizations for each of these.

\subsection{Zhang's Final FCI Orientation Rules}
\label{sec:zhang_rules}

For completeness, we recap here the orientation rules (R0--R10) that Zhang~\cite{zhang2008completeness} proves are both sound and complete for FCI in the presence of latent variables and selection bias. In these rules, \(\alpha,\beta,\gamma,\theta\) denote generic vertices, and the symbol ``\(\ast\)'' is a wildcard that can represent an arrowhead (\(\rightarrow\)), a tail (\(-\)), or a circle (\(\circ\)). Any rule that does not apply in a given configuration can be ignored.

\begin{description}[leftmargin=0pt]
\item[R0 \ (Unshielded collider):] 
  For every unshielded triple \(\langle \alpha, \gamma, \beta \rangle\)
  (i.e.\ \(\alpha\) and \(\beta\) both adjacent to \(\gamma\) but not to each other):
  \begin{itemize}
    \item If \(\gamma \not\in \text{Sepset}(\alpha,\beta)\), orient \(\alpha \ast\to \gamma \leftarrow\!\ast \beta\).
  \end{itemize}

\item[R1:]
  If \(\alpha \ast\to \beta \;\circ-\!\ast \gamma\) and \(\alpha\) and \(\gamma\) are not adjacent, orient 
  \(\beta \to \gamma\).
  
\item[R2:]
  If either 
  \(\alpha \to \beta \ast\to \gamma\)
  \textbf{or}
  \(\alpha \ast\to \beta \to \gamma\),
  and there is an unoriented edge \(\alpha \ast -\!\circ \gamma\), orient \(\alpha \ast\to \gamma\).

\item[R3:]
  If \(\alpha \ast\to \beta \leftarrow\!\ast \gamma\), and 
  \(\alpha \ast-\!\circ \theta \circ-\!\ast \gamma\), 
  with \(\alpha\) and \(\gamma\) not adjacent, and
  \(\theta \ast-\!\circ \beta\),
  then orient \(\theta \ast\to \beta\).

\item[R4 \ (Discriminating path):]
  Let \(u = \langle \theta, \dots, \alpha, \beta, \gamma \rangle\) be a \emph{discriminating path} between \(\theta\) and \(\gamma\) for \(\beta\), and suppose \(\beta \circ-\!\ast \gamma\) is currently unoriented. Then:
  \begin{itemize}
    \item If \(\beta \in \text{Sepset}(\theta,\gamma)\), orient \(\beta \circ-\!\ast \gamma\) as \(\beta \to \gamma\).
    \item Otherwise orient \(\langle \alpha, \beta, \gamma\rangle\) as \(\alpha \leftrightarrow \beta \leftrightarrow \gamma\).
  \end{itemize}

\item[R5:]
  If \(\alpha \circ-\!\circ \beta\) and there is an \textit{uncovered circle-only path}
  \(\langle \alpha, \ldots, \beta \rangle\) such that \(\alpha\) and \(\beta\) are \emph{not} adjacent to
  one another's immediate neighbors on that path, orient \(\alpha \circ-\!\circ \beta\) (and every edge on the path) as undirected, i.e.\ \(\alpha - \beta\).

\item[R6:]
  If \(\alpha - \beta \;\circ-\!\ast \gamma\),
  then orient \(\beta \circ-\!\ast \gamma\) as \(\beta - \gamma\).

\item[R7:]
  If \(\alpha - \beta \;\circ-\!\circ \gamma\) and \(\alpha\) and \(\gamma\) are not adjacent, orient 
  \(\beta \circ-\!\circ \gamma\) as \(\beta - \gamma\).

\item[R8:]
  If 
  \(\alpha \to \beta \to \gamma\)
  \textbf{or}
  \(\alpha - \beta \to \gamma\),
  and \(\alpha \circ\!\to \gamma\) is present, orient \(\alpha \circ\!\to \gamma\) as \(\alpha \to \gamma\).

\item[R9:]
  If \(\alpha \circ\!\to \gamma\) and there is an \emph{uncovered potentially-directed path} 
  \(\langle \alpha, \beta, \theta, \ldots, \gamma \rangle\)
  such that \(\beta\) and \(\gamma\) are not adjacent, then orient \(\alpha \circ\!\to \gamma\) as \(\alpha \to \gamma\).

\item[R10:]
  Suppose \(\alpha \circ\!\to \gamma\),
  \(\beta \to \gamma \leftarrow \theta\),
  and there exist uncovered potentially-directed paths 
  \(p_1: \alpha \to^\ast \beta\) 
  and 
  \(p_2: \alpha \to^\ast \theta\), 
  where the two neighbors of \(\alpha\) on \(p_1\) and \(p_2\) (denote them \(\nu\) and \(\omega\)) are distinct and not adjacent. In that case, orient \(\alpha \circ\!\to \gamma\) as \(\alpha \to \gamma\).
\end{description}

\subsection{Optimizing R4 Using Recursive Blocking}
\label{sec:opt-r4}

A major bottleneck among the final orientation rules is \textbf{Rule~R4}, which seeks \emph{discriminating paths} of the form
\[
  X *\!\rightarrow \dots \leftrightarrow\; W \;\leftarrow\!\circ\; V' \;\circ\!\rightarrow\; Y,
\]
where each internal collider along the path is a parent of \(Y\). In the standard approach \citep{spirtes2001causation}, applying R4 often requires exploring many potential conditioning sets, which is time-consuming and can lead to inaccuracies in dense graphs.

To mitigate this, we leverage the \emph{recursive blocking procedure} (Section~\ref{alg:block_paths_recursively}) to focus only on a small, targeted collection of conditioning sets. Specifically:

\begin{enumerate}
\item \textbf{Enumerate discriminating paths.}
  Identify all valid and unoriented discriminating paths from \(X\) to \(Y\). Because FCIT prunes many extraneous edges beforehand, the number of such paths is usually modest.

\item \textbf{Collect collider candidates.}
  For each path, record the possible collider nodes (including \(V'\) if its endpoint with \(Y\) is a circle). These nodes form a small set of “candidate colliders” that we either include or exclude in our conditioning sets.

\item \textbf{Apply recursive blocking.}
  For each subset \(Q\) of these candidate colliders, call the \textit{block\_paths\_recursively} subroutine to build a minimal set \(S\) that blocks all active paths from \(X\) to \(Y\)—assuming each node in \(Q\) is treated as a collider.

\item \textbf{Test subsets of common adjacents.}
  Let \(C\) be the set of common adjacents of \(X\) and \(Y\). For each subset \(C' \subseteq C\), we form \(S' = S \setminus C'\) (or vice versa, depending on how $Q$ is handled) and test whether \(X \ind Y \mid S'\). If \(X \ind Y \mid S'\) and \(V' \notin S'\), we orient the relevant structure as a new collider (e.g., \(W \leftrightarrow V' \leftarrow Y\)); otherwise, we orient \(\alpha \circ\!\to \gamma\) as a noncollider.
\end{enumerate}

By confining the conditioning sets to those implied by the recursive blocking procedure on legitimate discriminating paths, we drastically reduce both the path search space and the number of conditional independence tests performed. We give pseudocode for this procedure in Algorithm \ref{alg:ddp_recursive}.

\subsection{Shortest-Path Searches for R5 and R9}
\label{sec:r5-r9}

Rules \textbf{R5} and \textbf{R9} require checking whether a path composed entirely of circle endpoints exists between two nodes. Naively, one might enumerate all such paths, but this is prohibitively expensive for larger graphs. Instead, we construct an \emph{auxiliary circle-only} graph:

\begin{itemize}
\item The node set is the same as in the original PAG.
\item We add an undirected edge \(X - Y\) in the auxiliary graph if and only if \(X \circ\!\!-\!\!\circ Y\) in the PAG.
\end{itemize}

We then run a shortest-path algorithm (e.g., BFS or Dijkstra’s, typically BFS with uniform edge weights) from \(X\). If \(Y\) is reachable in this auxiliary graph, then at least one all-circle path exists in the original PAG. Because most edges with definite tails or arrowheads are excluded, the auxiliary graph is often much sparser, allowing us to detect or rule out circle-only paths much more efficiently than exhaustive enumeration.

\subsection{Uncovered Potentially-Directed Path Checking for R10}

\label{sec:r10-approach}

\textbf{Rule R10} is among the costliest FCI orientation rules because it requires verifying the existence of \emph{uncovered, potentially-directed (PD)} paths.
The rule applies when
$\alpha \circ\!\to \gamma,\quad$
$\beta \to \gamma \leftarrow \theta,$
and there exist uncovered PD paths
$p_1 : \alpha \to^* \beta$ and $p_2 : \alpha \to^* \theta$
whose first hops from~$\alpha$ are through distinct, nonadjacent neighbors~$\mu$ and~$\omega$.
If so, the circle on~$\alpha \circ\!\to \gamma$ is replaced by a tail ($\alpha \to \gamma$).

To guarantee correctness, we use an explicit \emph{depth-first search} that enforces both constraints:
(i) every step must follow an edge with no arrowhead into the previous node (potentially-directed), and
(ii) each triple must be unshielded (uncovered).
Subproblems are memoized to avoid redundant traversals.
This yields an exact implementation of R10 that ensures all orientations remain valid under the oracle semantics of Zhang’s rules.

\section{Supplementary Traces}

\subsection*{Supplementary Trace 1: Recursive Blocking Example}
\label{sec:supplementary_trace}

This section provides a worked example tracing the behavior of \textit{block\_paths\_recursively} on a case where the score-based DAG includes a spurious edge that FCIT correctly removes through recursive blocking. This example shows the recursive blocking procedure removing an extraneous edge due to latent confounding.

\paragraph{True model.} The true causal DAG includes an unobserved confounder $L$ that influences both $Y$ and $Z$, and is given by:
\[
X \rightarrow Y \leftarrow L \rightarrow Z \leftarrow W
\]
Here, $L$ is latent. The observed variables are $X$, $Y$, $Z$, and $W$. The induced PAG over these observed variables has the following structure:
\begin{itemize}
    \item $X \rightarrow Y$
    \item $W \rightarrow Z$
    \item $Y \leftrightarrow Z$ (due to unmeasured confounding via $L$)
\end{itemize}

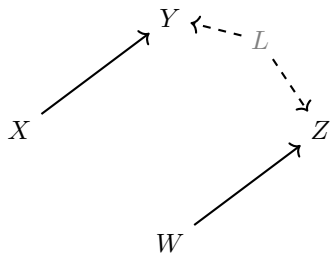
\begin{figure}[H]
\centering
\begin{tikzpicture}[->, thick, node distance=2cm, scale=1, every node/.style={transform shape},
    latent/.style={draw=none, fill=none, text=gray},
    bidirected/.style={<->, dashed, thick}]

% Nodes
\node (X) at (0,0) {$X$};
\node (Y) at (2,1.5) {$Y$};
\node (Z) at (4,0) {$Z$};
\node (W) at (2,-1.5) {$W$};
\node[latent] (L) at (3.2, 1.2) {$L$};

% Arrows
\draw[->] (X) -- (Y);
\draw[->] (W) -- (Z);
\draw[->, dashed] (L) -- (Y);
\draw[->, dashed] (L) -- (Z);

\end{tikzpicture}
\caption{True causal DAG including latent variable $L$ (dashed arrows). Only $X$, $Y$, $Z$, and $W$ are observed.}
\label{fig:true_dag}
\end{figure}

\paragraph{Score-derived DAG.} Since the score-based method cannot represent latent confounding directly, it introduces the edge $X \rightarrow Z$ to cover the collider $X *-> Y <-* Z$ so that the graph may remain Markov to the data. The resulting DAG is:

\[
X \rightarrow Y \rightarrow Z \leftarrow W,\quad X \rightarrow Z
\]

\begin{figure}[H]
\centering
\begin{tikzpicture}[->, thick, node distance=2cm, scale=1, every node/.style={transform shape}]

% Nodes
\node (X) at (0,0) {$X$};
\node (Y) at (2,1.5) {$Y$};
\node (Z) at (4,0) {$Z$};
\node (W) at (2,-1.5) {$W$};

% Arrows
\draw[->] (X) -- (Y);
\draw[->] (Y) -- (Z);
\draw[->] (W) -- (Z);
\draw[->] (X) to[bend left=15] (Z);

\end{tikzpicture}
\caption{DAG returned by score-based structure search. The edge $X \rightarrow Z$ does not belong in the true PAG and will be tested for removal.}
\label{fig:score_dag}
\end{figure}
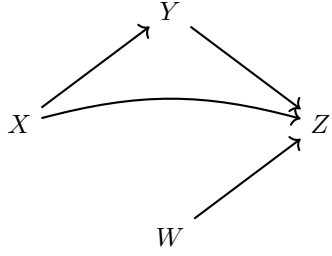

\paragraph{Blocking analysis.} We apply \textit{block\_paths\_recursively($X$, $Z$)} to test whether the edge $X \rightarrow Z$ can be removed.

\begin{table}[H]
\centering
\begin{tabular}{@{}lll@{}}
\toprule
\textbf{P} & \textbf{Type} & \textbf{Status under $B = \{Y\}$} \\
\midrule
$X \rightarrow Z$                & Direct edge        & Active \\
$X \rightarrow Y \rightarrow Z$ & Non-collider chain & Blocked (since $Y \in B$) \\
$X \rightarrow Z \leftarrow W$  & Collider at $Z$    & Blocked (since $Z \notin B$) \\
\bottomrule
\end{tabular}
\vspace{0.5em}
\caption{Paths from $X$ to $Z$ considered by \textit{block\_paths\_recursively}. All non-collider paths are BlockedAfter by $B = \{Y\}$.}
\label{tab:blocking_paths}
\end{table}

The recursive blocking procedure returns:
\[
\boxed{B = \{Y\}}
\]

Next, FCIT computes the set of common adjacents $C = \text{Adj}(X) \cap \text{Adj}(Z) = \{Y\}$, and tests conditional independence over all subsets of $B \setminus D$ for $D \subseteq C$:

\begin{table}[H]
\centering
\begin{tabular}{@{}lll@{}}
\toprule
\textbf{Subset $D \subseteq C = \{Y\}$} & \textbf{Tested set $B \setminus D$} & \textbf{Independence result} \\
\midrule
$\emptyset$     & $\{Y\}$     & Dependent (fails) \\
$\{Y\}$         & $\emptyset$ & Independent (succeeds) \\
\bottomrule
\end{tabular}
\vspace{0.5em}
\caption{Conditional independence tests performed by FCIT for $X \ind Z$. The independence is found given the empty set.}
\label{tab:ci_tests}
\end{table}

\paragraph{Conclusion.} Since $X \ind Z \mid \emptyset$, FCIT removes the edge $X \rightarrow Z$. Because the separating set does not contain $Y$, the unshielded triple $X - Y - Z$ is oriented as a collider:
\[
X \rightarrow Y \leftarrow Z
\]

This demonstrates how FCIT successfully recovers latent-induced collider structure by combining recursive blocking with selective independence testing.

\subsection*{Supplementary Trace 2: Discriminating Path Required for Edge Removal}
\label{sec:discriminating_path_trace}

This example illustrates how a discriminating path must be used during the edge removal phase so that FCIT can correctly orient a collider and remove a spurious edge. Without including this, incorrect adjacencies and orientations persist in the output PAG.

\paragraph{True DAG.} The true DAG is in Figure~\ref{fig:graph_trio} (a). Measured nodes are shown in rectangular boxes, and latent variables are shown in elliptical boxes.

\begin{figure}[H]
\centering
\begin{subfigure}[t]{0.31\textwidth}
    \centering
    \includegraphics[width=\linewidth]{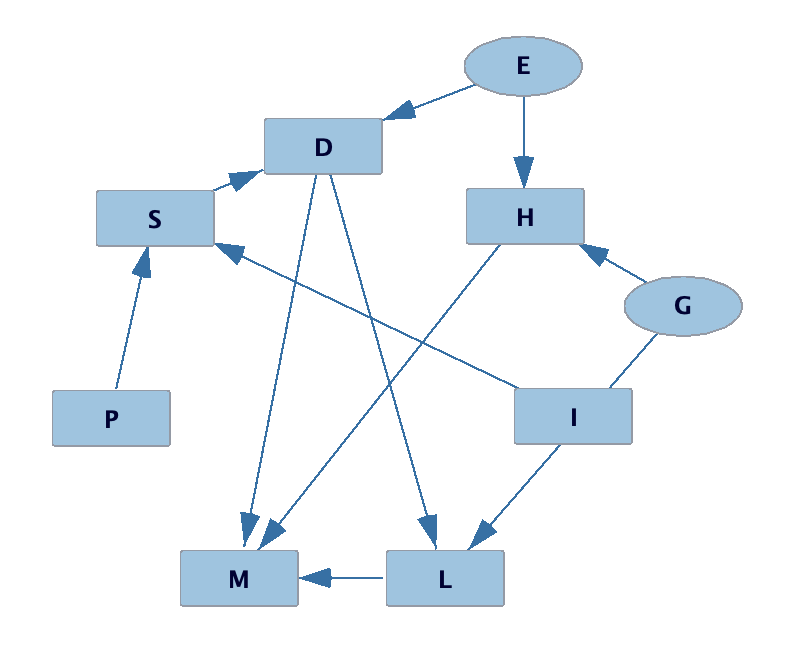}
    \caption{True DAG}
    \label{fig:true_dag_png}
\end{subfigure}
\hfill
\begin{subfigure}[t]{0.31\textwidth}
    \centering
    \includegraphics[width=\linewidth]{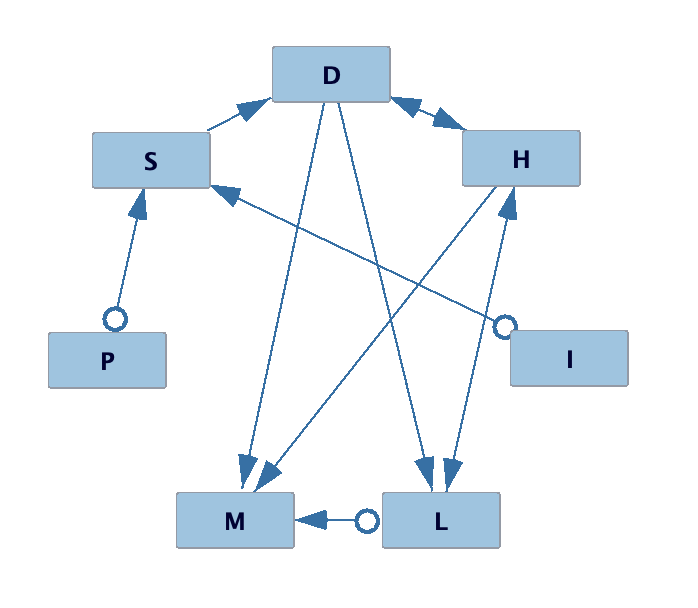}
    \caption{Induced PAG}
    \label{fig:true_pag_png}
\end{subfigure}
\hfill
\begin{subfigure}[t]{0.31\textwidth}
    \centering
    \includegraphics[width=\linewidth]{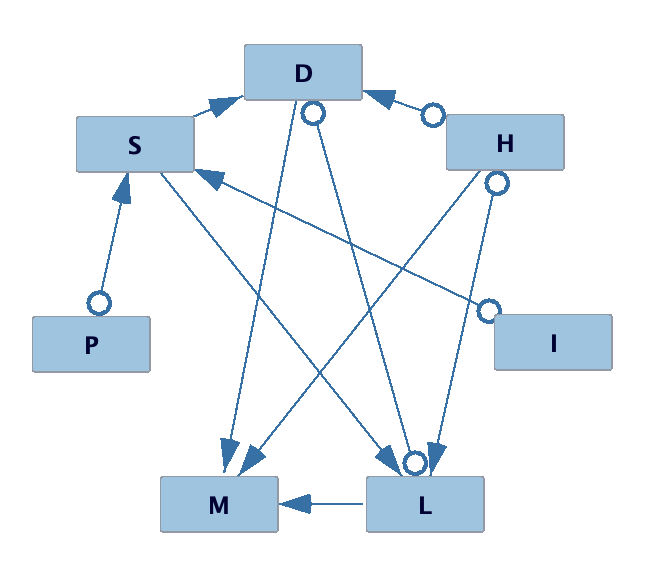}
    \caption{FCIT without discriminating path edge removal step}
    \label{fig:fci_tt_no_dp_png}
\end{subfigure}
\caption{Graph visualizations generated using Tetrad. Left: the true DAG (including latent variables). Middle: the correct PAG over observed variables. Right: PAG returned by FCIT when the discriminating path edge removal rule is disabled.}
\label{fig:graph_trio}
\end{figure}

\paragraph{Key structure.} In the induced PAG, the triple $D *$--$L$--$* M$ is initially unoriented. FCIT uses the discriminating path:
\[
D \leftrightarrow H \leftrightarrow L \rightarrow M
\]
to determine that $L$ is a collider in this triple. Orienting $D \rightarrow L \leftarrow M$ enables the recursive blocking procedure to find a separating set between $D$ and $M$, allowing the spurious edge to be removed.

Table~\ref{tab:pag_comparison_summary} shows a summary of the differences between the PAG found with the discriminating path edge removal step and the PAG found without it. Without that step, there are several orientation mistakes and one adjacency mistake.

\begin{table}[H]
\centering
\begin{tabular}{@{}lll@{}}
\toprule
\textbf{Edge} & \textbf{Correct PAG} & \textbf{Without Discriminating Path Step} \\
\midrule
D \textendash{} H & $D \leftrightarrow H$ & $H \to D$ \\
H \textendash{} L & $H \leftrightarrow L$ & $H \to L$ \\
D \textendash{} L & $D \to L$ & $D$ o--o $L$ \\
L \textendash{} M & $L$ o$\to$ $M$ & $L \to M$ \\
S \textendash{} L & (absent) & $S \to L$ \\
Other edges & identical & identical \\
\bottomrule
\end{tabular}
\caption{Comparison of key edges in the PAG with and without using the discriminating path orientation step in edge removal. Differences arise due to missed collider orientations affecting recursive blocking decisions.}
\label{tab:pag_comparison_summary}
\end{table}

\paragraph{Conclusion.}
This example demonstrates that discriminating path orientation must occur during edge removal for recursive blocking to correctly identify separating sets and return a valid PAG. This requirement is specific to FCIT. In FCI and GFCI, adjacencies $X !\text{–}! Y$ are removed prior to orientation, using subsets of adjacents or Possible-D-SEP sets. FCIT, by contrast, identifies separating sets by tracing paths recursively—an approach that depends on correct collider orientation during edge removal, including colliders identified via discriminating paths. This design enables a more efficient and accurate latent variable search: it leverages richer graphical structure during edge removal and reduces the number of conditional independence tests to those involving subsets of the common adjacents of $X$ and $Y$ after the initial CPDAG search, a much smaller search space.

\paragraph{Take-Home Message.}
Together, these traces illustrate how FCIT integrates recursive blocking with discriminating path orientation during edge removal, yielding separating sets that are both valid and efficient to compute. By using more of the graph’s structure earlier in the process, FCIT achieves correctness guarantees while sharply reducing the search space for conditional independence tests.

%----------------------------------------
\section{Auxiliary Figures}
\label{sec:auxiliary_figures}
%----------------------------------------

%----------------------------------------
\subsection*{Appendix: 20-Node Results for All Average Degrees}
%-------------------------------------------------

For completeness, we provide all figures for the 20-node case with average degrees 2 and 6  
(Figs.~\ref{fig:ap20deg2}–\ref{fig:cpu20deg2} and  
Figs.~\ref{fig:ap20deg6}–\ref{fig:cpu20deg6}, respectively).  
These exhibit trends consistent with the average-degree 4 case discussed in the main text.  
At lower densities (average degree 2) all algorithms recover sparser, more stable structures,  
whereas at higher densities (average degree 6) recall improves but orientation accuracy declines slightly.

% ------------------ avg degree 2 ------------------

\begin{figure}
    \centering
    \includegraphics[width=0.5\linewidth]{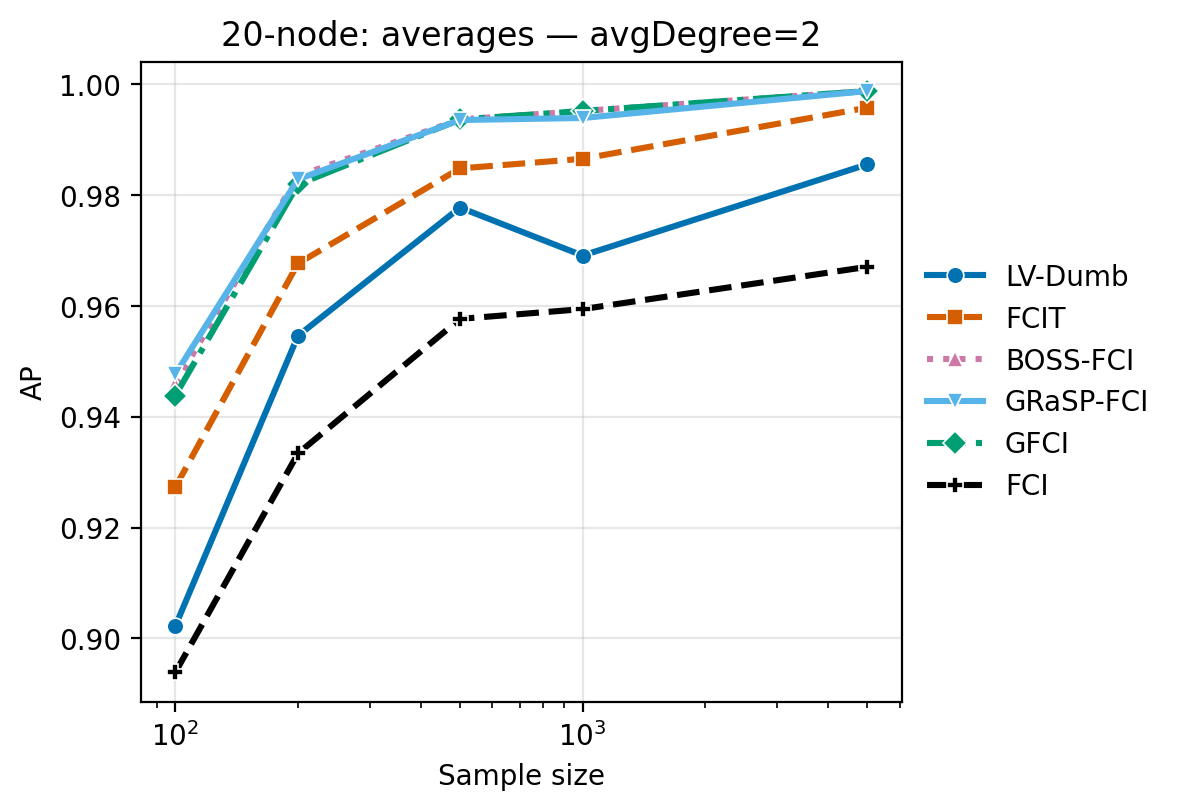}
    \caption{Adjacency Precision (AP) for 20-node graphs with average degree 2. 
    Precision remains uniformly high across all algorithms ($>0.93$). 
    \textsc{BOSS-FCI} \textsc{GRaSP-FCI}, and \textsc{GFCI}  achieve near-perfect accuracy ($\approx 1.0$) throughout, with \textsc{LV-Dumb} lagging behimd a bit. \textsc{FCI} is also slightly behind but still quite high. 
    Results are averaged over graphs with 0, 4, and 8 latent common causes.}
    \label{fig:ap20deg2}
\end{figure}

\begin{figure}
    \centering
    \includegraphics[width=0.5\linewidth]{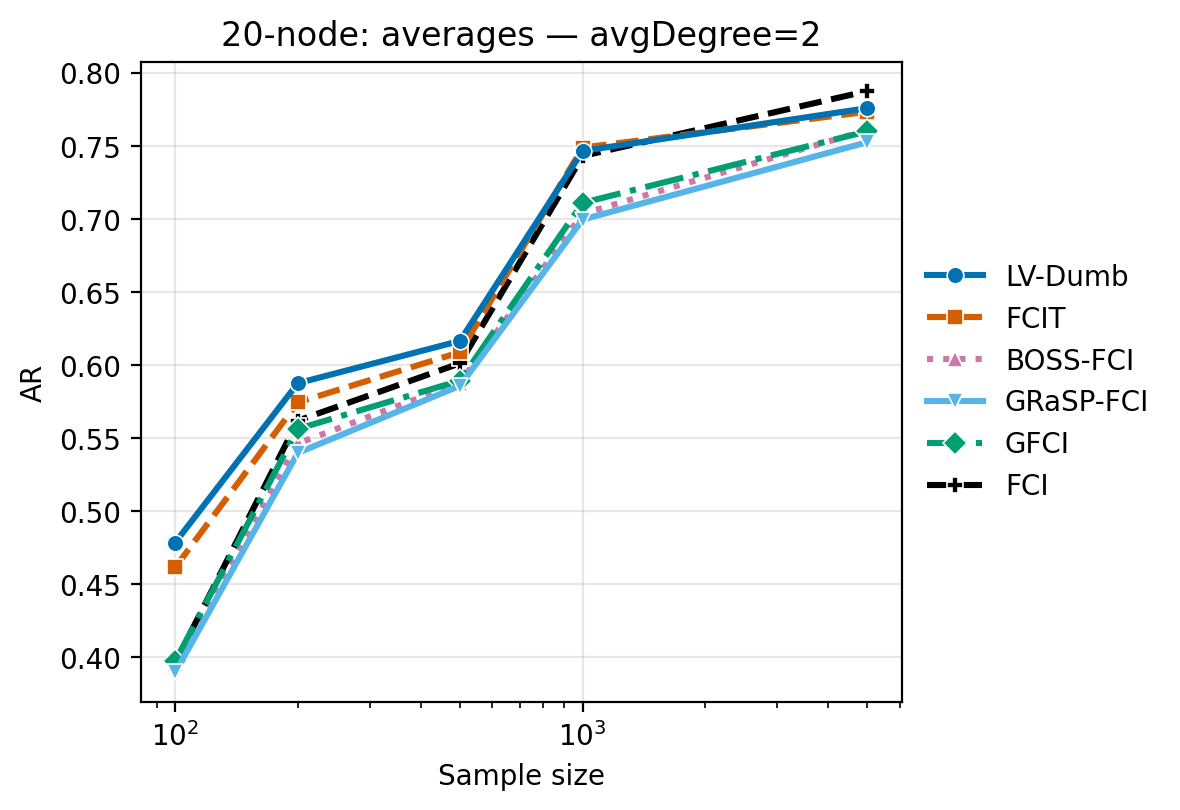}
    \caption{Adjacency Recall (AR) for 20-node graphs with average degree 2. 
    All algorithms improve sharply with sample size, essentially in parallel. 
    Results are averaged over graphs with 0, 4, and 8 latent common causes.}
    \label{fig:ar20deg2}
\end{figure}

\begin{figure}
    \centering
    \includegraphics[width=0.5\linewidth]{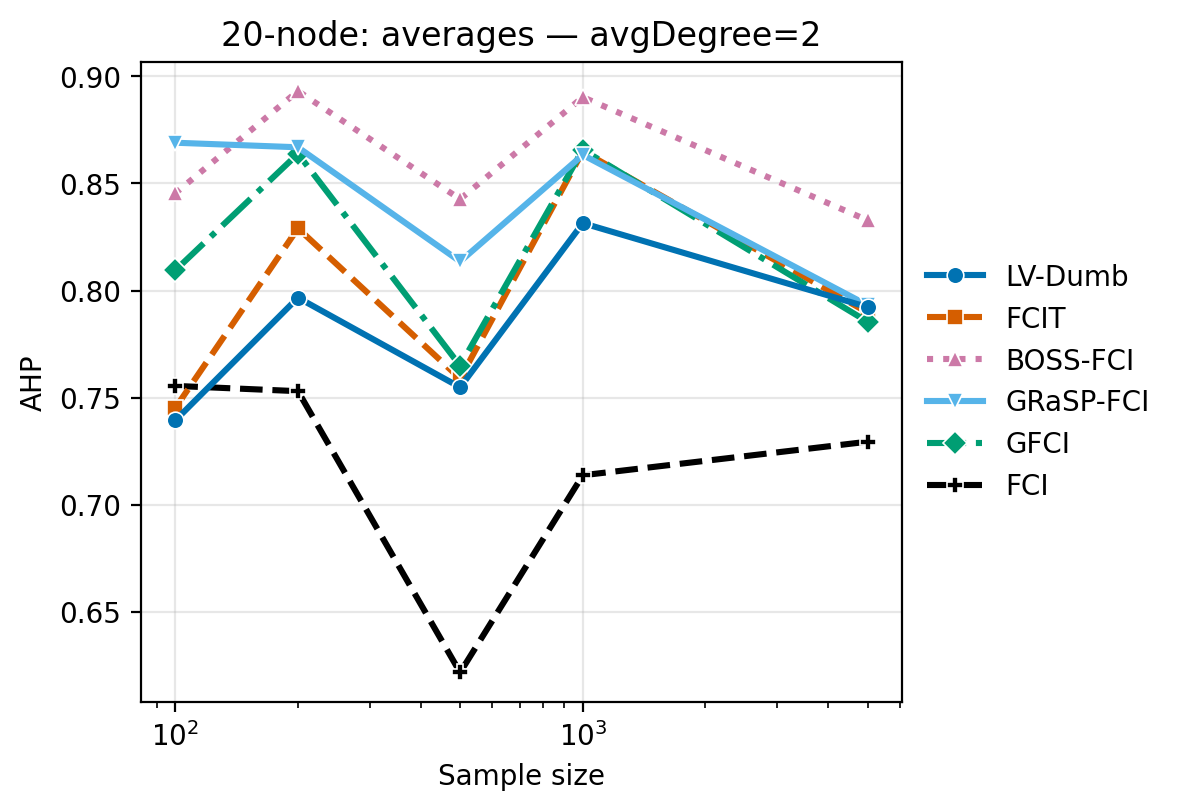}
    \caption{Arrowhead Precision (AHP) for 20-node graphs with average degree 2. 
    \textsc{LV-Dumb} lags at small samples, and \textsc{FCI} is behind overall; the other algorithms do quite well. 
    \textsc{FCI} remains slightly lower but converges steadily. 
    Results are averaged over graphs with 0, 4, and 8 latent common causes.}
    \label{fig:ahp20deg2}
\end{figure}

\begin{figure}
    \centering
    \includegraphics[width=0.5\linewidth]{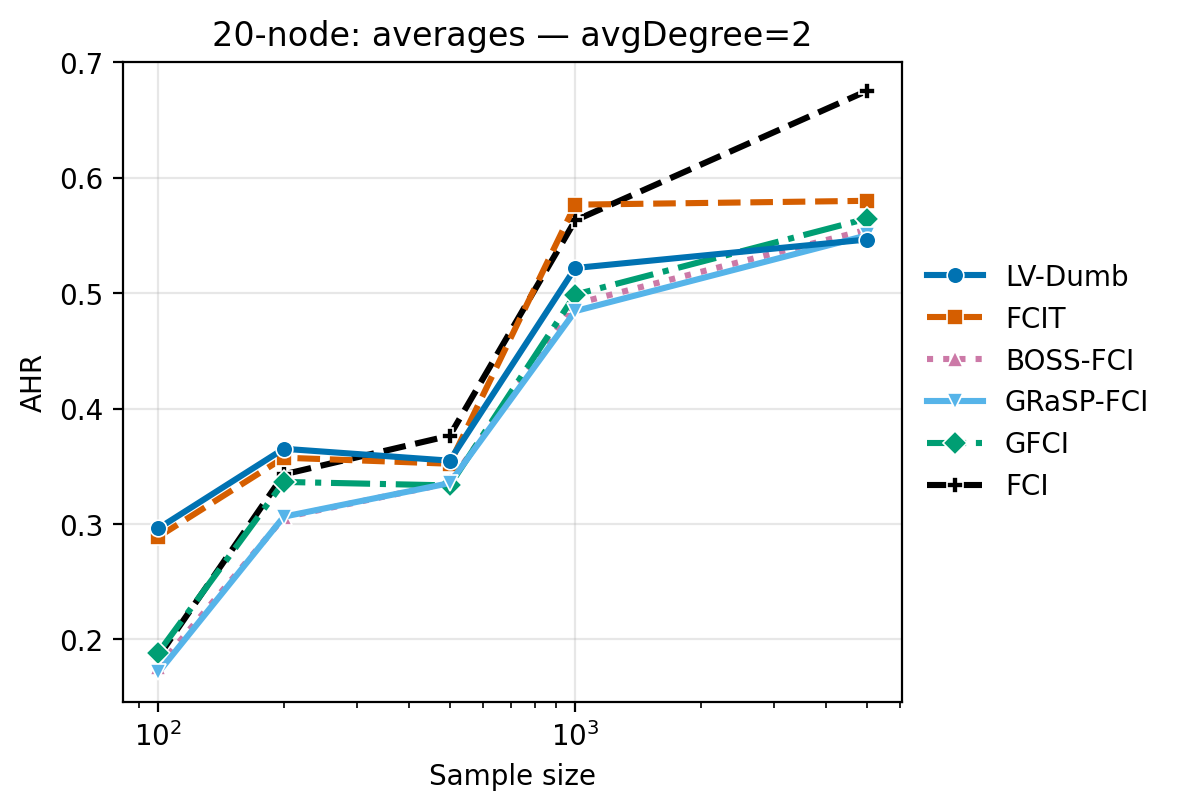}
    \caption{Arrowhead Recall (AHR) for 20-node graphs with average degree 2. 
    Recall improves monotonically with $N$ for all algorithms, exceeding 0.7 at the largest sample size. 
    \textsc{FCI} attains the highest overall recall, followed closely by \textsc{FCIT}, while \textsc{LV-Dumb} and the GFCI-family algorithms show similar convergence. 
    Results are averaged over graphs with 0, 4, and 8 latent common causes.}
    \label{fig:ahr20deg2}
\end{figure}

\begin{figure}
    \centering
    \includegraphics[width=0.5\linewidth]{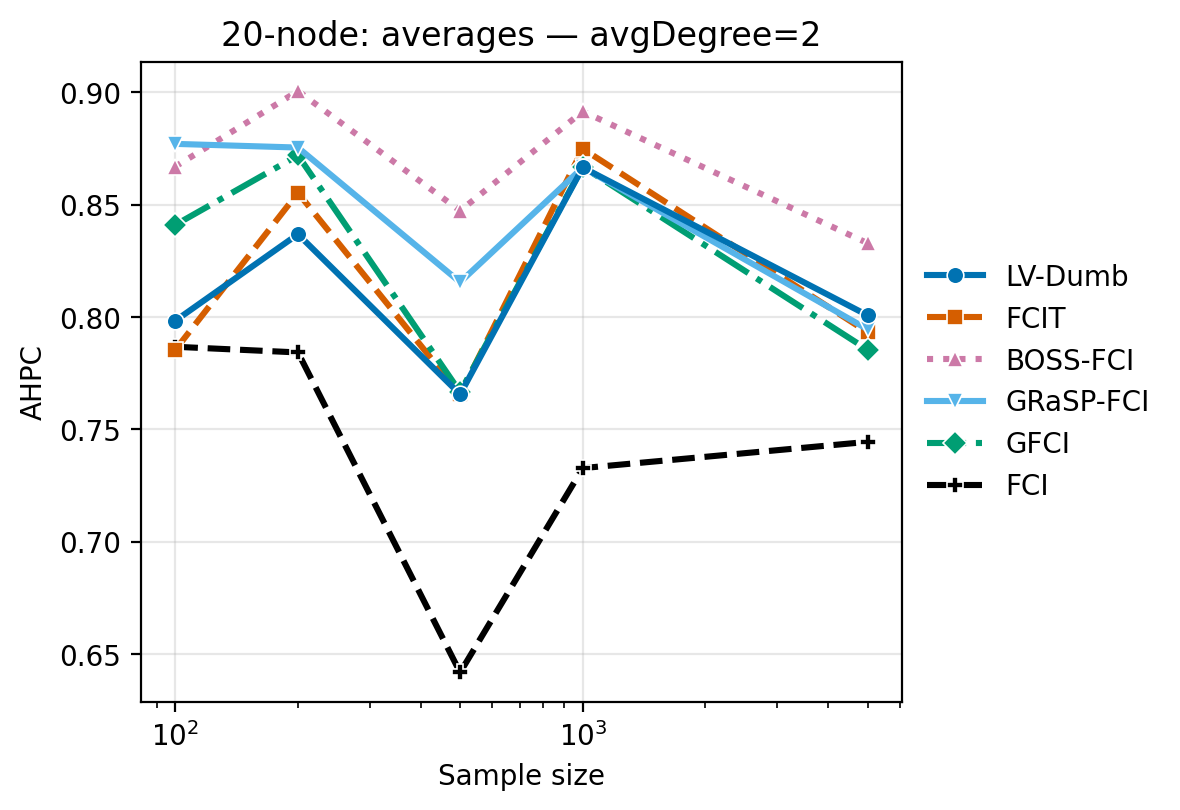}
    \caption{Arrowhead Precision for Common Adjacencies (AHPC) for 20-node graphs with average degree 2. 
    Most algorithms fare quite well here, with \textsc{FCI} lagging behind. 
    Results are averaged over graphs with 0, 4, and 8 latent common causes.}
    \label{fig:ahpc20deg2}
\end{figure}

\begin{figure}
    \centering
    \includegraphics[width=0.5\linewidth]{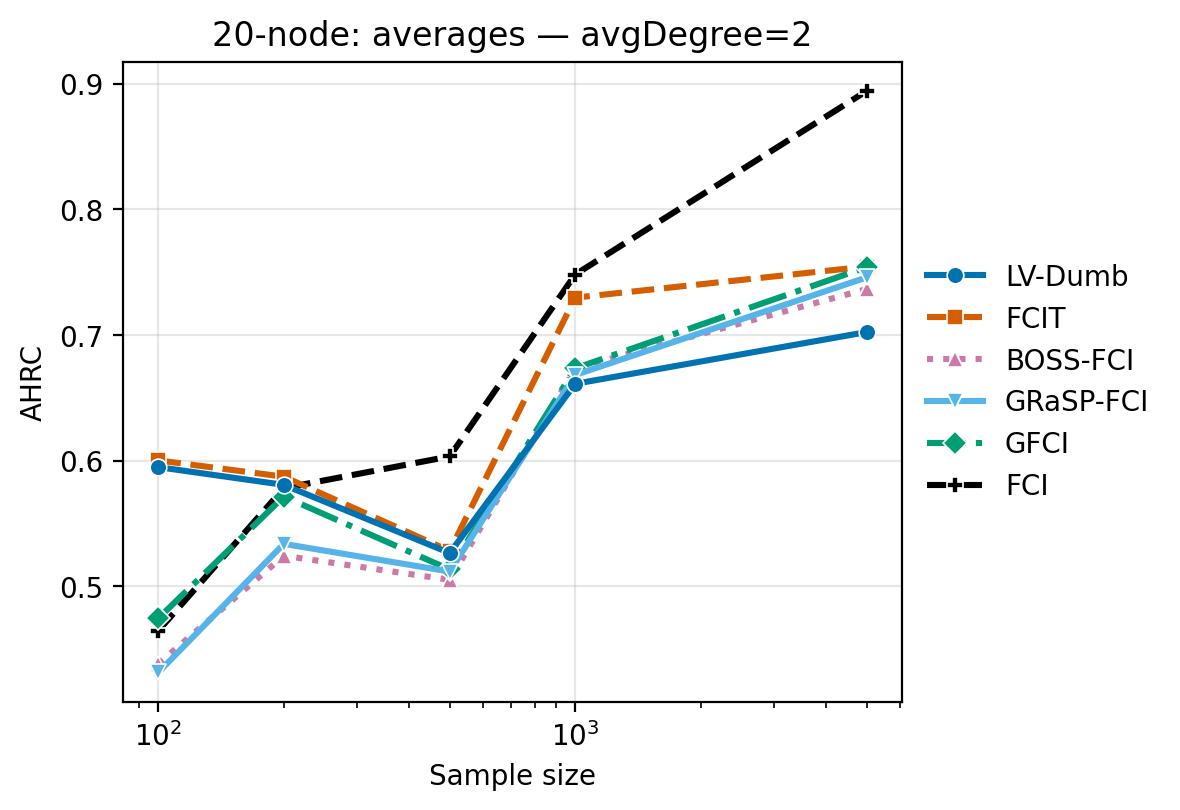}
    \caption{Arrowhead Recall for Common Adjacencies (AHRC) for 20-node graphs with average degree 2. 
    \textsc{FCI} achieves the highest arrowhead recall ($>0.9$) at large $N$, with other algorithms somewhat lower. Out of the algorithms with good AHPC, \textsc{FCIT} performs the best.
    Results are averaged over graphs with 0, 4, and 8 latent common causes.}
    \label{fig:ahrc20deg2}
\end{figure}

% --
\begin{figure}
    \centering
    \includegraphics[width=0.5\linewidth]{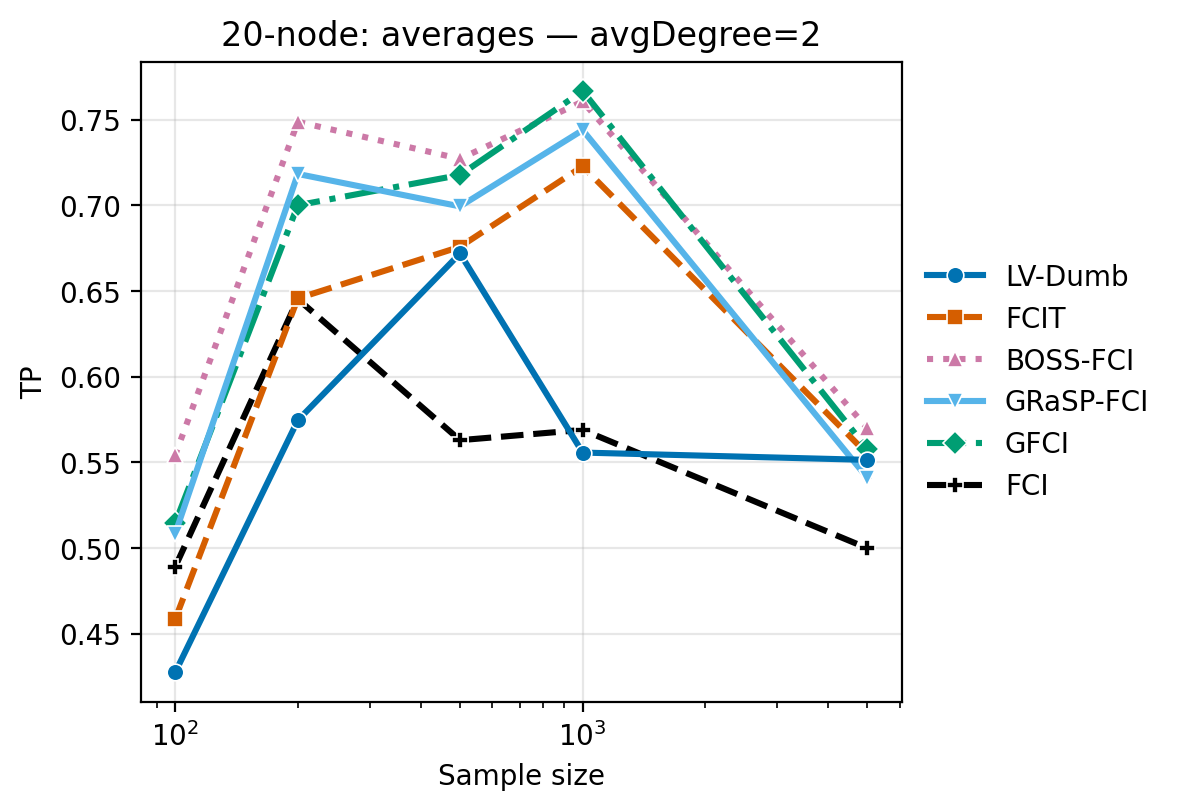}
    \caption{Tail Precision (TP) for 20-node graphs with average degree 2. 
    Most algorithms are in a top group, with \textsc{LV-Dumb} and \textsc{FCI} lagging behind.
    \textsc{FCI} trails slightly at larger sample sizes. 
    Results are averaged over graphs with 0, 4, and 8 latent common causes.}
    \label{fig:tp20deg2}
\end{figure}

\begin{figure}
    \centering
    \includegraphics[width=0.5\linewidth]{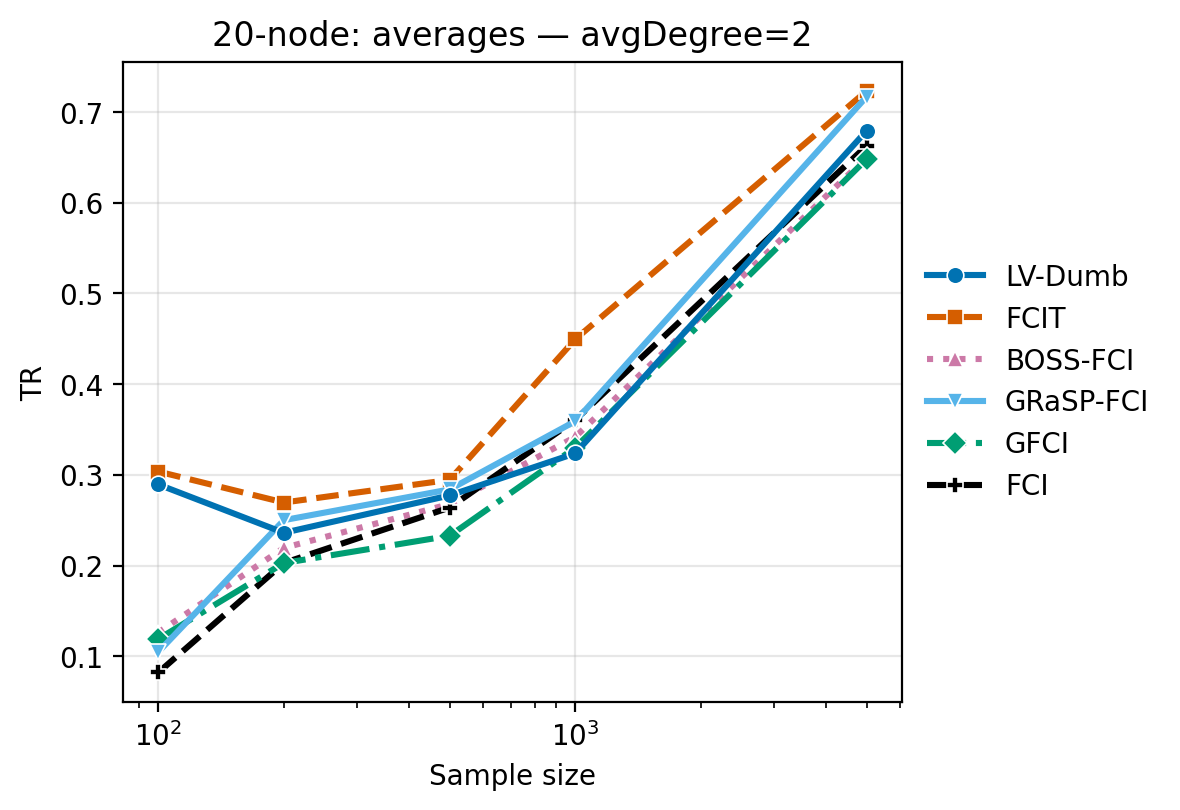}
    \caption{Tail Recall (TR) for 20-node graphs with average degree 2. 
    The algorithm all consistently improve with sample size, with \textsc{FCIT} having the largest tail recalls uniformly.
    Results are averaged over graphs with 0, 4, and 8 latent common causes.}
    \label{fig:tr20deg2}
\end{figure}
% --

\begin{figure}
    \centering
    \includegraphics[width=0.5\linewidth]{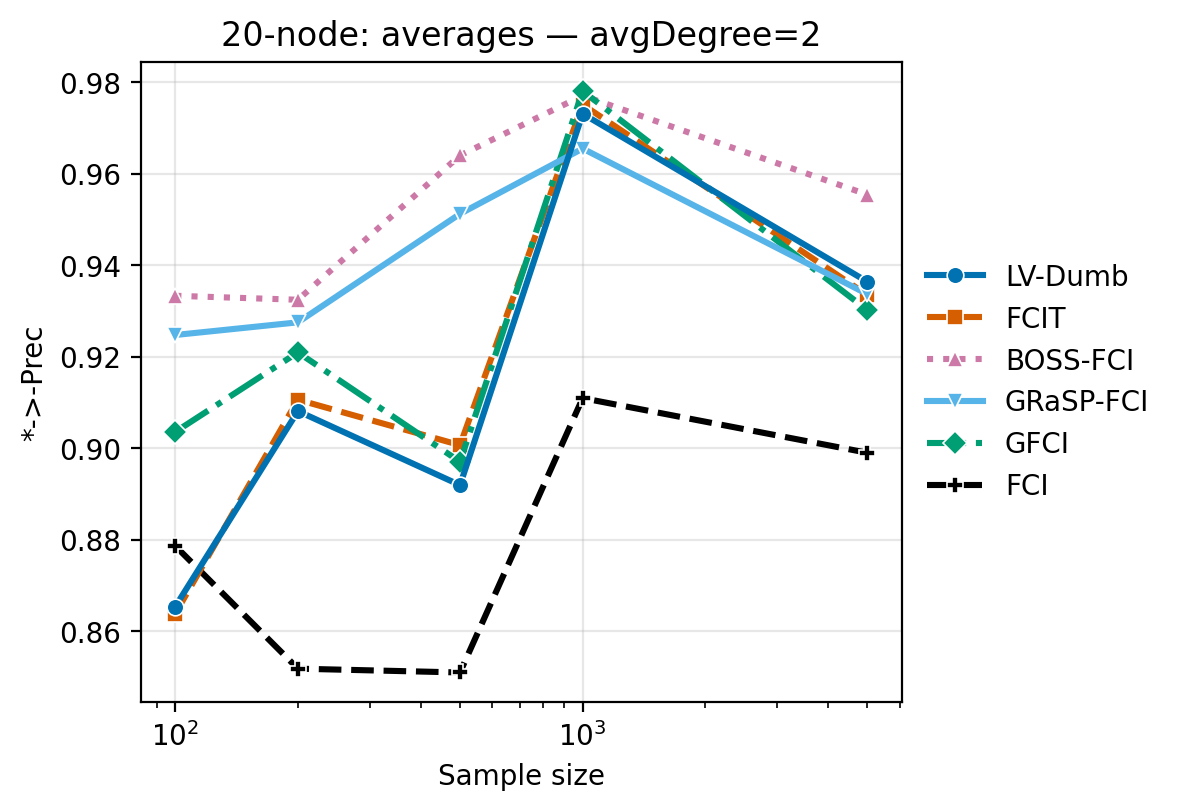}
    \caption{Arrow Path Precision for 20-node graphs with average degree 2. 
    This precision is excellent here for all algorithms except \textsc{FCI} as $N$ increases. 
    Results are averaged over graphs with 0, 4, and 8 latent common causes.}
    \label{fig:arrow20deg2}
\end{figure}

\begin{figure}
    \centering
    \includegraphics[width=0.5\linewidth]{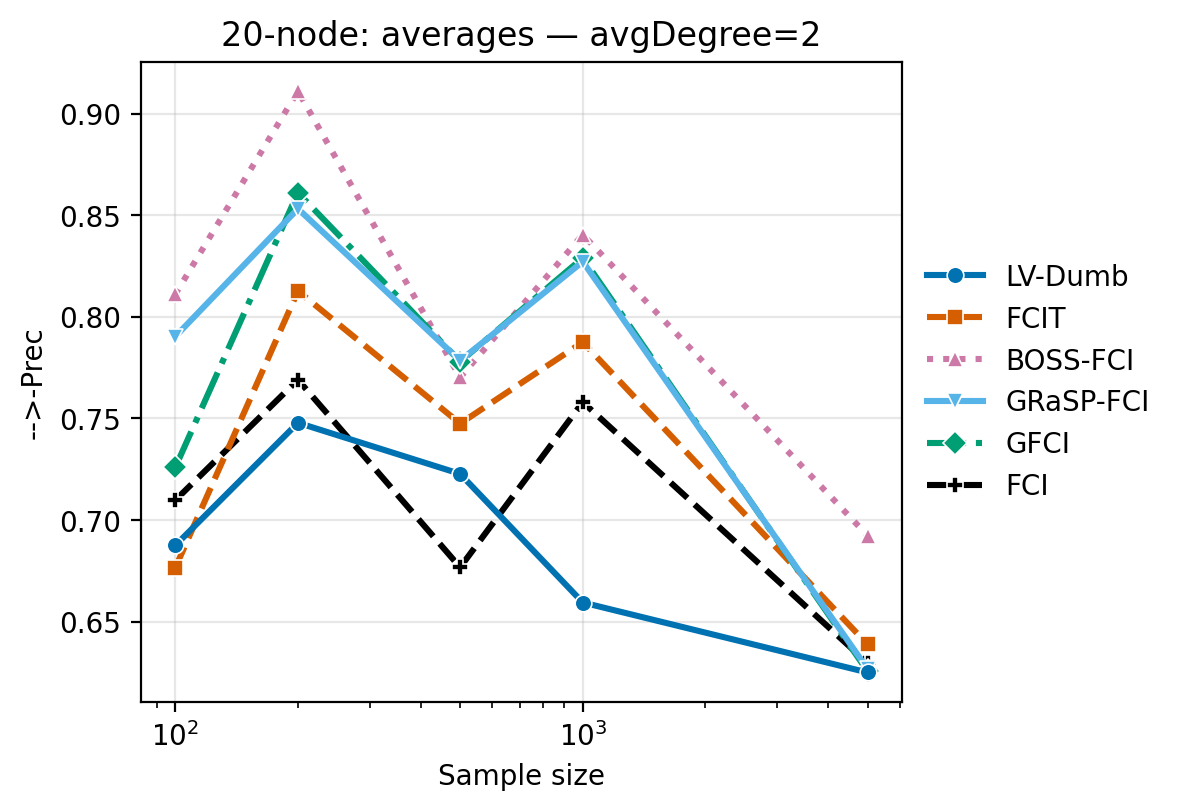}
    \caption{Tail Path Precision for 20-node graphs with average degree 2. 
    Precision varies strongly with $N$; for large samples. \textsc{LV-Dumb} is systematically behind.  
    Results are averaged over graphs with 0, 4, and 8 latent common causes.}
    \label{fig:tail20deg2}
\end{figure}

\begin{figure}
    \centering
    \includegraphics[width=0.5\linewidth]{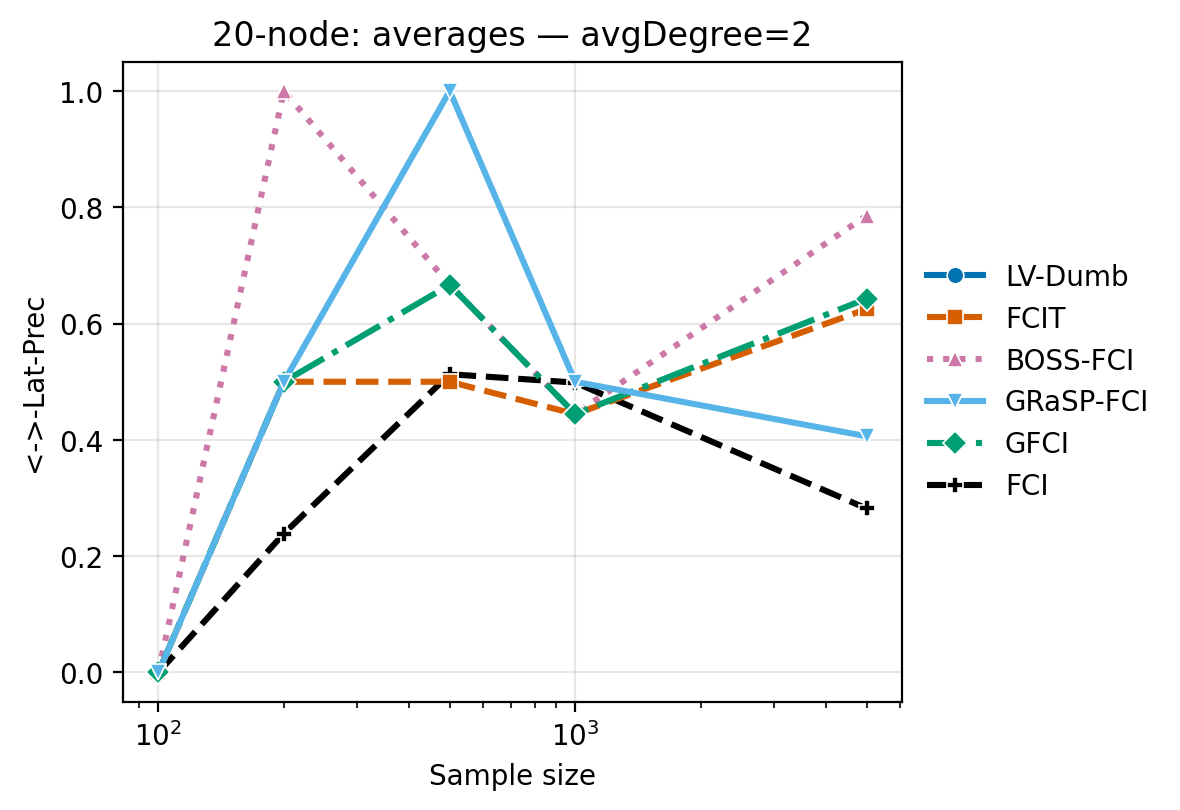}
    \caption{Bidirected Latent Path Precision for 20-node graphs with average degree 2. 
    There is considerable variation here, though \textsc{FCI} is lower overall.
    Results are averaged over graphs with 0, 4, and 8 latent common causes.}
    \label{fig:lat20deg2}
\end{figure}

\begin{figure}
    \centering
    \includegraphics[width=0.5\linewidth]{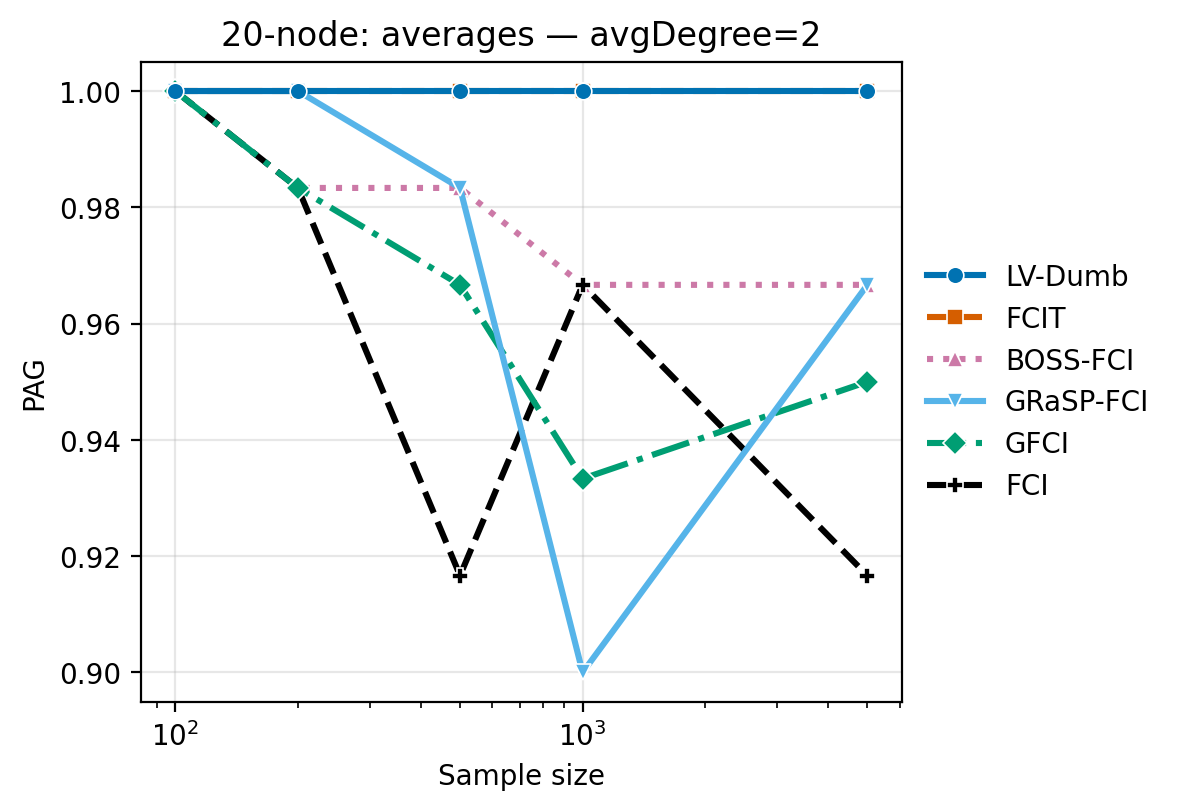}
    \caption{Proportion of well-formed PAGs for 20-node graphs with average degree 2. 
    There is good PAG precision for this regime, with \textsc{LV-Dumb} and \textsc{FCIT} returning perfect scores, as predicted.
    Results are averaged over graphs with 0, 4, and 8 latent common causes.}
    \label{fig:pag20deg2}
\end{figure}

\begin{figure}
    \centering
    \includegraphics[width=0.5\linewidth]{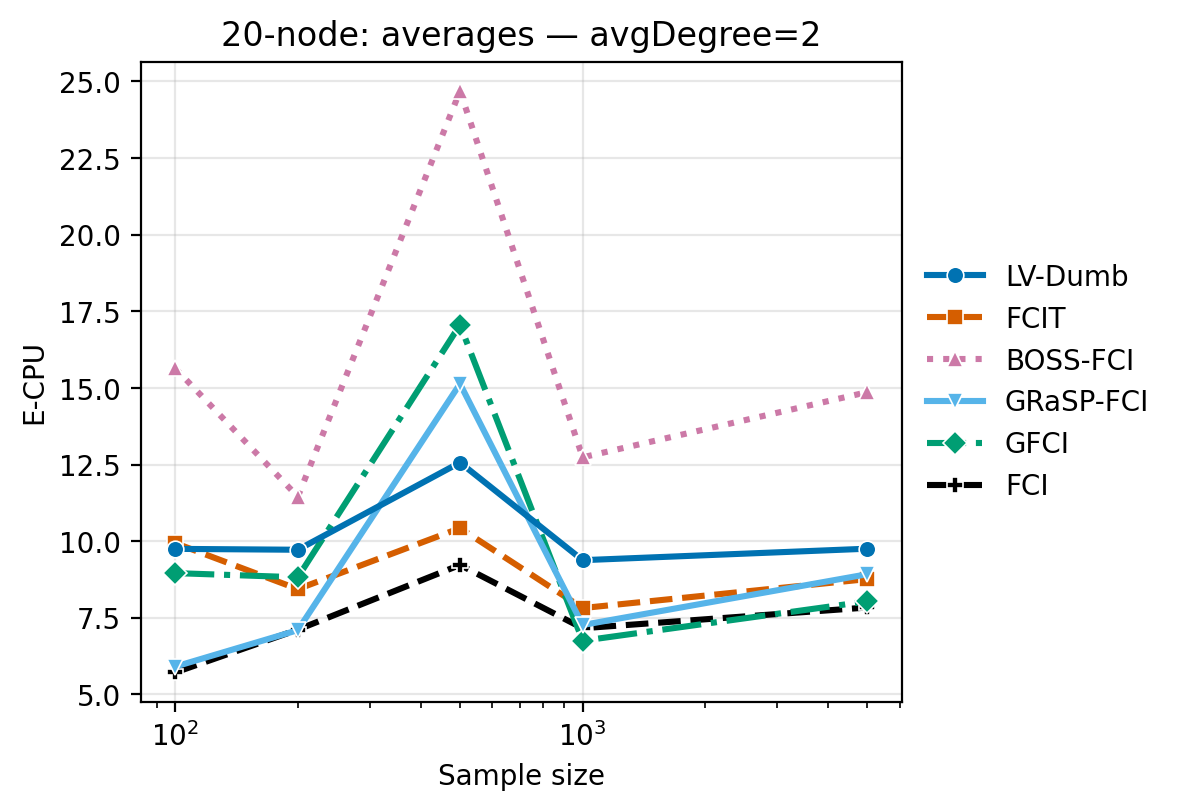}
    \caption{Runtime (CPU milliseconds) for 20-node graphs with average degree 2. 
    All algorithms run quickly in this regime. 
    Results are averaged over graphs with 0, 4, and 8 latent common causes.}
    \label{fig:cpu20deg2}
\end{figure}

% ------------------ avg degree 6 ------------------

\begin{figure}
    \centering
    \includegraphics[width=0.5\linewidth]{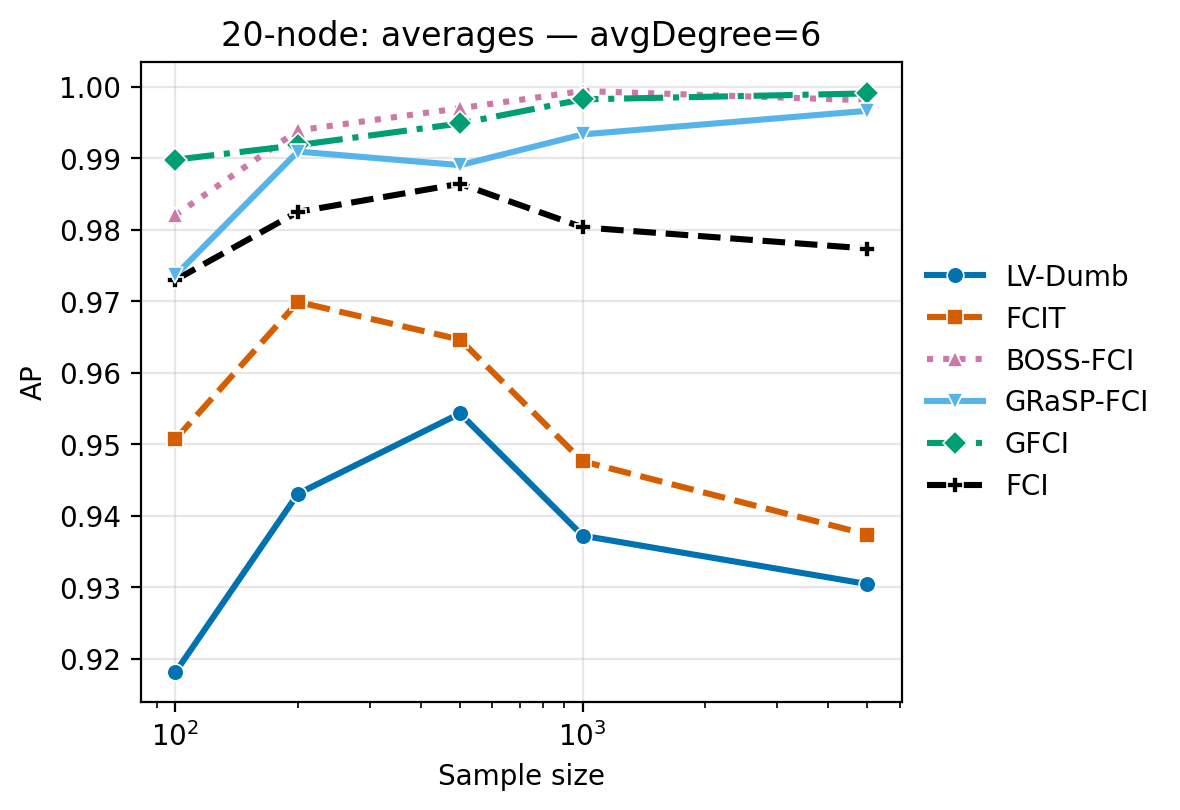}
    \caption{Adjacency Precision (AP) for 20-node graphs with average degree 6. 
    All algorithms maintain high precision ($>0.93$) across sample sizes.
    Results are averaged over graphs with 0, 4, and 8 latent common causes.}
    \label{fig:ap20deg6}
\end{figure}

\begin{figure}
    \centering
    \includegraphics[width=0.5\linewidth]{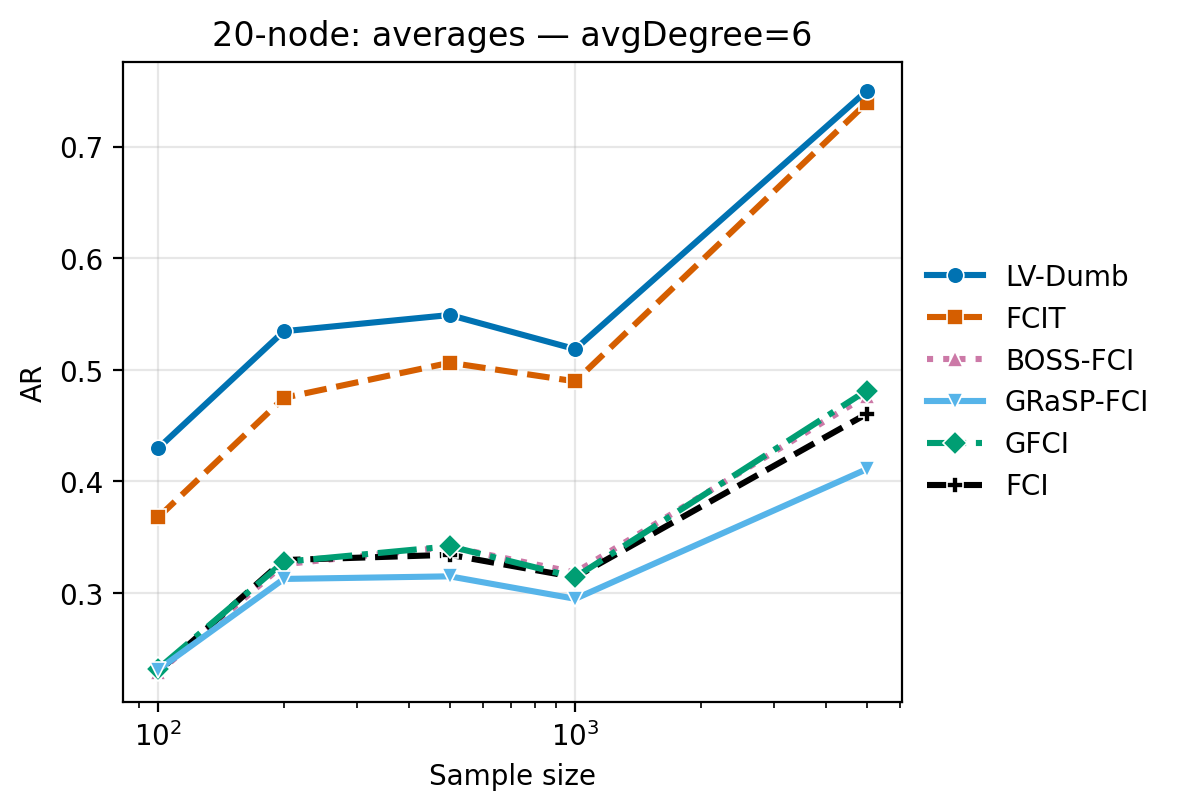}
    \caption{Adjacency Recall (AR) for 20-node graphs with average degree 6. 
    \textsc{LV-Dumb} and \textsc{FCIT} show the strongest gains with sample size, reaching $\approx 0.7$ at large $N$. 
    \textsc{GFCI}, \textsc{BOSS-FCI}, and \textsc{FCI} improve more modestly, and \textsc{GRaSP-FCI} remains lowest. 
    Results are averaged over graphs with 0, 4, and 8 latent common causes.}
    \label{fig:ar20deg6}
\end{figure}

\begin{figure}
    \centering
    \includegraphics[width=0.5\linewidth]{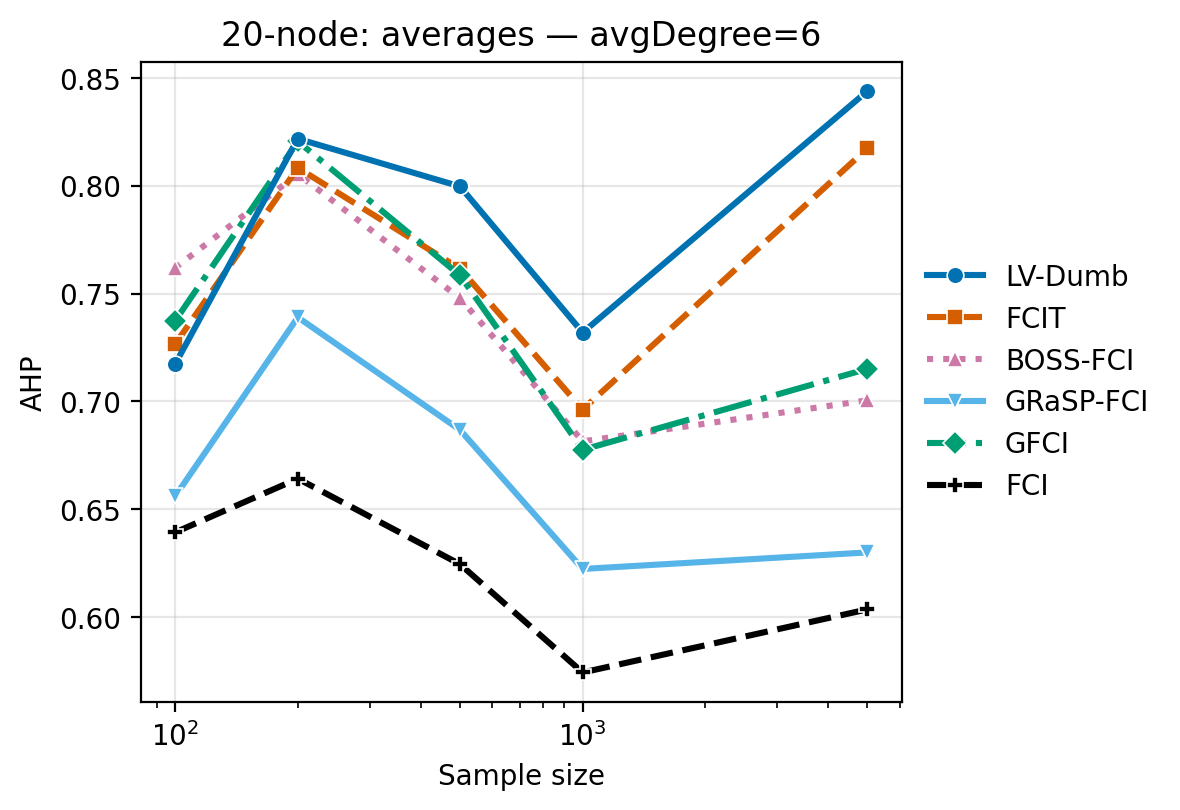}
    \caption{Arrowhead Precision (AHP) for 20-node graphs with average degree 6. 
    \textsc{LV-Dumb} attains the highest precision and improves with $N$; \textsc{FCIT} is next best. 
    \textsc{GFCI} and \textsc{BOSS-FCI} remain around $0.7$, \textsc{GRaSP-FCI} is lower, and \textsc{FCI} lowest overall. 
    There is a consistent dip at $N{=}500$. 
    Results are averaged over graphs with 0, 4, and 8 latent common causes.}
    \label{fig:ahp20deg6}
\end{figure}

\begin{figure}
    \centering
    \includegraphics[width=0.5\linewidth]{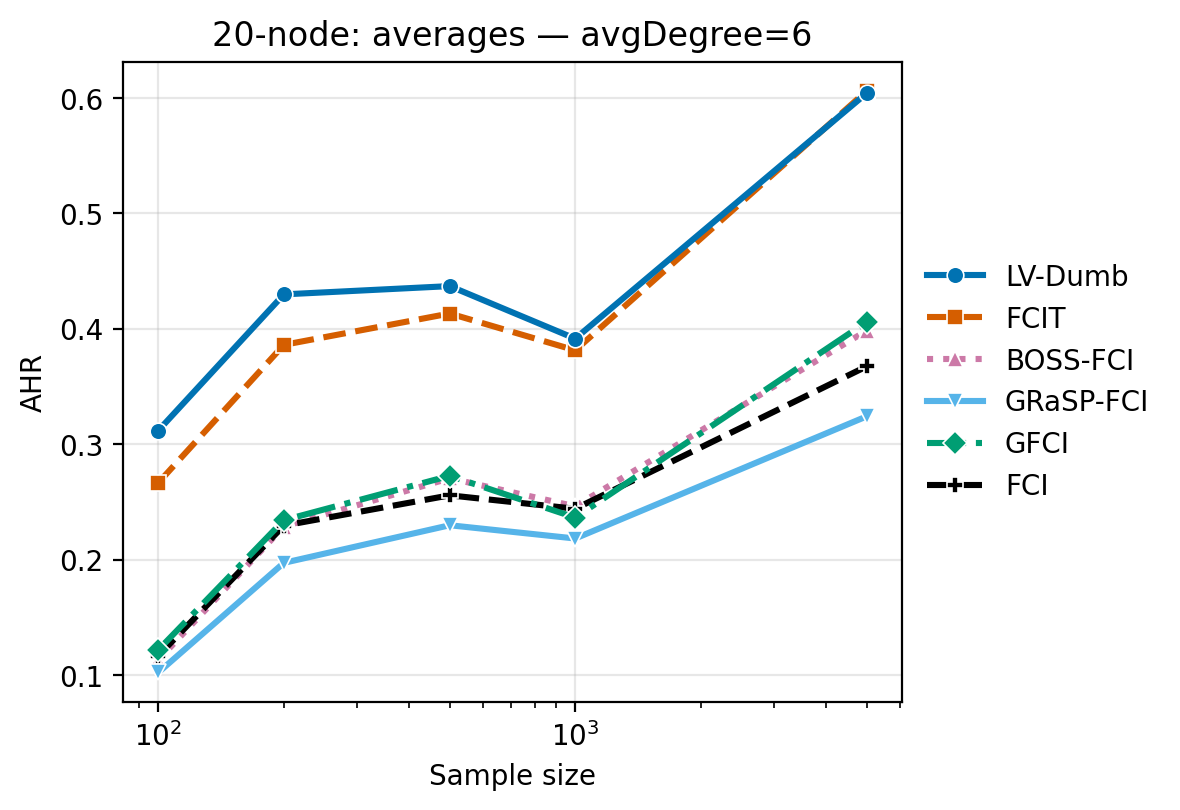}
    \caption{Arrowhead Recall (AHR) for 20-node graphs with average degree 6. 
    All curves rise with $N$; \textsc{LV-Dumb} and \textsc{FCIT} are highest at large $N$ ($\approx 0.6$), while the GFCI-family methods and \textsc{FCI} trail. 
    Results are averaged over graphs with 0, 4, and 8 latent common causes.}
    \label{fig:ahr20deg6}
\end{figure}

\begin{figure}
    \centering
    \includegraphics[width=0.5\linewidth]{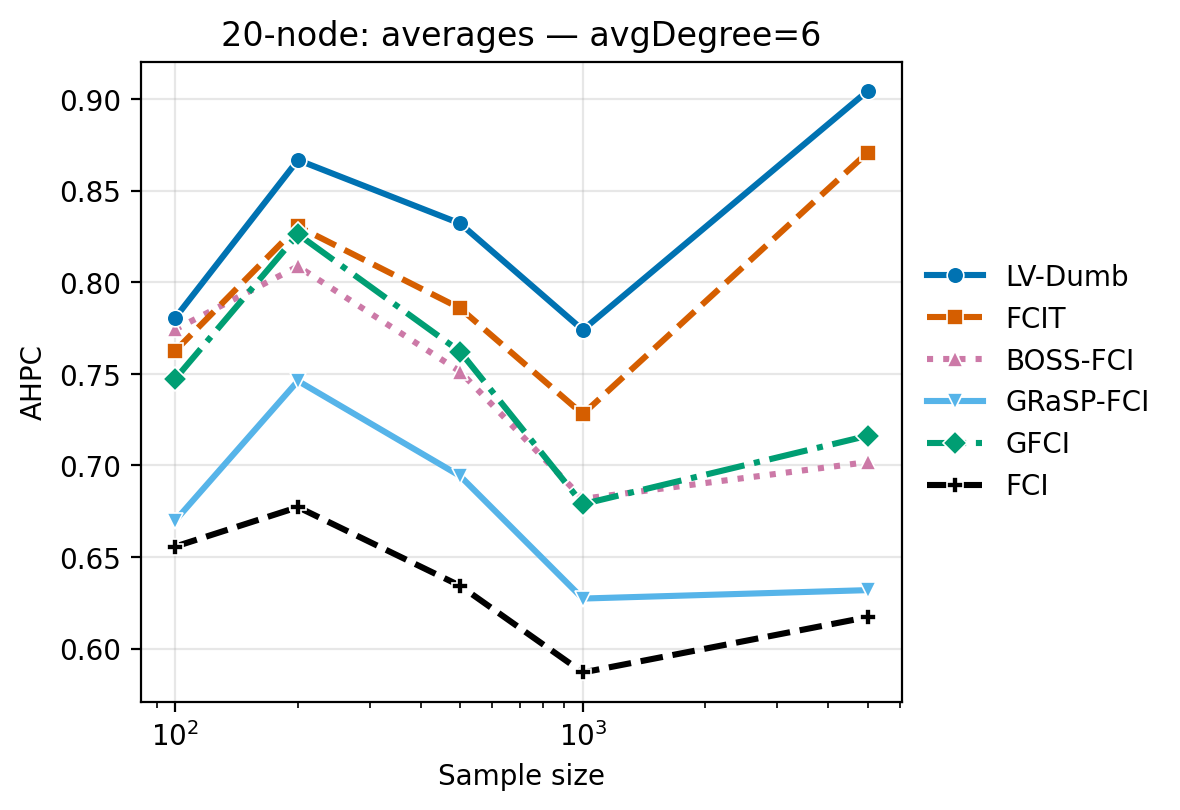}
    \caption{Arrowhead Precision for Common Adjacencies (AHPC) for 20-node graphs with average degree 6. 
    \textsc{LV-Dumb} maintains the highest precision and improves with $N$; \textsc{FCIT} also increases steadily. 
    \textsc{BOSS-FCI} and \textsc{GFCI} remain around $0.72$, \textsc{GRaSP-FCI} is lower, and \textsc{FCI} lowest overall. 
    There is a consistent dip at $N{=}500$. 
    Results are averaged over graphs with 0, 4, and 8 latent common causes.}
    \label{fig:ahpc20deg6}
\end{figure}

\begin{figure}
    \centering
    \includegraphics[width=0.5\linewidth]{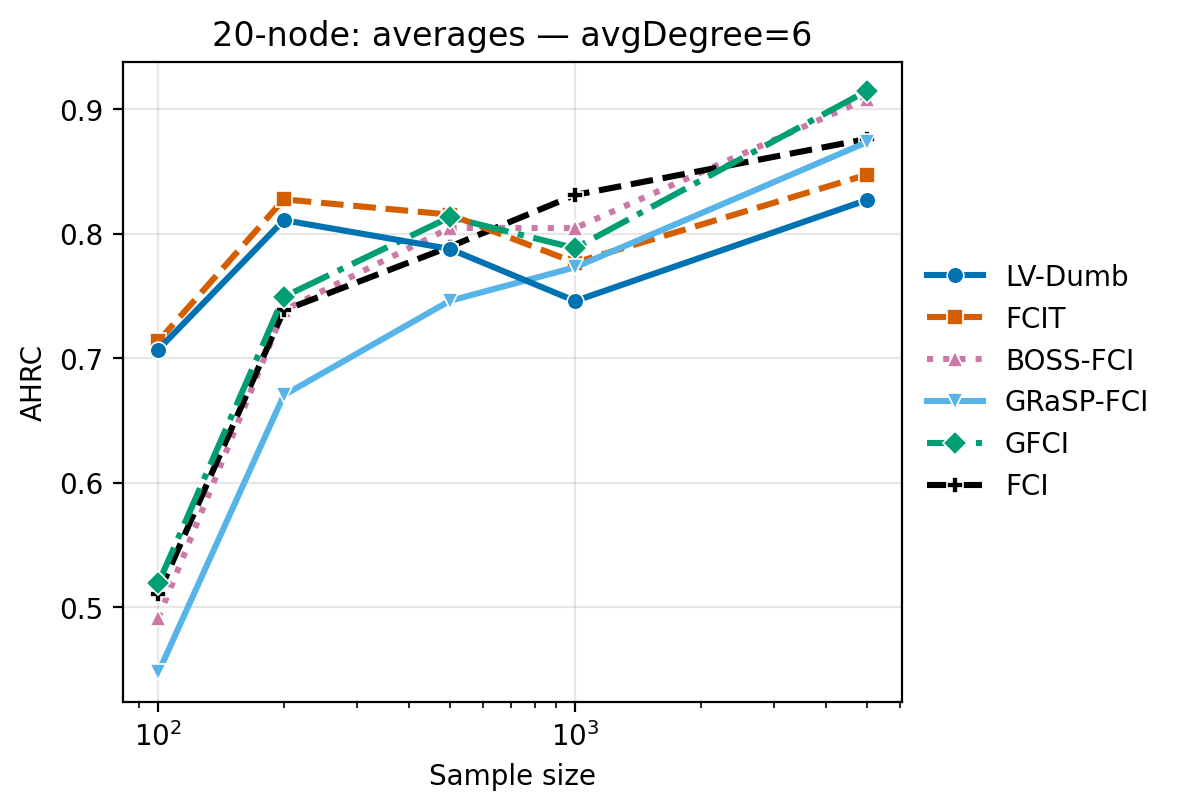}
    \caption{Arrowhead Recall for Common Adjacencies (AHRC) for 20-node graphs with average degree 6. 
    Recall increases with $N$ for all methods. 
    At large $N$, \textsc{BOSS-FCI} is highest (near $0.93$), followed by \textsc{GFCI} and \textsc{FCI}; \textsc{FCIT} and \textsc{GRaSP-FCI} are close behind, and \textsc{LV-Dumb} is slightly lower. 
    Results are averaged over graphs with 0, 4, and 8 latent common causes.}
    \label{fig:ahrc20deg6}
\end{figure}

% --
\begin{figure}
    \centering
    \includegraphics[width=0.5\linewidth]{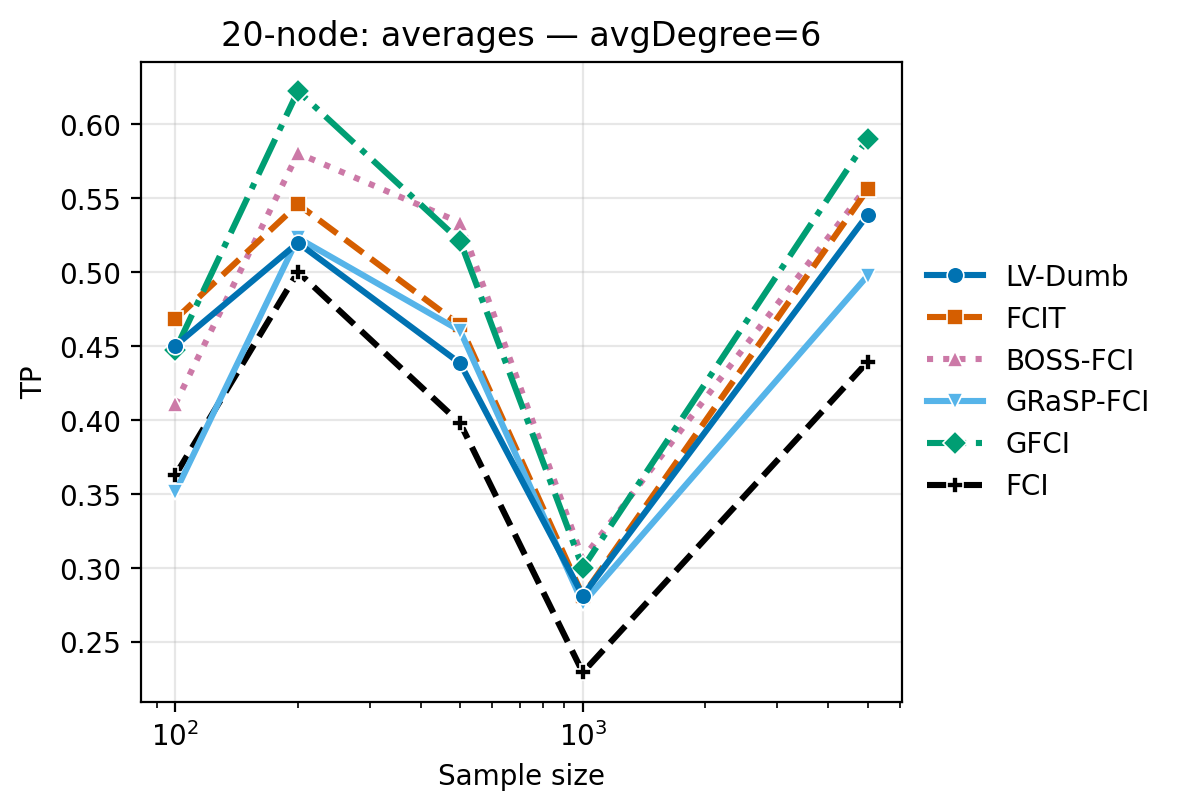}
    \caption{Tail Precision (TP) for 20-node graphs with average degree 6. 
    \textsc{GRaSP-FCI} and \textsc{BOSS-FCI}, and \textsc{FCIT} outperform the other algorithms here, though none of the algorithms performs especially well. 
    Results are averaged over graphs with 0, 4, and 8 latent common causes.}
    \label{fig:tp20deg6}
\end{figure}

\begin{figure}
    \centering
    \includegraphics[width=0.5\linewidth]{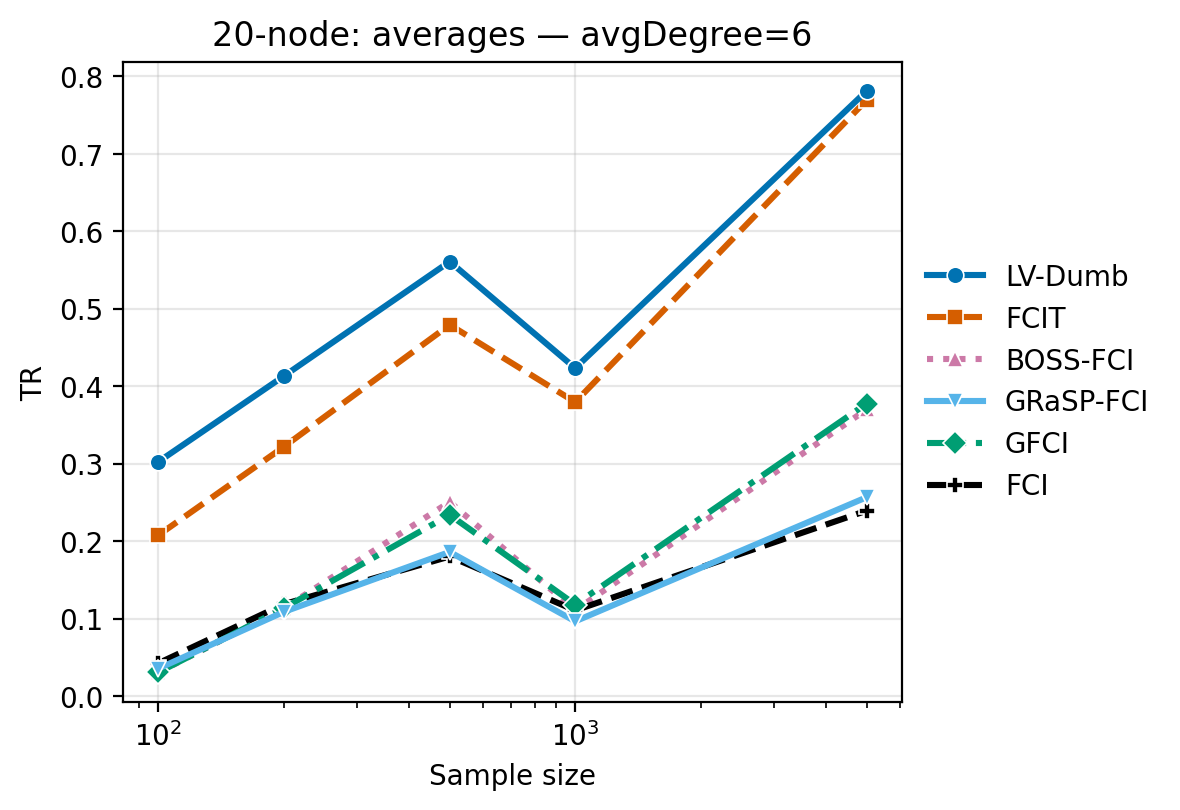}
    \caption{Tail Recall (TR) for 20-node graphs with average degree 6. 
    At the all sample sizes \textsc{LV-Dumb} and \textsc{FCIT} have the highest scores, while the others falls short. 
    Results are averaged over graphs with 0, 4, and 8 latent common causes.}
    \label{fig:tr20deg6}
\end{figure}
% --

\begin{figure}
    \centering
    \includegraphics[width=0.5\linewidth]{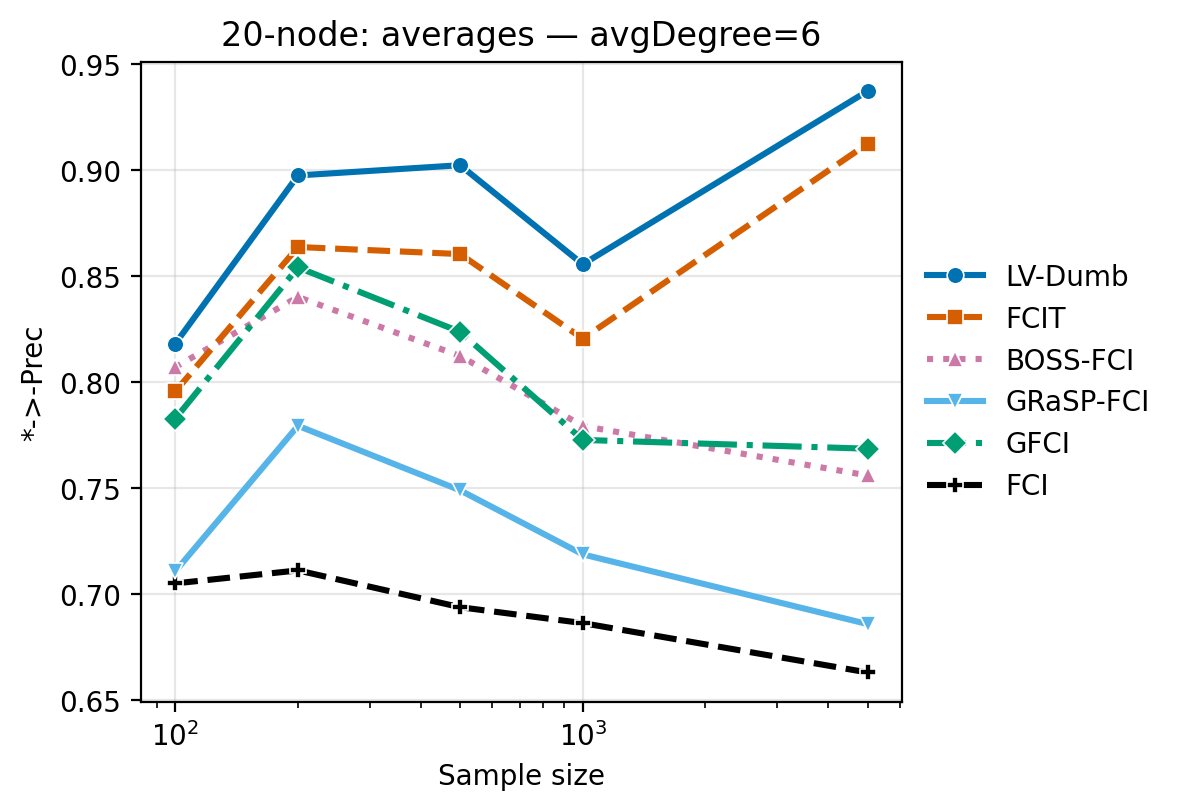}
    \caption{Arrow Path Precision for 20-node graphs with average degree 6. 
    \textsc{LV-Dumb} and \textsc{FCIT} achieve the highest precision, increasing to $\approx 0.96$ at large $N$. 
    \textsc{BOSS-FCI} and \textsc{GFCI} cluster around $0.80$; \textsc{GRaSP-FCI} is lower and \textsc{FCI} lowest overall. 
    Results are averaged over graphs with 0, 4, and 8 latent common causes.}
    \label{fig:arrow20deg6}
\end{figure}

\begin{figure}
    \centering
    \includegraphics[width=0.5\linewidth]{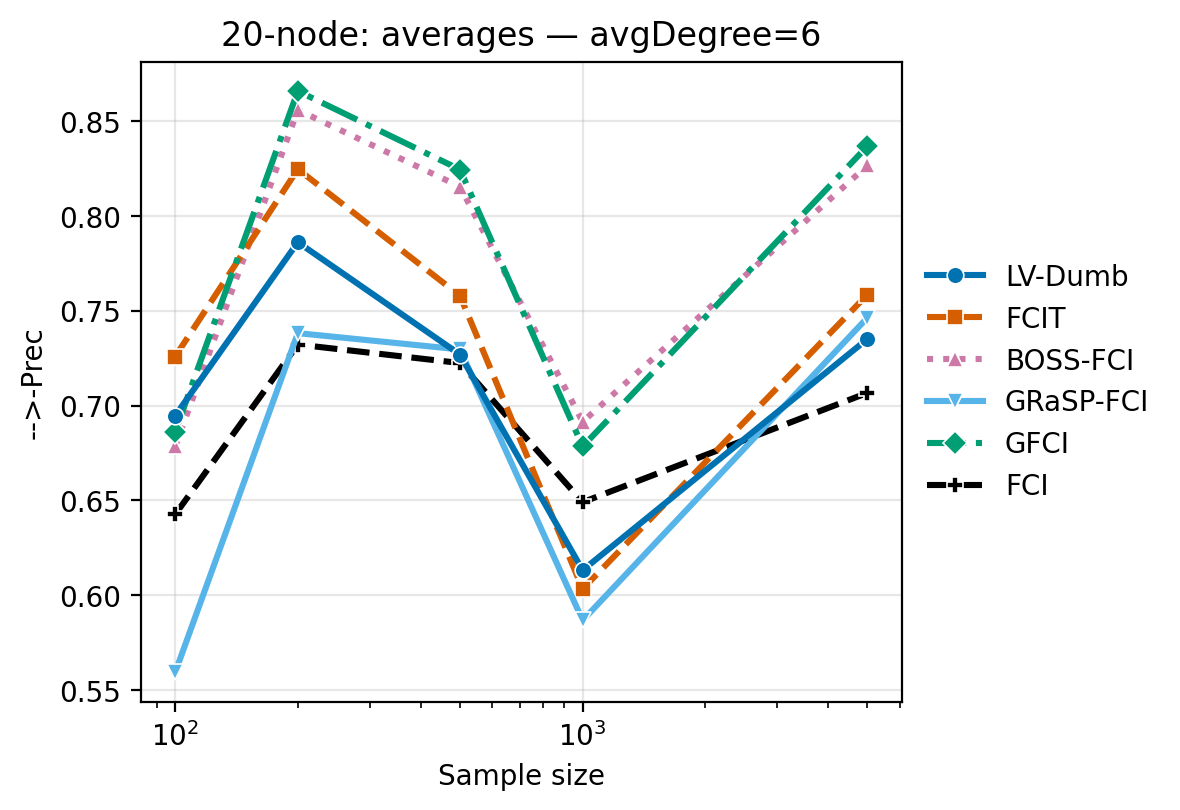}
    \caption{Tail Path Precision for 20-node graphs with average degree 6. 
    There is considerable variation for all methods across sample sizes. \textsc{GRaSP-FCI} and \textsc{BOSS-FCI} outperform the others consistently.
    Results are averaged over graphs with 0, 4, and 8 latent common causes.}
    \label{fig:tail20deg6}
\end{figure}

\begin{figure}
    \centering
    \includegraphics[width=0.5\linewidth]{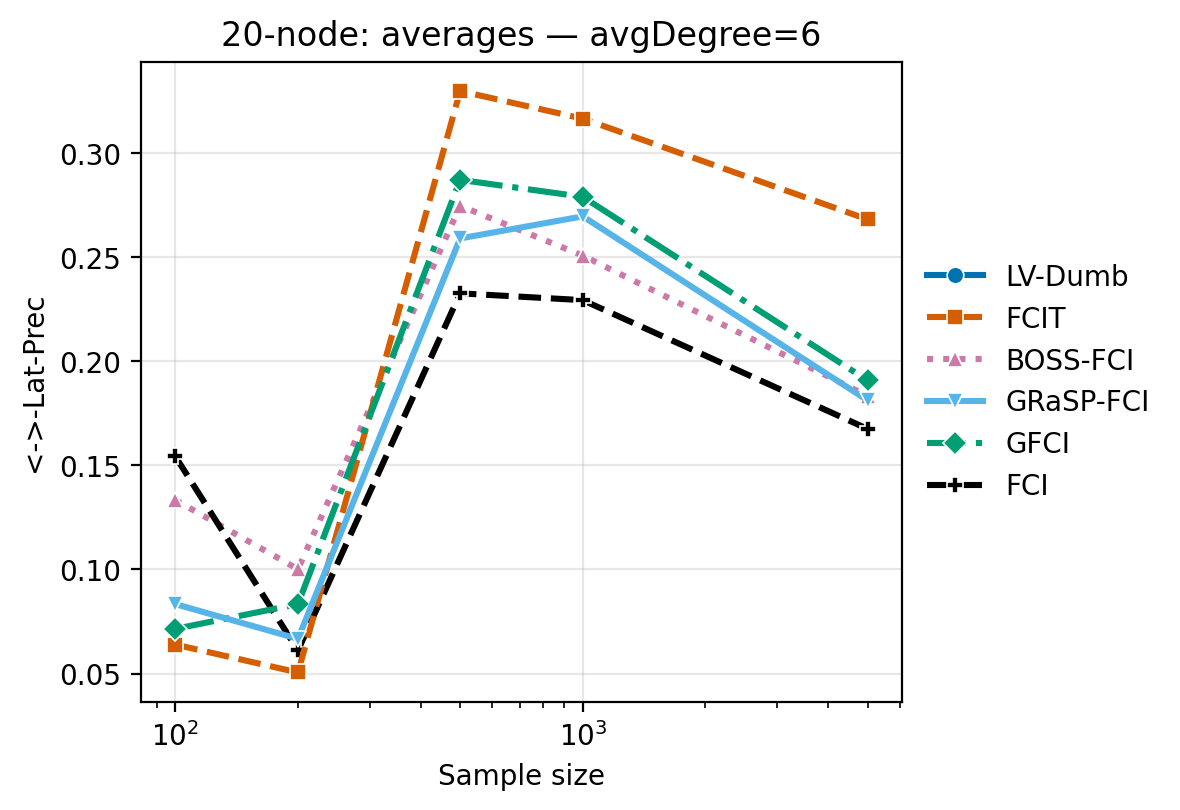}
    \caption{Bidirected Latent Path Precision for 20-node graphs with average degree 6. 
    There is weak performance on this statistic here.
    Precision peaks near intermediate $N$ and eases slightly at the largest $N$. 
    \textsc{FCIT} attains the highest precision at mid-range, with other algorithms trailing. 
    \textsc{LV-Dumb} orients no bidirected edges, yielding undefined precision, so it is not plotted. 
    Results are averaged over graphs with 0, 4, and 8 latent common causes.}
    \label{fig:lat20deg6}
\end{figure}

\begin{figure}
    \centering
    \includegraphics[width=0.5\linewidth]{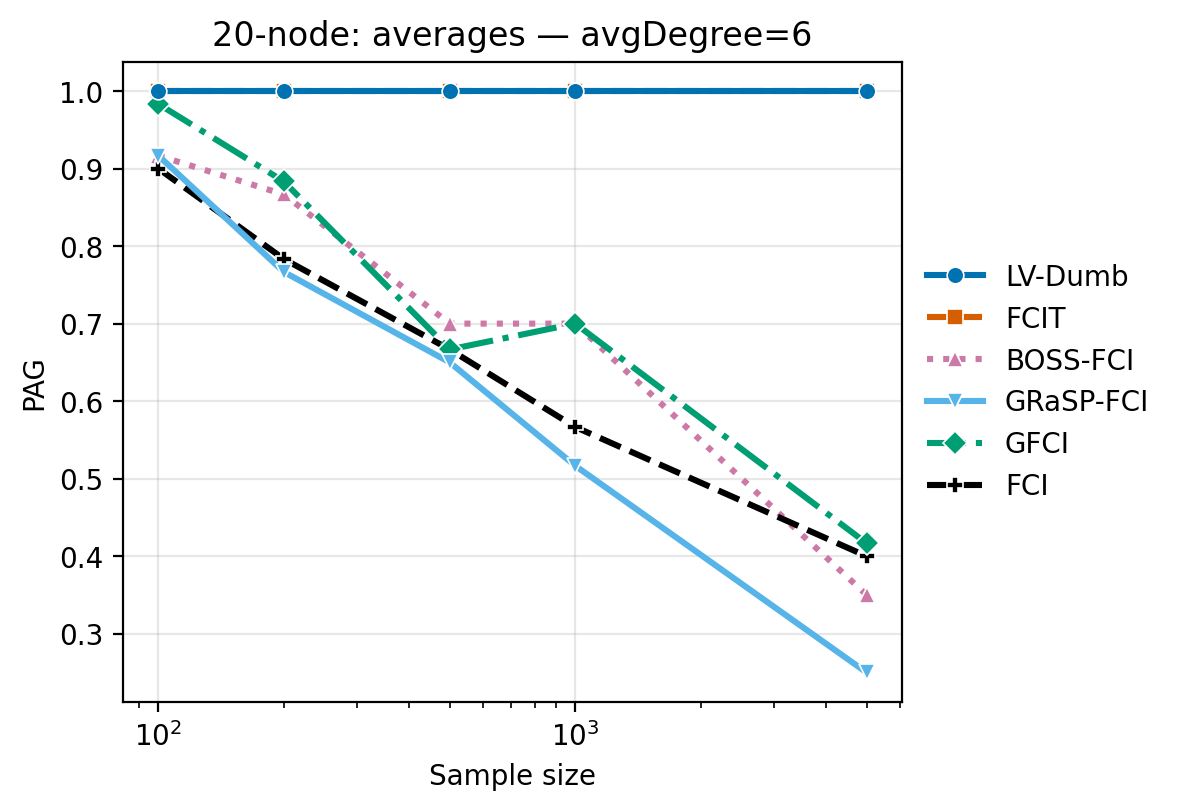}
    \caption{Proportion of well-formed PAGs for 20-node graphs with average degree 6. 
    \textsc{LV-Dumb} and \textsc{FCIT} consistently produce valid PAGs ($=1.0$), confirming structural correctness. 
    Other algorithms fall off, sometimes sharply.
    Results are averaged over graphs with 0, 4, and 8 latent common causes.}
    \label{fig:pag20deg6}
\end{figure}

\begin{figure}
    \centering
    \includegraphics[width=0.5\linewidth]{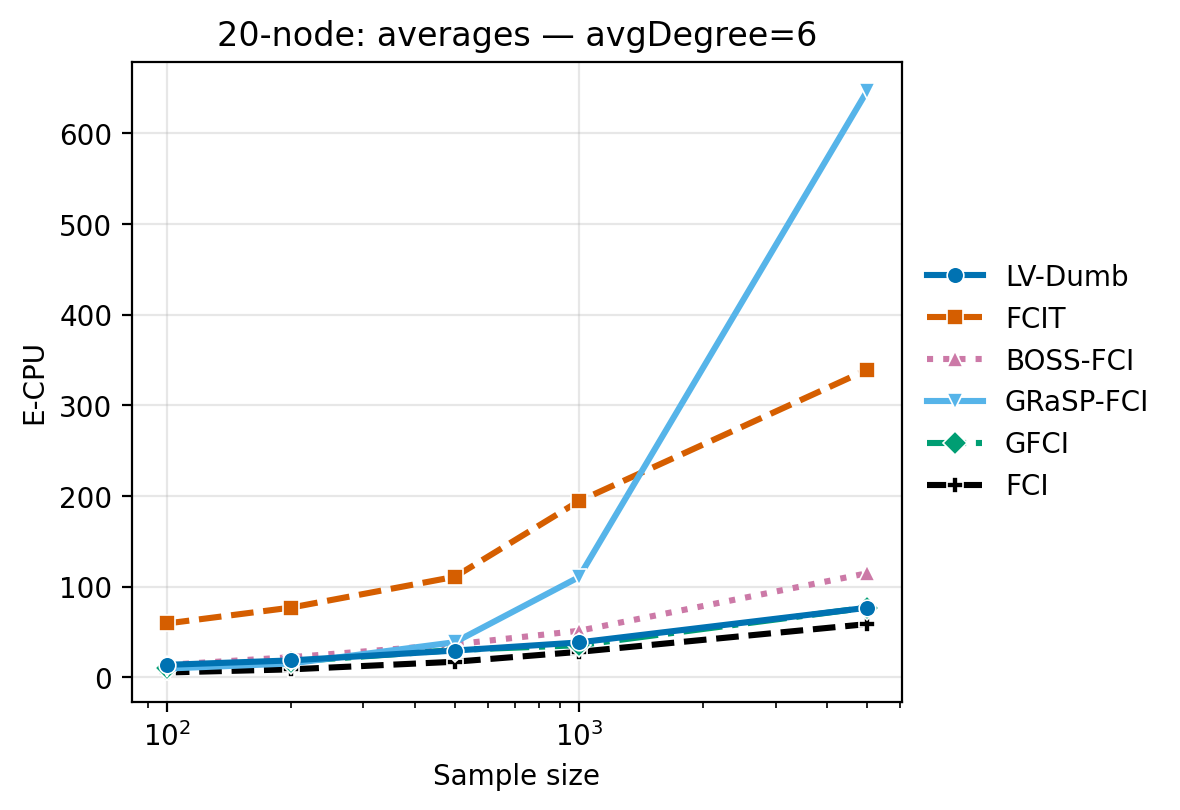}
    \caption{Runtime (CPU milliseconds) for 20-node graphs with average degree 6. 
    \textsc{GRaSP-FCI} incurs the highest cost, growing steeply at large $N$ ($\sim\!600$ ms). 
    \textsc{FCIT} is moderate, while \textsc{LV-Dumb}, \textsc{LV-Dumb}, \textsc{FCI}, and \textsc{GFCI} remain consistently faster. 
    Results are averaged over graphs with 0, 4, and 8 latent common causes.}
    \label{fig:cpu20deg6}
\end{figure}

\end{document}